\pdfoutput=1
\documentclass{article}
\usepackage{log_2024}

\usepackage{booktabs}
\usepackage{multirow}
\usepackage{amsfonts,amsthm}
\usepackage{thmtools, thm-restate}
\usepackage{graphicx}
\usepackage[sort,round]{natbib}
\usepackage{acronym}
\usepackage{csquotes}
\usepackage{subcaption}
\usepackage{algorithm}
\usepackage[noend]{algpseudocode}
\algnewcommand{\algorithmicbreak}{\textbf{break}}
\algnewcommand{\algorithmicforeach}{\textbf{for each}}
\algnewcommand\Break{\algorithmicbreak}
\algdef{SE}[FOR]{ForEach}{EndForEach}[1]
	{\algorithmicforeach\ #1\ \algorithmicdo}
	{\algorithmicend\ \algorithmicforeach}
\algdef{SE}[DOWHILE]{DoWhile}{EndDoWhile}{\algorithmicdo}[1]{\algorithmicwhile\ #1}
\makeatletter
	\ifthenelse{\equal{\ALG@noend}{t}}{\algtext*{EndForEach}}{}
	\ifthenelse{\equal{\ALG@noend}{t}}{\algtext*{EndDoWhile}}{}
\makeatother
\usepackage{tikz}
\usetikzlibrary{angles,arrows,arrows.meta,automata,backgrounds,calc,patterns,positioning,quotes,shapes,shadows}
\AtEndPreamble{\usepackage[capitalise]{cleveref}}
\AtEndPreamble{\acrodef{aacp}[$\alpha$-ACP]{$\alpha$-advanced colour passing}
\acrodef{acp}[ACP]{advanced colour passing}
\acrodef{bn}[BN]{Bayesian network}
\acrodef{cp}[CP]{colour passing}
\acrodef{crv}[CRV]{counting randvar}
\acrodef{decor}[DECOR]{detection of commutative factors}
\acrodef{deft}[DEFT]{detection of exchangeable factors}
\acrodef{er}[ER]{entity-relationship}
\acrodef{fg}[FG]{factor graph}
\acrodef{ljt}[LJT]{lifted junction tree}
\acrodef{lv}[logvar]{logical variable}
\acrodef{lve}[LVE]{lifted variable elimination}
\acrodef{mln}[MLN]{Markov logic network}
\acrodef{mn}[MN]{Markov network}
\acrodef{mpdag}[MPDAG]{maximally oriented partially directed acyclic graph}
\acrodef{pcfg}[PCFG]{parametric causal factor graph}
\acrodef{pcrv}[PCRV]{parameterised CRV}
\acrodef{ppcfg}[PPCFG]{partially directed parametric causal factor graph}
\acrodef{pf}[parfactor]{parametric factor}
\acrodef{pfg}[PFG]{parametric factor graph}
\acrodef{prv}[PRV]{parameterised randvar}
\acrodef{rv}[randvar]{random variable}
\acrodef{ve}[VE]{variable elimination}
\acrodef{wl}[WL]{Weisfeiler-Leman}

\crefname{algocf}{Alg.}{Algs.}
\Crefname{algocf}{Algorithm}{Algorithms}
\crefname{algorithm}{Alg.}{Algs.}
\Crefname{algorithm}{Algorithm}{Algorithms}
\crefname{corollary}{Cor.}{Cors.}
\Crefname{corollary}{Corollary}{Corollaries}
\crefname{definition}{Def.}{Defs.}
\Crefname{definition}{Definition}{Definitions}
\crefname{example}{Ex.}{Exs.}
\Crefname{example}{Example}{Examples}
\crefname{proposition}{Prop.}{Props.}
\Crefname{proposition}{Proposition}{Propositions}
\crefname{section}{Sec.}{Secs.}
\Crefname{section}{Section}{Sections}
\crefname{theorem}{Thm.}{Thms.}
\Crefname{theorem}{Theorem}{Theorems}

\newcommand{\alginput}[1]{\hspace*{\algorithmicindent} \textbf{Input:} #1}
\newcommand{\algoutput}[1]{\hspace*{\algorithmicindent} \textbf{Output:} #1}
\newcommand{\abs}[1]{\lvert #1 \rvert}

\newcommand{\domain}[1]{\ensuremath{\mathrm{dom}(#1)}}

\newcommand{\range}[1]{\ensuremath{\mathrm{range}(#1)}}

\newcommand{\true}{\ensuremath{\mathrm{true}}}

\definecolor{myyellow}{RGB}{247,192,26}
\definecolor{myblue}{RGB}{37,122,164}
\definecolor{mygreen}{RGB}{78,155,133}
\definecolor{mypurple}{RGB}{86,51,94}

\definecolor{newblue}{RGB}{50,113,173}
\definecolor{newred}{RGB}{222,32,36}
\definecolor{newgreen}{RGB}{70,165,69}
\definecolor{newpurple}{RGB}{140,69,152}

\definecolor{cborange}{RGB}{230,159,0}
\definecolor{cbblue}{RGB}{30,136,229}
\definecolor{cbbluedark}{RGB}{46,37,133}
\definecolor{cbpurple}{RGB}{170,68,153}
\definecolor{cbgreen}{RGB}{0,77,64}
\definecolor{cbgreenlight}{RGB}{93,168,153}
\definecolor{cbbrown}{RGB}{126,41,84}

\pgfdeclarelayer{bg}
\pgfsetlayers{bg,main}

\tikzset{
	rv/.style={draw, ellipse},
	pf/.style={draw, rectangle, fill = gray!30},
	arc/.style = {->, >={[round,sep]Stealth}},
	doublearc/.style = {<->, >={[round,sep]Stealth}},
}

\newcommand\factor[6]{
	\node[pf, #1=#3 of #2, label={#4:{#5}}](#6) {};
}

\newcommand\nodecolorshift[5]{
	\node[circle, fill=#1, above right=0.1cm of #2, inner sep=0pt, minimum size=2mm, xshift=#4, yshift=#5](#3) {};
}

\newcommand\factorcolor[3]{
	\node[circle, fill=#1, above right=0cm and 0.05cm of #2, inner sep=0pt, minimum size=2mm](#3) {};
}
\newcommand\factorcolorshift[4]{
	\node[circle, fill=#1, above right=0cm and 0.05cm of #2, inner sep=0pt, minimum size=2mm, xshift=#4](#3) {};
}

\newcommand\pfs[8]{
	\node[pf, #1=#3 of #2, xshift=-1mm, yshift=1mm](#6) {};
	\node[pf, #1=#3 of #2, label={[label distance=1mm]#4:{#5}}](#7) {};
	\node[pf, #1=#3 of #2, xshift=1mm, yshift=-1mm](#8) {};
}

}

\newtheorem{theorem}{Theorem}

\newtheorem{definition}{Definition}
\newtheorem{example}{Example}

\title[Lifted Model Construction without Normalisation]{Lifted Model Construction without Normalisation: A Vectorised Approach to Exploit Symmetries in Factor Graphs}

\author[Malte Luttermann, Ralf Möller, Marcel Gehrke]{
	Malte Luttermann\textsuperscript{1}, Ralf Möller\textsuperscript{2} \and Marcel Gehrke\textsuperscript{2} \\
	\textsuperscript{1}German Research Center for Artificial Intelligence (DFKI), Lübeck \\
	\textsuperscript{2}Institute for Humanities-Centered Artificial Intelligence, University of Hamburg \\
	\email{malte.luttermann@dfki.de,\{ralf.moeller,marcel.gehrke\}@uni-hamburg.de}
}

\begin{document}

\maketitle

\begin{abstract}
	Lifted probabilistic inference exploits symmetries in a probabilistic model to allow for tractable probabilistic inference with respect to domain sizes of \aclp{lv}.
	We found that the current state-of-the-art algorithm to construct a lifted representation in form of a \acl{pfg} misses symmetries between factors that are exchangeable but scaled differently, thereby leading to a less compact representation.
	In this paper, we propose a generalisation of the \ac{acp} algorithm, which is the state of the art to construct a \acl{pfg}.
	Our proposed algorithm allows for potentials of factors to be scaled arbitrarily and efficiently detects more symmetries than the original \ac{acp} algorithm.
	By detecting strictly more symmetries than \ac{acp}, our algorithm significantly reduces online query times for probabilistic inference when the resulting model is applied, which we also confirm in our experiments.
\end{abstract}

\acresetall

\section{Introduction}
\Acp{pfg} are probabilistic relational models, i.e., they combine probabilistic models and relational logic (which can be seen as first-order logic with known universes) to efficiently reason about objects and their relationships under uncertainty.
To allow for tractable probabilistic inference (e.g., inference requiring polynomial time) with respect to domain sizes of \aclp{lv}, \acp{pfg} use representatives of indistinguishable objects to represent groups of \acp{rv}, thereby yielding a more compact model that can be exploited by lifted inference algorithms for faster inference.
Here, probabilistic inference (or just inference for short) refers to the task of computing marginal distributions of \acp{rv} given observations for other \acp{rv} (see \cref{appendix:example_inference} for more details).
Clearly, to run a lifted inference algorithm on a \ac{pfg}, the \ac{pfg} has to be constructed first.
The current state-of-the-art algorithm to construct a \ac{pfg} is the \ac{acp} algorithm.
The \ac{acp} algorithm begins with a propositional model in form of a \ac{fg} and exploits symmetries therein to obtain a \ac{pfg} entailing equivalent semantics as the initial \ac{fg}.
During the course of \ac{acp}, potentials of factors are compared to decide whether factors are equivalent and thus might be grouped.
However, all potentials of the factors must be scaled equally for \ac{acp} to be able to detect symmetries between factors.
In other words, \ac{acp} fails to detect symmetries between factors that are exchangeable but whose potentials differ only by a scalar, thereby leading to a less compact lifted representation if potentials are not normalised before running \ac{acp}.
In this paper, we solve the problem of constructing a \ac{pfg} from a given \ac{fg} such that the resulting \ac{pfg} entails equivalent semantics as the initial \ac{fg} and exchangeable factors are detected independent of the scale of their potentials.
We therefore allow potentials to be learned from different data sources without having to perform a normalisation step while at the same time obtaining a more compact representation for lifted inference than the output of \ac{acp}.

In previous work, \citet{Poole2003a} introduces \acp{pfg} and \acl{lve} as an inference algorithm to carry out lifted probabilistic inference in \acp{pfg}.
Lifted inference exploits symmetries in a probabilistic model by using a representative of indistinguishable objects for computations while maintaining exact answers~\citep{Niepert2014a}.
By using \aclp{lv} in \acp{prv} to represent groups of indistinguishable \acp{rv}, \acl{lve} operating on \acp{pfg} is able to allow for tractable probabilistic inference with respect to domain sizes of \aclp{lv}~\citep{Taghipour2013b}.
After its first introduction, \acl{lve} has been steadily refined by many researchers to reach its current form~\citep{DeSalvoBraz2005a,DeSalvoBraz2006a,Milch2008a,Kisynski2009a,Taghipour2013a,Braun2018a}.
Recently, \citet{Luttermann2024b,Luttermann2024g} extend \acp{pfg} to incorporate causal knowledge, thereby allowing for lifted causal inference in addition to lifted probabilistic inference.
In any case (purely probabilistic or causal), the construction of a \ac{pfg} (or its causal extension, respectively) is necessary to apply lifted inference algorithms afterwards.
The \enquote{CompressFactorGraph} algorithm~\citep{Kersting2009a,Ahmadi2013a} builds on work by \citet{Singla2008a} and detects symmetries in an \ac{fg} to obtain possible groups of \acp{rv} and factors by deploying a colour passing procedure similar to the \acl{wl} algorithm~\citep{Weisfeiler1968a}, which is commonly used to test for graph isomorphism.
To obtain a valid \ac{pfg}, the resulting groups must be represented by introducing \aclp{lv} in \acp{prv}, and the current state-of-the-art algorithm to construct a valid \ac{pfg} entailing equivalent semantics as an initially given \ac{fg} is the \ac{acp} algorithm~\citep{Luttermann2024a,Luttermann2024d,Luttermann2024f}.
While \ac{acp} successfully constructs a valid \ac{pfg} from a given \ac{fg}, it requires all potentials of factors to be scaled by the same scalar, which imposes a serious limitation for practical applications, e.g., when potentials are learned from various data sources and normalisation is undesirable due to floating point arithmetic issues.

To circumvent the requirement of equally scaled potential values in all factors, we propose a modification of the \ac{acp} algorithm that encodes potential values of factors as vectors.
In an earlier work, \citet{Gehrke2020a} show that potentials of factors can be conceived as vectors such that the cosine similarity provides a useful measure to check whether factors \enquote{behave similarly}, thereby allowing to keep symmetries over time and to avoid groundings during temporal probabilistic inference.
By using vector representations of factors' potentials, the potentials do not have to be scaled by the same scalar and thus, we make use of this property already during the construction procedure of the \ac{pfg} in this paper.
Detecting symmetries independent of scalars already during the construction of the lifted representation (i.e., before lifted inference takes place) yields a more compact representation right from the beginning and thereby significantly speeds up online inference afterwards.
We formally show that using such a vector representation maintains equivalent semantics.
Further, we demonstrate that the vector representation can easily be incorporated into the \ac{acp} algorithm to obtain a more compact representation, which we also confirm in our empirical evaluation.

The remaining part of this paper is structured as follows.
First, we provide the necessary background information and introduce notations.
We begin to recap \acp{fg} and afterwards formalise the problem of detecting exchangeable factors.
Thereafter, we take a closer look at the problem of detecting exchangeable factors independent of the scale of their potentials and present our approach that makes use of vector representations of potentials to solve this problem.
We then embed the vectorised approach into the framework of the \ac{acp} algorithm to obtain a generalisation of \ac{acp}, which we evaluate empirically to demonstrate its practical effectiveness before we conclude.

\section{Background}
We begin by defining \acp{fg} as propositional probabilistic graphical models.
An \ac{fg} compactly encodes a full joint probability distribution over a set of \acp{rv} by factorising the distribution into a product of factors~\citep{Frey1997a,Kschischang2001a}.
\begin{definition}[Factor Graph]
	An \emph{\ac{fg}} $G = (\boldsymbol V, \boldsymbol E)$ is an undirected bipartite graph consisting of a node set $\boldsymbol V = \boldsymbol R \cup \boldsymbol \Phi$, where $\boldsymbol R = \{R_1, \ldots, R_n\}$ is a set of variable nodes (\acp{rv}) and $\boldsymbol \Phi = \{\phi_1, \ldots, \phi_m\}$ is a set of factor nodes (functions), as well as a set of edges $\boldsymbol E \subseteq \boldsymbol R \times \boldsymbol \Phi$.
	The term $\range{R_i}$ denotes the possible values of a \ac{rv} $R_i$.
	There is an edge between a variable node $R_i$ and a factor node $\phi_j$ in $\boldsymbol E$ if $R_i$ appears in the argument list of $\phi_j$.
	The argument list $\mathcal A_j$ of a factor $\phi_j(\mathcal A_j)$ is a sequence of \acp{rv} from $\boldsymbol R$.
	A factor is a function that maps its arguments to a positive real number, called potential.
	The semantics of $G$ is given by
	\begin{align} \label{eq:fg_semantics}
		P_G = \frac{1}{Z} \prod_{j=1}^m \phi_j(\mathcal A_j),
	\end{align}
	where $Z$ is the normalisation constant and $\mathcal A_j$ denotes the \acp{rv} occurring in $\phi_j$'s argument list.
\end{definition}
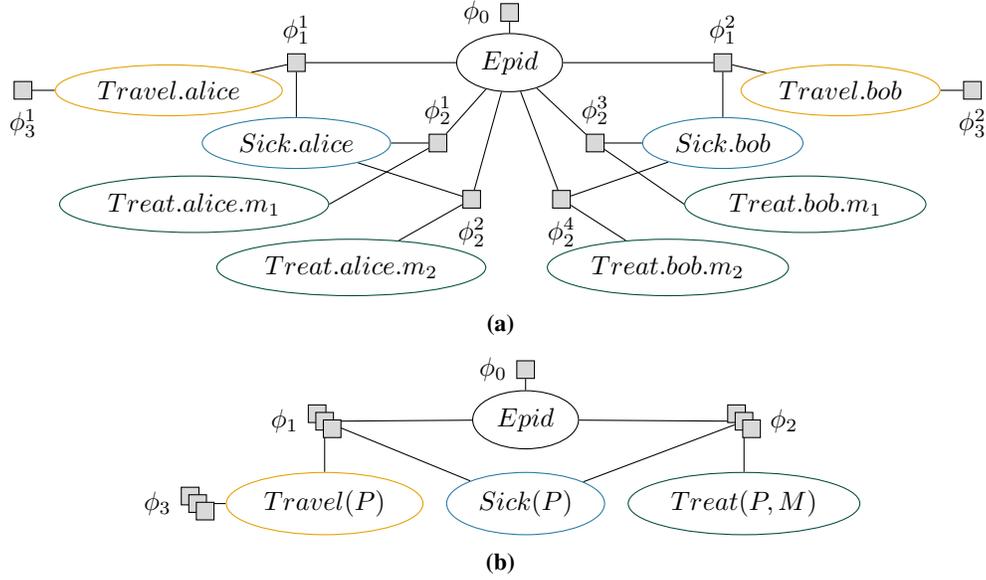
\begin{figure}
	\centering
	\begin{subfigure}{\linewidth}
		\centering
		\begin{tikzpicture}[
	rv/.append style={fill=white}
]
	\node[rv] (E) {$Epid$};

	\factor{above}{E}{0.15cm}{[label distance=0mm]180}{$\phi_0$}{F0}

	\factor{left}{E}{2cm}{[label distance=0mm]90}{$\phi_1^1$}{F1_1}
	\factor{right}{E}{2cm}{[label distance=0mm]90}{$\phi_1^2$}{F1_2}

	\node[rv, draw=myblue, below = 0.6cm of F1_1] (SickA) {$Sick.alice$};
	\node[rv, draw=myblue, below = 0.6cm of F1_2] (SickB) {$Sick.bob$};

	\factor{right}{SickA}{0.5cm}{[label distance=0mm]90}{$\phi_2^1$}{F2_1}
	\factor{below right}{F2_1}{0.5cm and 0.2cm}{[label distance=0mm]270}{$\phi_2^2$}{F2_2}

	\factor{left}{SickB}{0.5cm}{[label distance=0mm]90}{$\phi_2^3$}{F2_3}
	\factor{below left}{F2_3}{0.5cm and 0.2cm}{[label distance=0mm]270}{$\phi_2^4$}{F2_4}

	\node[rv, draw=cborange, below left = 0.0cm and 0.5cm of F1_1] (TravelA) {$Travel.alice$};
	\node[rv, draw=cborange, below right = 0.0cm and 0.5cm of F1_2] (TravelB) {$Travel.bob$};
	\node[rv, draw=cbgreen, below left = 0.3cm and -0.8cm of SickA] (TreatAM1) {$Treat.alice.m_1$};
	\node[rv, draw=cbgreen, below right = 0.3cm and -0.45cm of TreatAM1] (TreatAM2) {$Treat.alice.m_2$};
	\node[rv, draw=cbgreen, below right = 0.3cm and -0.8cm of SickB] (TreatBM1) {$Treat.bob.m_1$};
	\node[rv, draw=cbgreen, below left = 0.3cm and -0.45cm of TreatBM1] (TreatBM2) {$Treat.bob.m_2$};

	\factor{left}{TravelA}{0.3cm}{[label distance=0mm]270}{$\phi_3^1$}{F3_1}
	\factor{right}{TravelB}{0.3cm}{[label distance=0mm]270}{$\phi_3^2$}{F3_2}

	\begin{pgfonlayer}{bg}
		\draw (E) -- (F0);
		\draw (E) -- (F1_1);
		\draw (E) -- (F2_1);
		\draw (E) -- (F2_2);
		\draw (E) -- (F1_2);
		\draw (E) -- (F2_3);
		\draw (E) -- (F2_4);
		\draw (SickA) -- (F1_1);
		\draw (SickA) -- (F2_1);
		\draw (SickA) -- (F2_2);
		\draw (TravelA) -- (F1_1);
		\draw (TreatAM1.east) -- (F2_1);
		\draw (TreatAM2) -- (F2_2);
		\draw (SickB) -- (F1_2);
		\draw (SickB) -- (F2_3);
		\draw (SickB) -- (F2_4);
		\draw (TravelB) -- (F1_2);
		\draw (TreatBM1.west) -- (F2_3);
		\draw (TreatBM2) -- (F2_4);
		\draw (TravelA) -- (F3_1);
		\draw (TravelB) -- (F3_2);
	\end{pgfonlayer}
\end{tikzpicture}
		\caption{}
		\label{fig:example_fg_epid}
	\end{subfigure}

	\begin{subfigure}{\linewidth}
		\centering
		\begin{tikzpicture}
	\node[rv] (E) {$Epid$};
	\node[rv, draw=myblue, below = 0.3cm of E] (S) {$Sick(P)$};
	\node[rv, draw=cborange, left = 0.3cm of S] (Travel) {$Travel(P)$};
	\node[rv, draw=cbgreen, right = 0.3cm of S] (Treat) {$Treat(P,M)$};
	\factor{above}{E}{0.15cm}{180}{$\phi_0$}{G0}
	\pfs{above}{Travel}{0.55cm}{180}{$\phi_1$}{G1a}{G1}{G1b}
	\pfs{above}{Treat}{0.55cm}{0}{$\phi_2$}{G2a}{G2}{G2b}
	\pfs{left}{Travel}{0.25cm}{180}{$\phi_3$}{G3a}{G3}{G3b}

	\begin{pgfonlayer}{bg}
		\draw (E) -- (G0);
		\draw (E) -- (G1);
		\draw (E) -- (G2);
		\draw (S) -- (G1);
		\draw (S) -- (G2);
		\draw (Travel) -- (G1);
		\draw (Treat) -- (G2);
		\draw (Travel) -- (G3);
	\end{pgfonlayer}
\end{tikzpicture}
		\caption{}
		\label{fig:example_pfg_epid}
	\end{subfigure}
	\caption{(a) An \ac{fg} encoding a full joint probability distribution for an epidemic example~\citep{Hoffmann2022a}, (b) a \ac{pfg} corresponding to the lifted representation of the \ac{fg} shown in \cref{fig:example_fg_epid}. The mappings of argument values to potentials of the factors are omitted for brevity.}
	\label{fig:example_fg_pfg_epid}
\end{figure}
\begin{example}
	\Cref{fig:example_fg_epid} shows an \ac{fg} for an epidemic example.
	The \ac{fg} consists of two people ($alice$ and $bob$) as well as two possible medications ($m_1$ and $m_2$) for treatment.
	For each person, there are two Boolean \acp{rv} (that is, \acp{rv} having a Boolean range) $Sick$ and $Travel$, indicating whether the person is sick and travels, respectively.
	Moreover, there is another Boolean \ac{rv} $Treat$ for each combination of person and medication, specifying whether the person is treated with the medication.
	The Boolean \ac{rv} $Epid$ states whether an epidemic is present.
\end{example}
Lifted inference algorithms exploit symmetries in \acp{fg} to allow for tractable probabilistic inference with respect to domain sizes of \aclp{lv}.
In a lifted representation such as a \ac{pfg}, \aclp{prv} and \aclp{pf} represent sets of \acp{rv} and factors, respectively~\citep{Poole2003a}.
Symmetries in \acp{fg} frequently occur in relational models and are highly relevant in many real world domains.
For example, in the epidemic domain, each person influences the probability of an epidemic in the same way---that is, the probability of having an epidemic depends on the number of sick people and not on individual people being sick.
In other words, the probability for an epidemic is the same if there is a single sick person and the remaining people in the universe are not sick, independent of whether $alice$ or $bob$ is sick.
Analogously, there are symmetries in many other domains, e.g., for movies the popularity of an actor influences the success of a movie in the same way for each actor being part of the movie, and so on.
\begin{example}
	A \ac{pfg} corresponding to the lifted representation of the \ac{fg} illustrated in \cref{fig:example_fg_epid} is shown in \cref{fig:example_pfg_epid}.
	Here, two \aclp{lv} $P$ and $M$ with domains $\domain{P} = \{alice, bob\}$ and $\domain{M} = \{m_1, m_2\}$ are introduced to represent groups of indistinguishable people and medications, respectively.
	Further, there are \aclp{pf} that represent groups of factors, e.g., $\phi_1$ represents $\phi_1^1$ and $\phi_1^2$.
	The underlying assumption is that $\phi_1^1$ and $\phi_1^2$ are exchangeable and hence encode equivalent semantics.
	By using \aclp{lv} in \aclp{prv}, the number of \aclp{prv} and \aclp{pf} in the graph remains constant even if the number of people and medications increases.
\end{example}
To detect symmetries in an \ac{fg} and obtain a \ac{pfg} for lifted inference, the \ac{acp} algorithm~\citep{Luttermann2024a} is the current state of the art.
The \ac{acp} algorithm employs a colour passing routine to identify symmetric subgraphs and transforms a given \ac{fg} into a \ac{pfg} entailing equivalent semantics as the initial \ac{fg}.
A formal description of the \ac{acp} algorithm is given in \cref{appendix:acp}.
For now, it is important to understand that \ac{acp} has to detect exchangeable factors during the course of the algorithm.
Exchangeable factors are factors that encode equivalent semantics and play a crucial role when detecting and exploiting symmetries in an \ac{fg}.
In the next section, we investigate the problem of detecting exchangeable factors in \acp{fg} independent of the scale of their potentials in detail and provide an efficient solution to this problem.

\section{Avoiding Normalisation During Lifted Model Construction}
Before we formally define the notion of exchangeable factors, let us take a look at the upcoming example, which illustrates the idea of having differently scaled potentials in exchangeable factors.
\begin{example}
	Consider again the \ac{fg} depicted in \cref{fig:example_fg_epid} and let us assume we want to learn the potentials of the factors from observed data (e.g., by counting the occurrences of combinations of range values).
	For example, to obtain the potentials for the factor $\phi_1^1(Travel.alice, Sick.alice, Epid)$, occurrences of $alice$ becoming sick when travelling are counted.
	Analogously, occurrences of $bob$ becoming sick when travelling are counted to obtain the potentials of $\phi_1^2(Travel.bob, Sick.bob, Epid)$.
	If both factors encode equivalent potentials, they can be grouped (as in \cref{fig:example_pfg_epid}).
	Now, assume $alice$ travels twice as much as $bob$ and both become sick on every second trip on average.
	In consequence, $alice$ and $bob$ \enquote{behave identically} with respect to becoming sick when travelling but the potentials of $\phi_1^1$ and $\phi_1^2$ lie on a different scale---in this particular example, the potentials of $\phi_1^1$ are equal to the potentials of $\phi_1^2$ times two (because $alice$ travels twice as much as $bob$).
\end{example}
A straightforward solution to deal with different scales of potentials is to normalise potentials.
However, we cannot always assume that a given \ac{fg} contains normalised potentials by default and in practical applications, the normalisation of potentials is often undesirable as it results in additional floating point arithmetics causing numerical issues.
We thus develop a solution that does not require potentials to be normalised but still detects exchangeable factors, which we formally define next.
\begin{definition}[Exchangeable Factors] \label{def:exchangeable_scaled}
	Let $\phi_1(R_1, \ldots, R_n)$ and $\phi_2(R'_1, \ldots, R'_n)$ denote two factors in an \ac{fg} $G$.
	Then, $\phi_1$ and $\phi_2$ represent equivalent potentials if and only if there exists a scalar $\alpha \in \mathbb{R}^+$ and a permutation $\pi$ of $\{1, \ldots, n\}$ such that for all $r_1, \ldots, r_n \in \times_{i=1}^n \range{R_i}$ it holds that $\phi_1(r_1, \ldots, r_n) = \alpha \cdot \phi_2(r_{\pi(1)}, \ldots, r_{\pi(n)})$.
	Factors that represent equivalent potentials are called \emph{exchangeable factors}.
\end{definition}
Note that as a necessary condition, exchangeable factors must be defined over the same function domain and hence must have the same number of arguments.
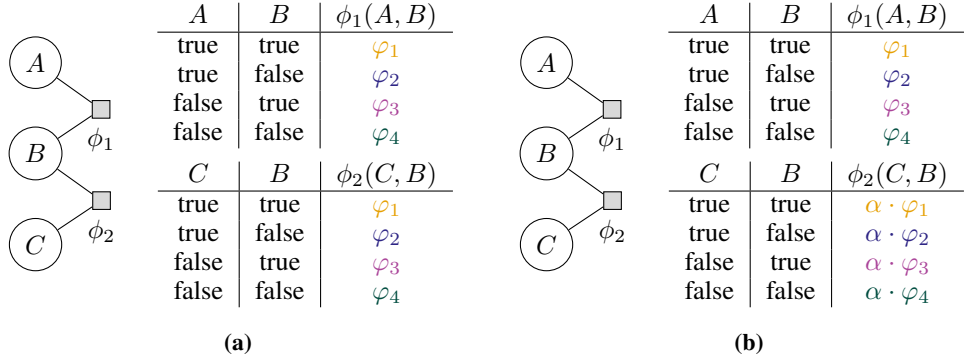
\begin{figure}
	\centering
	\begin{subfigure}{0.48\linewidth}
		\centering
		\begin{tikzpicture}
	\node[circle, draw] (A) {$A$};
	\node[circle, draw] (B) [below = 0.5cm of A] {$B$};
	\node[circle, draw] (C) [below = 0.5cm of B] {$C$};
	\factor{below right}{A}{0.25cm and 0.5cm}{270}{$\phi_1$}{f1}
	\factor{below right}{B}{0.25cm and 0.5cm}{270}{$\phi_2$}{f2}

	\node[right = 0.5cm of f1, yshift=4.5mm] (tab_f1) {
		\begin{tabular}{c|c|c}
			$A$   & $B$   & $\phi_1(A,B)$                  \\ \hline
			true  & true  & $\color{cborange}   \varphi_1$ \\
			true  & false & $\color{cbbluedark} \varphi_2$ \\
			false & true  & $\color{cbpurple}   \varphi_3$ \\
			false & false & $\color{cbgreen}    \varphi_4$ \\
		\end{tabular}
	};

	\node[right = 0.5cm of f2, yshift=-4.5mm] (tab_f2) {
		\begin{tabular}{c|c|c}
			$C$   & $B$   & $\phi_2(C, B)$ \\ \hline
			true  & true  & $\color{cborange}   \varphi_1$ \\
			true  & false & $\color{cbbluedark} \varphi_2$ \\
			false & true  & $\color{cbpurple}   \varphi_3$ \\
			false & false & $\color{cbgreen}    \varphi_4$ \\
		\end{tabular}
	};

	\draw (A) -- (f1);
	\draw (B) -- (f1);
	\draw (B) -- (f2);
	\draw (C) -- (f2);
\end{tikzpicture}
		\caption{}
		\label{fig:example_fg}
	\end{subfigure}
	\begin{subfigure}{0.48\linewidth}
		\centering
		\begin{tikzpicture}
	\node[circle, draw] (A) {$A$};
	\node[circle, draw] (B) [below = 0.5cm of A] {$B$};
	\node[circle, draw] (C) [below = 0.5cm of B] {$C$};
	\factor{below right}{A}{0.25cm and 0.5cm}{270}{$\phi_1$}{f1}
	\factor{below right}{B}{0.25cm and 0.5cm}{270}{$\phi_2$}{f2}

	\node[right = 0.5cm of f1, yshift=4.5mm] (tab_f1) {
		\begin{tabular}{c|c|c}
			$A$   & $B$   & $\phi_1(A,B)$                  \\ \hline
			true  & true  & $\color{cborange}   \varphi_1$ \\
			true  & false & $\color{cbbluedark} \varphi_2$ \\
			false & true  & $\color{cbpurple}   \varphi_3$ \\
			false & false & $\color{cbgreen}    \varphi_4$ \\
		\end{tabular}
	};

	\node[right = 0.5cm of f2, yshift=-4.5mm] (tab_f2) {
		\begin{tabular}{c|c|c}
			$C$   & $B$   & $\phi_2(C, B)$                              \\ \hline
			true  & true  & $\color{cborange}   \alpha \cdot \varphi_1$ \\
			true  & false & $\color{cbbluedark} \alpha \cdot \varphi_2$ \\
			false & true  & $\color{cbpurple}   \alpha \cdot \varphi_3$ \\
			false & false & $\color{cbgreen}    \alpha \cdot \varphi_4$ \\
		\end{tabular}
	};

	\draw (A) -- (f1);
	\draw (B) -- (f1);
	\draw (B) -- (f2);
	\draw (C) -- (f2);
\end{tikzpicture}
		\caption{}
		\label{fig:example_fg_scaled}
	\end{subfigure}
	\caption{(a) An exemplary \ac{fg}, (b) another \ac{fg} encoding equivalent semantics as the \ac{fg} shown in (a) but containing a factor $\phi_2$ whose potentials are scaled by factor $\alpha \in \mathbb{R}^+$.}
	\label{fig:example_fg_and_scaled}
\end{figure}
\begin{example}
	Take a look at the \ac{fg} depicted in \cref{fig:example_fg}, which features two factors $\phi_1$ and $\phi_2$ that map to the exact same potential values $\varphi_i \in \mathbb{R}^+$, $i \in \{1,\ldots,4\}$.
	In this scenario, both tables of mappings from assignments of arguments to potential values are identical (i.e., $\alpha = 1$ and $\pi$ is the identity function) and hence, it is easy to tell that $\phi_1$ and $\phi_2$ are exchangeable.
	If we now consider the \ac{fg} shown in \cref{fig:example_fg_scaled}, we can observe that the potential values of $\phi_2$ are scaled by a factor $\alpha \in \mathbb{R}^+$.
	Despite the scaling, $\phi_1$ and $\phi_2$ encode equivalent semantics and thus are exchangeable.
\end{example}
We remark that in general, $\pi$ does not have to be the identity function, i.e., there might be situations where, for example, the argument positions of $C$ and $B$ in $\phi_2$ are swapped and the potential values in the table read $\varphi_1, \varphi_3, \varphi_2, \varphi_4$ from top to bottom instead of $\varphi_1, \varphi_2, \varphi_3, \varphi_4$.
Note that the potential mappings are still the same but their order is a different one.
For now, we focus on the scalar $\alpha$ and assume that $\pi$ is the identity function, that is, the arguments are already ordered such that exchangeable arguments are located at the same argument positions if there are any exchangeable arguments.
Later on, in \cref{sec:permuted_arguments_scaled}, we also show how to deal with arbitrary permutations of arguments.

A fundamental insight is that factors, whose potential mappings are equivalent up to a scalar $\alpha$, are semantically equivalent.
The intuition here is that the \emph{ratio} of the potentials within a factor is the relevant part for the semantics of the factor whereas the absolute values do not matter.
For example, think of a factor $\phi$ that has two mappings in total, one for the assignment $\mathrm{true}$ and one for the assignment $\mathrm{false}$.
Semantically, it does not matter whether $\phi$ maps $\mathrm{true}$ to $1$ and $\mathrm{false}$ to $2$ or $\mathrm{true}$ to $2$ and $\mathrm{false}$ to $4$ because in both cases, $\phi$ weights the assignment $\mathrm{false}$ twice as much as the assignment $\mathrm{true}$.
The normalisation constant $Z$ in \cref{eq:fg_semantics} ensures that in both cases, the probability for $\mathrm{true}$ is $1/3$ and the probability for $\mathrm{false}$ is $2/3$.
We next formalise this insight.
\begin{theorem} \label{th:scaling_semantics}
	Let $G = (\boldsymbol V, \boldsymbol E)$ denote an \ac{fg} with $\boldsymbol V = \boldsymbol R \cup \boldsymbol \Phi$, where $\boldsymbol R = \{R_1, \ldots, R_n\}$ is a set of \acp{rv} and $\boldsymbol \Phi = \{\phi_1, \ldots, \phi_m\}$ is a set of factors.
	Then, scaling any factor $\phi_k \in \boldsymbol \Phi$ by a scalar $\alpha \in \mathbb{R}^+$ leaves the semantics of $G$ unchanged.
\end{theorem}
\begin{proof}
	Recall that the semantics of $G$ (before scaling) is given by $P_G = \frac{1}{Z} \prod_{j=1}^m \phi_j(\mathcal A_j)$, where $\mathcal A_j$ denotes the \acp{rv} occurring in $\phi_j$'s argument list and $Z$ is the normalisation constant, defined as
	\begin{align}
		Z = \sum\limits_{\boldsymbol a \in \times_{i=1}^n \range{R_i}} \prod\limits_{j = 1}^m \phi_j(\mathcal A_j = \boldsymbol a_j),
	\end{align}
	where $\boldsymbol a_j$ denotes the assigned values to arguments $\mathcal A_j$ according to the assignment $\boldsymbol a$.
	Now, assume that $\phi_k \in \Phi$ is scaled by $\alpha \in \mathbb{R}^+$.
	Then, $P_G$ changes to $P_G = \frac{1}{Z} \cdot \alpha \cdot \prod_{j=1}^m \phi_j(\mathcal A_j)$ and $Z$ changes to $Z = \alpha \cdot \sum_{\boldsymbol a \in \times_{i=1}^n \range{R_i}} \prod_{j = 1}^m \phi_j(\mathcal A_j = \boldsymbol a_j)$.
	In consequence, it holds that $P_G = \frac{1}{\alpha \cdot Z} \cdot \alpha \cdot \prod_{j=1}^m \phi_j(\mathcal A_j)$, which is equivalent to the original definition of $P_G$ as $\alpha$ cancels out.
\end{proof}
\Cref{th:scaling_semantics} implies that it is also possible to scale various factors by different scalars without changing the semantics of the underlying model.
Using this insight, it becomes clear that \acp{fg} can be further compressed by taking into account factors that are exchangeable up to a scalar $\alpha$.
We next show how exchangeable factors can efficiently be detected independent of the scaling factor $\alpha$.

\subsection{Dealing with Scaled Potentials} \label{sec:vector_rep_cosine_distance}
Previous work by \citet{Gehrke2020a} shows that potentials of factors can be conceived as vectors such that the cosine similarity of the vectors can be used to check whether factors \enquote{behave identically}, thereby avoiding groundings in temporal probabilistic inference.
We apply the idea of representing potentials as vectors to detect exchangeable factors independent of a scaling factor already during the construction of a \ac{pfg} to obtain a more compact model even before online inference takes place.
\begin{definition}[Vector Representation of Factors]
	Let $\phi(R_1, \ldots, R_n)$ denote a factor.
	The \emph{vector representation} of $\phi$ is defined as the vector $\vec \phi = (\phi(\boldsymbol a))_{\boldsymbol a \in \times_{i=1}^n \range{R_i}}$.
\end{definition}
\begin{example}
	Consider the factors $\phi_1$ and $\phi_2$ depicted in \cref{fig:example_fg_scaled}.
	The vector representations of $\phi_1$ and $\phi_2$ are given by $\vec \phi_1 = (\varphi_1, \varphi_2, \varphi_3, \varphi_4)$ and $\vec \phi_2 = (\alpha \varphi_1, \alpha \varphi_2, \alpha \varphi_3, \alpha \varphi_4)$, respectively.
\end{example}
Given a vector representation of a factor, the idea is that vectors of exchangeable factors point to the same direction in the vector space.
Thus, the angle between those vectors can be computed to determine whether the factors are exchangeable because exchangeable factors have vector representations whose angle is equal to zero (i.e., they are collinear).
The upcoming example illustrates this idea.
\begin{figure}
	\centering
	\begin{tikzpicture}
		\draw[arc] (0, 0) -- (5.4, 0);
		\draw[arc] (0, 0) -- (0, 2.9);

		\foreach \x in {0, 1, ..., 10}
			\draw ({\x / 2}, 0.1) -- ({\x / 2}, -0.1) node[below] {\x};

		\foreach \y in {0, 1, ..., 5}
			\draw (0.1, {\y / 2}) -- (-0.1, {\y / 2}) node[left] {\y};

		\draw[arc, thick, cborange] (0, 0) -- (4, 1);
		\node[cborange] at (4, 1) [above] {$\vec \phi_1$};
		\draw[arc, thick, cbgreen] (0, 0) -- (2, 0.5);
		\node[cbgreen] at (2, 0.5) [above] {$\vec \phi_2$};

		\draw[arc, thick, cbbluedark] (0, 0) -- (1, 1);
		\node[cbbluedark] at (1, 1) [above] {$\vec \phi_3$};
		\draw[arc, thick, cbpurple] (0, 0) -- (2.2, 1.8);
		\node[cbpurple] at (2.2, 1.8) [above] {$\vec \phi_4$};

		\coordinate (o) at (0, 0);
		\coordinate (v1) at (2, 0.5);
		\coordinate (v2) at (2.2, 1.8);

		\pic [draw, angle radius=11mm, angle eccentricity=0.8, "$\theta$"] {angle = v1--o--v2};
\end{tikzpicture}
	\caption{Vector representations for exemplary factors $\phi_1, \ldots, \phi_4$. For the sake of this example, every factor maps two possible assignments to a potential value each.
	The mappings of the factors are encoded as vectors and are given by $\vec \phi_1 = (8, 2)$ (i.e., $\phi_1$ maps its first assignment to potential value $8$ and the second assignment to potential value $2$), $\vec \phi_2 = (4, 1)$, $\vec \phi_3 = (2, 2)$, and $\vec \phi_4 = (4.4, 3.6)$.}
	\label{fig:example_vector_representation}
\end{figure}
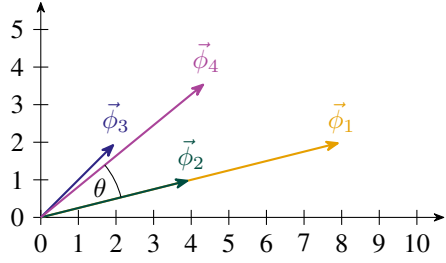
\begin{example}
	Take a look at \cref{fig:example_vector_representation}, which shows the vector representations $\vec \phi_1 = (8, 2)$, $\vec \phi_2 = (4, 1)$, $\vec \phi_3 = (2, 2)$, and $\vec \phi_4 = (4.4, 3.6)$ for exemplary factors $\phi_1, \ldots, \phi_4$.
	To allow for a two-dimensional visualisation, every factor has two possible assignments (e.g., due to having a single Boolean argument).
	The angle between $\vec \phi_1$ and $\vec \phi_2$ is exactly zero, indicating that $\phi_1$ and $\phi_2$ are collinear and hence exchangeable, which can be verified as $\phi_1(\boldsymbol a) = 2 \cdot \phi_2(\boldsymbol a)$ holds for all assignments $\boldsymbol a$.
	At the same time, the angle between, e.g., $\vec \phi_1$ and $\vec \phi_3$ is much larger than zero, indicating that $\phi_1$ and $\phi_3$ are not exchangeable.
	Further, observe that the angle between $\vec \phi_3$ and $\vec \phi_4$ is not exactly zero but close to zero, indicating that $\phi_3$ and $\phi_4$ are not equivalent but approximately equivalent.
\end{example}
By using vector representations and the cosine similarity between them, exchangeable factors can efficiently be detected in practice.
The cosine similarity between two vector representations of factors lies within the interval $[0, 1]$ and reaches its maximum value of one if the angle between the vectors is zero. 
To obtain a distance measure, we define the cosine distance as one minus the cosine similarity.
\begin{definition}[Cosine Distance, \citealp{Gehrke2020a}]
	Let $\phi_1(R_1, \ldots, R_n)$ and $\phi_2(R'_1, \ldots, R'_n)$ denote two factors.
	The \emph{cosine distance} between $\phi_1$ and $\phi_2$ is defined as
	\begin{align} \label{eq:cosine_distance}
		D_{\cos}(\phi_1, \phi_2) = 1 - \frac{\sum\limits_{\boldsymbol a \in \times_{i=1}^n \range{R_i}} \phi_1(\boldsymbol a) \cdot \phi_2(\boldsymbol a)}{\sqrt{\sum\limits_{\boldsymbol a \in \times_{i=1}^n \range{R_i}} \phi_1(\boldsymbol a)^2} \cdot \sqrt{\sum\limits_{\boldsymbol a \in \times_{i=1}^n \range{R'_i}} \phi_2(\boldsymbol a)^2}}.
	\end{align}
	If $\phi_1$ and $\phi_2$ are defined over different function domains, we define $D_{\cos}(\phi_1, \phi_2) = \infty$.
\end{definition}
A fundamental advantage of using vector representations in combination with the cosine distance to search for exchangeable factors is that it is also possible to allow for a small deviation of $D_{\cos}$ from zero (e.g., dependent on a hyperparameter $\varepsilon$).
While it is also conceivable to directly compare the tables of potential mappings and allowing for a deviation controlled by $\varepsilon$, a direct comparison of tables becomes sophisticated in settings where both a deviation and a scaling factor $\alpha$ have to be considered at the same time.
Vector representations circumvent such issues and allow for a straightforward comparison of factors at any time. 
In this paper, however, we focus on the problem of exact lifted model construction, i.e., we aim to transform a given \ac{fg} into a \ac{pfg} entailing equivalent semantics as the initial \ac{fg}.
Allowing for a deviation between potentials results in the problem setup of approximate lifted model construction, which is a different problem not considered in detail here. 

Before we continue to deal with permutations of arguments in addition to scaled potentials, we formally show that the cosine distance is a suitable measure to check for exchangeability.
\begin{restatable}{theorem}{cosineDistanceMeasureTheorem}
	Let $\phi_1(R_1, \ldots, R_n)$ and $\phi_2(R'_1, \ldots, R'_n)$ denote two factors.
	If $\phi_1$ and $\phi_2$ are exchangeable, then it holds that $D_{\cos}(\phi_1, \phi_2) = 0$.
\end{restatable}
\begin{proof}[Proof Sketch]
	If $\phi_1$ and $\phi_2$ are exchangeable, there exists a scalar $\alpha \in \mathbb{R}^+$ and a permutation $\pi$ of $\{1, \ldots, n\}$ such that for all $r_1, \ldots, r_n \in \times_{i=1}^n \range{R_i}$ it holds that $\phi_1(r_1, \ldots, r_n) = \alpha \cdot \phi_2(r_{\pi(1)}, \ldots, r_{\pi(n)})$.
	Without loss of generality, assume that the arguments of $\phi_2$ are rearranged such that for all $r_1, \ldots, r_n \in \times_{i=1}^n \range{R_i}$ it holds that $\phi_1(r_1, \ldots, r_n) = \alpha \cdot \phi_2(r_1, \ldots, r_n)$.
	Then, entering $\phi_1(r_1, \ldots, r_n) = \alpha \cdot \phi_2(r_1, \ldots, r_n)$ into \cref{eq:cosine_distance} yields $D_{\cos}(\phi_1, \phi_2) = 0$.
\end{proof}
Note that the cosine distance is a measure to check for collinearity of the vectors $\vec \phi_1$ and $\vec \phi_2$, that is, to check whether there exists a scalar $\alpha$ such that $\vec \phi_1 = \alpha \cdot \vec \phi_2$.
As we only have to check for collinearity, we can avoid computing the cosine distance (and hence avoid floating point arithmetics during exchangeability checks) by checking the equality of products of potential values.
The technical details for collinearity checks using only multiplication operations are given in \cref{appendix:collinearity_wo_division}.
However, also keep in mind that even when using the cosine distance to determine collinearity of vectors (which involves floating point arithmetics), we are able to avoid floating point numbers in the tables of potential mappings of the factors, which is the more important place to avoid floating point arithmetics (because during probabilistic inference, these numbers are multiplied).

So far, we did not pay attention to permutations of arguments when looking for exchangeable factors.
In practice, however, we cannot assume that exchangeable arguments are always located at the same argument position in their respective factors.
Therefore, in the next section, we investigate the problem of detecting exchangeable factors independent of the scale of their potentials while at the same time taking arbitrary permutations of their arguments into account.

\subsection{Dealing with Permutations of Arguments} \label{sec:permuted_arguments_scaled}
A straightforward approach to handle permutations of arguments when searching for exchangeable factors is to iterate over all possible argument permutations of one of the factors, rearrange its arguments and its table of potential mappings accordingly, and then compute the cosine distance as described in \cref{sec:vector_rep_cosine_distance}.
If there exists a permutation such that the cosine distance is zero, the factors are exchangeable, otherwise they are not.
Such an approach, however, is computationally expensive as it iterates over $O(n!)$ argument permutations for a factor with $n$ arguments in the worst case.
\citet{Luttermann2024d} introduce the \emph{\ac{deft}} algorithm, which avoids iterating over all permutations of arguments and thereby allows to efficiently detect exchangeable factors according to \cref{def:exchangeable_scaled} where $\alpha = 1$.
In other words, \ac{deft} is able to efficiently handle permutations of arguments but does not consider differently scaled potentials.
We now combine the ideas of \ac{deft} and the vector representation in combination with the cosine distance to handle both scalars different from one and permutations of arguments simultaneously.

The idea behind the \ac{deft} algorithm is that a factor maps its arguments to potential values that can be distributed across so-called \emph{buckets}.
Buckets count the occurrences of specific range values in an assignment for a subset of a factor's arguments and within these buckets, possible permutations of arguments are heavily restricted such that not all permutations have to be considered.
Before we illustrate the idea at an example, we give a formal definition of a bucket.
\begin{definition}[Bucket, \citealp{Luttermann2024d}]
	Let $\phi(R_1, \ldots, R_n)$ denote a factor and let $\boldsymbol S \subseteq \{R_1, \ldots, R_n\}$ denote a subset of $\phi$'s arguments such that $\range{R_i} = \range{R_j}$ holds for all $R_i, R_j \in \boldsymbol S$.
	Further, let $\mathcal V$ denote the range of the elements in $\boldsymbol S$ (identical for all $R_i \in \boldsymbol S$).
	Then, a \emph{bucket} $b$ entailed by $\boldsymbol S$ is a set of tuples $\{(v_i, n_i)\}_{i = 1}^{\abs{\mathcal V}}$, $v_i \in \mathcal V$, $n_i \in \mathbb{N}$, and $\sum_i n_i = \abs{\boldsymbol S}$, such that $n_i$ specifies the number of occurrences of potential value $v_i$ in an assignment for all \acp{rv} in $\boldsymbol S$.
	A shorthand notation for $\{(v_i, n_i)\}_{i = 1}^{\abs{\mathcal V}}$ is $[n_1, \dots, n_{\abs{\mathcal V}}]$.
	In abuse of notation, we denote by $\phi^{\succ}(b)$ the ordered multiset of potentials a bucket $b$ is mapped to by $\phi$ (in order of their appearance in $\phi$'s table of potential mappings). 
	The set of all buckets entailed by $\phi$ is denoted as $\mathcal B(\phi)$.
\end{definition}
\begin{figure}
	\centering
	\begin{tikzpicture}
	\node[circle, draw] (A) {$A$};
	\node[circle, draw] (B) [below = 0.5cm of A] {$B$};
	\node[circle, draw] (C) [below = 0.5cm of B] {$C$};
	\factor{below right}{A}{0.25cm and 0.5cm}{270}{$\phi_1$}{f1}
	\factor{below right}{B}{0.25cm and 0.5cm}{270}{$\phi_2$}{f2}

	\node[right = 0.5cm of f1, yshift=4.5mm] (tab_f1) {
		\begin{tabular}{c|c|c|c}
			$A$   & $B$   & $\phi_1(A,B)$                  & $b$     \\ \hline
			true  & true  & $\color{cborange}   \varphi_1$ & $[2,0]$ \\
			true  & false & $\color{cbbluedark} \varphi_2$ & $[1,1]$ \\
			false & true  & $\color{cbpurple}   \varphi_3$ & $[1,1]$ \\
			false & false & $\color{cbgreen}    \varphi_4$ & $[0,2]$ \\
		\end{tabular}
	};

	\node[right = 0.5cm of f2, yshift=-4.5mm] (tab_f2) {
		\begin{tabular}{c|c|c|c}
			$B$   & $C$   & $\phi_2(B,C)$                               & $b$     \\ \hline
			true  & true  & $\color{cborange}   \alpha \cdot \varphi_1$ & $[2,0]$ \\
			true  & false & $\color{cbbluedark} \alpha \cdot \varphi_3$ & $[1,1]$ \\
			false & true  & $\color{cbpurple}   \alpha \cdot \varphi_2$ & $[1,1]$ \\
			false & false & $\color{cbgreen}    \alpha \cdot \varphi_4$ & $[0,2]$ \\
		\end{tabular}
	};

	\node[right = 7cm of B] (buckets) {
		\begin{tabular}{c|c|c}
			$b$     & $\phi_1^{\succ}(b)$                                                          & $\phi_2^{\succ}(b)$                                                                      \\ \hline
			$[2,0]$ & $\langle {\color{cborange}\varphi_1} \rangle$                                & $\langle {\color{cborange}\alpha\cdot\varphi_1} \rangle$                                 \\
			$[1,1]$ & $\langle {\color{cbpurple}\varphi_2}, {\color{cbbluedark}\varphi_3} \rangle$ & $\langle {\color{cbbluedark}\alpha\varphi_3}, {\color{cbpurple}\alpha\varphi_2} \rangle$ \\
			$[0,2]$ & $\langle {\color{cbgreen}\varphi_4} \rangle$                                 & $\langle {\color{cbgreen}\alpha\cdot\varphi_4} \rangle$                                  \\
		\end{tabular}
	};

	\draw (A) -- (f1);
	\draw (B) -- (f1);
	\draw (B) -- (f2);
	\draw (C) -- (f2);
\end{tikzpicture}
	\caption{An \ac{fg} entailing equivalent semantics as the \acp{fg} shown in \cref{fig:example_fg,fig:example_fg_scaled} with corresponding buckets. Note that the arguments of the factor $\phi_2$ are arranged in a different order than in \cref{fig:example_fg,fig:example_fg_scaled} as $B$ appears at position one and $C$ at position two, whereas in the previous examples, $C$ was at position one and $B$ at position two.}
	\label{fig:example_buckets_and_scale}
\end{figure}
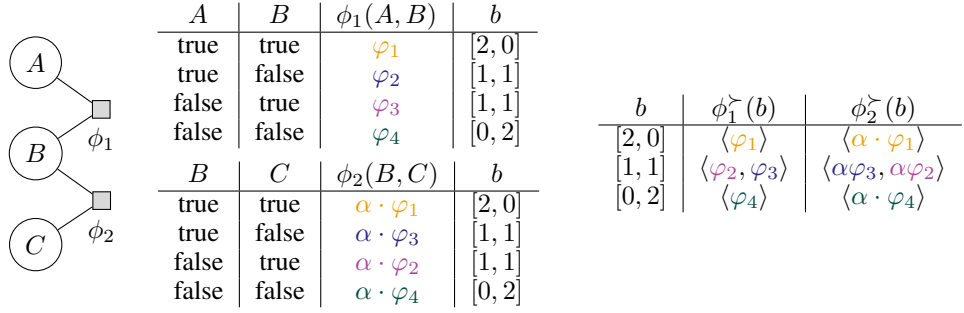
\begin{example}
	Take a look at \cref{fig:example_buckets_and_scale}, which displays an \ac{fg} entailing equivalent semantics as the \acp{fg} shown in \cref{fig:example_fg,fig:example_fg_scaled} with corresponding buckets.
	In this example, $B$ appears at position one and $C$ at position two in $\phi_2$ whereas in \cref{fig:example_fg,fig:example_fg_scaled}, $C$ was at position one and $B$ at position two.
	Both $\phi_1$ and $\phi_2$ entail three buckets $\{(\mathrm{true}, 2), (\mathrm{false}, 0)\}$, $\{(\mathrm{true}, 1), (\mathrm{false}, 1)\}$, $\{(\mathrm{true}, 0), (\mathrm{false}, 2)\}$---or $[2,0]$, $[1,1]$, $[0,2]$ in shorthand notation.
	Every bucket corresponds to at least one assignment, e.g., the bucket $[1,1]$ corresponds to all assignments that contain one $\mathrm{true}$ and one $\mathrm{false}$ value.
	$\phi_1$ maps $[1,1]$ to $\langle \varphi_2, \varphi_3 \rangle$ and $\phi_2$ maps $[1,1]$ to $\langle \alpha\varphi_3, \alpha\varphi_2 \rangle$.
\end{example}
\citet{Luttermann2024d} show that two factors $\phi_1$ and $\phi_2$ are exchangeable (for the setting of $\alpha = 1$) if and only if there exists a permutation of their arguments such that $\phi_1^{\succ}(b) = \phi_2^{\succ}(b)$ for all buckets $b$ entailed by the arguments of $\phi_1$ and $\phi_2$.
The \ac{deft} algorithm exploits this property by checking for each bucket $b$ whether arguments can be rearranged such that $\phi_1^{\succ}(b) = \phi_2^{\succ}(b)$ holds.
The idea is that the potential values in the ordered multisets determine possible permutations of arguments and thus, \ac{deft} avoids iterating over all permutations of arguments.
For example, assuming that $\alpha = 1$ in \cref{fig:example_buckets_and_scale}, we know that $\alpha \varphi_3$ must be located at position one and $\alpha \varphi_2$ at position two in $\phi_2^{\succ}$ to match the order of $\phi_1^{\succ}$.
The corresponding assignments of $\alpha \varphi_2$, i.e., ($\mathrm{false}, \mathrm{true}$), and $\alpha \varphi_3$, i.e., ($\mathrm{true}, \mathrm{false}$), are then used to determine possible positions of arguments to achieve that $\alpha \varphi_3$ is located at position one and $\alpha \varphi_2$ at position two in $\phi_2^{\succ}$.
For now, it is sufficient to understand that identical potential values in $\phi_1^{\succ}$ and $\phi_2^{\succ}$ must be found.
Further technical details about the \ac{deft} algorithm are given in \citep{Luttermann2024d}.
We next generalise the \ac{deft} algorithm such that it is able to handle the setting of $\alpha \neq 1$ as well.
A crucial observation is that the orders of the potential values in the ordered multisets must be identical up to the scaling factor $\alpha$.
\begin{theorem} \label{th:exchangeable_buckets_scaled}
	Let $\phi_1$ and $\phi_2$ denote two factors.
	Then, $\phi_1$ and $\phi_2$ are exchangeable if and only if there exists a permutation of their arguments such that $\phi_1^{\succ}(b) = \alpha \cdot \phi_2^{\succ}(b)$ for all buckets $b$ entailed by the arguments of $\phi_1$ and $\phi_2$, where, by abuse of notation, $\alpha \cdot \phi_2^{\succ}(b)$ denotes the ordered multiset resulting from multiplying each potential value in $\phi_2^{\succ}(b)$ by $\alpha$.
\end{theorem}
\begin{proof}
	For the first direction, it holds that $\phi_1$ and $\phi_2$ are exchangeable.
	According to \cref{def:exchangeable_scaled}, there exists a scalar $\alpha \in \mathbb{R}^+$ and a permutation of $\phi_2$'s arguments such that $\phi_1$ and $\phi_2$ have identical tables of potential mappings up to the scaling factor $\alpha$.
	In consequence, for every bucket $b$ it holds that $\phi_1^{\succ}(b) = \alpha \cdot \phi_2^{\succ}(b)$ since both tables read identical potential values up to $\alpha$ from top to bottom.

	For the second direction, it holds that $\phi_1^{\succ}(b) = \alpha \cdot \phi_2^{\succ}(b)$ for all buckets $b$.
	Converting the buckets back to tables of potential mappings then results in identical tables of potential mappings up to the scalar $\alpha$, which implies that $\phi_1$ and $\phi_2$ are exchangeable.
\end{proof}
Using the insight from \cref{th:exchangeable_buckets_scaled}, the \ac{deft} algorithm can be adapted to search for identical potential values up to scalar $\alpha$ in the ordered multisets of potential values, as shown in \cref{alg:exchangeable_vector}.
\begin{algorithm}[t]
	\caption{Detection of Exchangeable Factors without Normalisation}
	\label{alg:exchangeable_vector}
	\alginput{Two factors $\phi_1(R_1, \dots, R_n)$ and $\phi_2(R'_1, \dots, R'_m)$.} \\
	\algoutput{$\mathrm{true}$ if $\phi_1$ and $\phi_2$ are exchangeable, else $\mathrm{false}$.}
	\begin{algorithmic}[1]
		\If{$n \neq m \lor \mathcal B(\phi_1) \neq \mathcal B(\phi_2)$}
			\State \Return $\mathrm{false}$\;
		\EndIf
		\ForEach{$b \in \mathcal B(\phi_1)$}\Comment{It holds that $\mathcal B(\phi_1) = \mathcal B(\phi_2)$}\;
			\State $\alpha \gets \max(\phi_1^{\succ}(b)) \mathbin{/} \max(\phi_2^{\succ}(b))$\;
			\If{$\alpha$ differs from $\alpha$ for a previous bucket}
				\State \Return $\mathrm{false}$\;
			\EndIf
			\State $C_b \gets$ Possible swaps to obtain $\phi_1^{\succ}(b) = \alpha \cdot \phi_2^{\succ}(b)$\; \label{line:swap}
		\EndForEach
		\If{there exists a swap of $\phi_2$'s arguments in $\bigcap_{b \in \mathcal B(\phi_1)} C_b$ such that $D_{\cos}(\phi_1, \phi_2) = 0$}
			\State \Return $\mathrm{true}$\;
		\Else
			\State \Return $\mathrm{false}$\;
		\EndIf
	\end{algorithmic}
\end{algorithm}
To do so, $\alpha$ is determined first, which is done by computing $\alpha = \max(\phi_1^{\succ}(b)) \mathbin{/} \max(\phi_2^{\succ}(b))$ in bucket $b$.
Note that $\alpha$ must be identical for every bucket $b$, otherwise the two factors cannot be exchangeable.
Having determined $\alpha$, possible permutations of arguments are obtained by looking for identical potential values up to scalar $\alpha$.
Possible permutations of arguments are then verified (or rejected) using the cosine distance between $\phi_1$ and $\phi_2$ after rearranging $\phi_2$'s arguments.
\begin{example}
	Consider again the factors $\phi_1$ and $\phi_2$ depicted in \cref{fig:example_buckets_and_scale}.
	Rearranging $\phi_2$'s arguments such that $B$ is placed at position two and $C$ at position one in $\phi_2$'s argument list yields the table of potential mappings for $\phi_2$ depicted in \cref{fig:example_fg_scaled} and thus results in $D_{\cos}(\phi_1, \phi_2) = 0$.
\end{example}
We next demonstrate the practical effectiveness of \cref{alg:exchangeable_vector} in our empirical evaluation.

\section{Experiments} \label{sec:aacp_experiments}
To assess the effectiveness of \cref{alg:exchangeable_vector} in practice, we compare the run times of running \acl{lve} on the output of \ac{acp} in its original form and of running \acl{lve} on the output of \ac{acp} extended by running \cref{alg:exchangeable_vector} to detect exchangeable factors (\acused{aacp}\ac{aacp}).\footnote{Note that the run time required to perform probabilistic inference on a model directly depends on the graph size of the model, i.e., the presented run times also give information about the compactness of the models.}
\begin{figure}
	\centering
	\resizebox{0.49\textwidth}{!}{
\begin{tikzpicture}[x=1pt,y=1pt]
\definecolor{fillColor}{RGB}{255,255,255}
\path[use as bounding box,fill=fillColor,fill opacity=0.00] (0,0) rectangle (209.58,115.63);
\begin{scope}
\path[clip] (  0.00,  0.00) rectangle (209.58,115.63);
\definecolor{drawColor}{RGB}{255,255,255}
\definecolor{fillColor}{RGB}{255,255,255}

\path[draw=drawColor,line width= 0.6pt,line join=round,line cap=round,fill=fillColor] (  0.00,  0.00) rectangle (209.58,115.63);
\end{scope}
\begin{scope}
\path[clip] ( 39.63, 29.80) rectangle (204.08,110.13);
\definecolor{fillColor}{RGB}{255,255,255}

\path[fill=fillColor] ( 39.63, 29.80) rectangle (204.08,110.13);
\definecolor{drawColor}{RGB}{230,159,0}

\path[draw=drawColor,line width= 0.6pt,line join=round] ( 47.10, 33.62) --
	( 47.39, 34.21) --
	( 47.98, 34.52) --
	( 49.15, 34.17) --
	( 51.49, 35.06) --
	( 56.17, 35.00) --
	( 65.53, 36.08) --
	( 84.26, 36.54) --
	(121.71, 38.72) --
	(196.61, 40.35);
\definecolor{drawColor}{RGB}{46,37,133}

\path[draw=drawColor,line width= 0.6pt,dash pattern=on 2pt off 2pt ,line join=round] ( 47.10, 33.45) --
	( 47.39, 35.80) --
	( 47.98, 38.60) --
	( 49.15, 41.59) --
	( 51.49, 48.33) --
	( 56.17, 57.38) --
	( 65.53, 66.24) --
	( 84.26, 77.53) --
	(121.71, 92.79) --
	(196.61,106.48);
\definecolor{drawColor}{RGB}{230,159,0}
\definecolor{fillColor}{RGB}{230,159,0}

\path[draw=drawColor,line width= 0.4pt,line join=round,line cap=round,fill=fillColor] ( 51.49, 35.06) circle (  1.96);

\path[draw=drawColor,line width= 0.4pt,line join=round,line cap=round,fill=fillColor] ( 47.39, 34.21) circle (  1.96);

\path[draw=drawColor,line width= 0.4pt,line join=round,line cap=round,fill=fillColor] ( 49.15, 34.17) circle (  1.96);

\path[draw=drawColor,line width= 0.4pt,line join=round,line cap=round,fill=fillColor] ( 56.17, 35.00) circle (  1.96);

\path[draw=drawColor,line width= 0.4pt,line join=round,line cap=round,fill=fillColor] ( 47.10, 33.62) circle (  1.96);

\path[draw=drawColor,line width= 0.4pt,line join=round,line cap=round,fill=fillColor] (196.61, 40.35) circle (  1.96);

\path[draw=drawColor,line width= 0.4pt,line join=round,line cap=round,fill=fillColor] ( 47.98, 34.52) circle (  1.96);

\path[draw=drawColor,line width= 0.4pt,line join=round,line cap=round,fill=fillColor] (121.71, 38.72) circle (  1.96);

\path[draw=drawColor,line width= 0.4pt,line join=round,line cap=round,fill=fillColor] ( 65.53, 36.08) circle (  1.96);

\path[draw=drawColor,line width= 0.4pt,line join=round,line cap=round,fill=fillColor] ( 84.26, 36.54) circle (  1.96);
\definecolor{fillColor}{RGB}{46,37,133}

\path[fill=fillColor] ( 49.53, 46.36) --
	( 53.45, 46.36) --
	( 53.45, 50.29) --
	( 49.53, 50.29) --
	cycle;

\path[fill=fillColor] ( 45.43, 33.84) --
	( 49.36, 33.84) --
	( 49.36, 37.77) --
	( 45.43, 37.77) --
	cycle;

\path[fill=fillColor] ( 47.19, 39.63) --
	( 51.11, 39.63) --
	( 51.11, 43.55) --
	( 47.19, 43.55) --
	cycle;

\path[fill=fillColor] ( 54.21, 55.42) --
	( 58.13, 55.42) --
	( 58.13, 59.34) --
	( 54.21, 59.34) --
	cycle;

\path[fill=fillColor] ( 45.14, 31.49) --
	( 49.06, 31.49) --
	( 49.06, 35.42) --
	( 45.14, 35.42) --
	cycle;

\path[fill=fillColor] (194.65,104.52) --
	(198.57,104.52) --
	(198.57,108.44) --
	(194.65,108.44) --
	cycle;

\path[fill=fillColor] ( 46.02, 36.63) --
	( 49.94, 36.63) --
	( 49.94, 40.56) --
	( 46.02, 40.56) --
	cycle;

\path[fill=fillColor] (119.75, 90.83) --
	(123.67, 90.83) --
	(123.67, 94.75) --
	(119.75, 94.75) --
	cycle;

\path[fill=fillColor] ( 63.57, 64.28) --
	( 67.50, 64.28) --
	( 67.50, 68.20) --
	( 63.57, 68.20) --
	cycle;

\path[fill=fillColor] ( 82.30, 75.57) --
	( 86.22, 75.57) --
	( 86.22, 79.49) --
	( 82.30, 79.49) --
	cycle;
\end{scope}
\begin{scope}
\path[clip] (  0.00,  0.00) rectangle (209.58,115.63);
\definecolor{drawColor}{RGB}{0,0,0}

\path[draw=drawColor,line width= 0.6pt,line join=round] ( 39.63, 29.80) --
	( 39.63,110.13);

\path[draw=drawColor,line width= 0.6pt,line join=round] ( 41.05,107.67) --
	( 39.63,110.13) --
	( 38.20,107.67);
\end{scope}
\begin{scope}
\path[clip] (  0.00,  0.00) rectangle (209.58,115.63);
\definecolor{drawColor}{gray}{0.30}

\node[text=drawColor,anchor=base east,inner sep=0pt, outer sep=0pt, scale=  0.88] at ( 34.68, 36.91) {30};

\node[text=drawColor,anchor=base east,inner sep=0pt, outer sep=0pt, scale=  0.88] at ( 34.68, 52.36) {100};

\node[text=drawColor,anchor=base east,inner sep=0pt, outer sep=0pt, scale=  0.88] at ( 34.68, 66.46) {300};

\node[text=drawColor,anchor=base east,inner sep=0pt, outer sep=0pt, scale=  0.88] at ( 34.68, 81.90) {1000};

\node[text=drawColor,anchor=base east,inner sep=0pt, outer sep=0pt, scale=  0.88] at ( 34.68, 96.00) {3000};
\end{scope}
\begin{scope}
\path[clip] (  0.00,  0.00) rectangle (209.58,115.63);
\definecolor{drawColor}{gray}{0.20}

\path[draw=drawColor,line width= 0.6pt,line join=round] ( 36.88, 39.94) --
	( 39.63, 39.94);

\path[draw=drawColor,line width= 0.6pt,line join=round] ( 36.88, 55.39) --
	( 39.63, 55.39);

\path[draw=drawColor,line width= 0.6pt,line join=round] ( 36.88, 69.49) --
	( 39.63, 69.49);

\path[draw=drawColor,line width= 0.6pt,line join=round] ( 36.88, 84.94) --
	( 39.63, 84.94);

\path[draw=drawColor,line width= 0.6pt,line join=round] ( 36.88, 99.03) --
	( 39.63, 99.03);
\end{scope}
\begin{scope}
\path[clip] (  0.00,  0.00) rectangle (209.58,115.63);
\definecolor{drawColor}{RGB}{0,0,0}

\path[draw=drawColor,line width= 0.6pt,line join=round] ( 39.63, 29.80) --
	(204.08, 29.80);

\path[draw=drawColor,line width= 0.6pt,line join=round] (201.62, 28.38) --
	(204.08, 29.80) --
	(201.62, 31.23);
\end{scope}
\begin{scope}
\path[clip] (  0.00,  0.00) rectangle (209.58,115.63);
\definecolor{drawColor}{gray}{0.20}

\path[draw=drawColor,line width= 0.6pt,line join=round] ( 46.81, 27.05) --
	( 46.81, 29.80);

\path[draw=drawColor,line width= 0.6pt,line join=round] ( 83.38, 27.05) --
	( 83.38, 29.80);

\path[draw=drawColor,line width= 0.6pt,line join=round] (119.95, 27.05) --
	(119.95, 29.80);

\path[draw=drawColor,line width= 0.6pt,line join=round] (156.53, 27.05) --
	(156.53, 29.80);

\path[draw=drawColor,line width= 0.6pt,line join=round] (193.10, 27.05) --
	(193.10, 29.80);
\end{scope}
\begin{scope}
\path[clip] (  0.00,  0.00) rectangle (209.58,115.63);
\definecolor{drawColor}{gray}{0.30}

\node[text=drawColor,anchor=base,inner sep=0pt, outer sep=0pt, scale=  0.88] at ( 46.81, 18.79) {0};

\node[text=drawColor,anchor=base,inner sep=0pt, outer sep=0pt, scale=  0.88] at ( 83.38, 18.79) {250};

\node[text=drawColor,anchor=base,inner sep=0pt, outer sep=0pt, scale=  0.88] at (119.95, 18.79) {500};

\node[text=drawColor,anchor=base,inner sep=0pt, outer sep=0pt, scale=  0.88] at (156.53, 18.79) {750};

\node[text=drawColor,anchor=base,inner sep=0pt, outer sep=0pt, scale=  0.88] at (193.10, 18.79) {1000};
\end{scope}
\begin{scope}
\path[clip] (  0.00,  0.00) rectangle (209.58,115.63);
\definecolor{drawColor}{RGB}{0,0,0}

\node[text=drawColor,anchor=base,inner sep=0pt, outer sep=0pt, scale=  1.00] at (121.85,  7.44) {$d$};
\end{scope}
\begin{scope}
\path[clip] (  0.00,  0.00) rectangle (209.58,115.63);
\definecolor{drawColor}{RGB}{0,0,0}

\node[text=drawColor,rotate= 90.00,anchor=base,inner sep=0pt, outer sep=0pt, scale=  1.00] at ( 12.39, 69.97) {time (ms)};
\end{scope}
\begin{scope}
\path[clip] (  0.00,  0.00) rectangle (209.58,115.63);

\path[] ( 41.35, 78.13) rectangle ( 97.11,118.04);
\end{scope}
\begin{scope}
\path[clip] (  0.00,  0.00) rectangle (209.58,115.63);
\definecolor{drawColor}{RGB}{230,159,0}

\path[draw=drawColor,line width= 0.6pt,line join=round] ( 48.29,105.31) -- ( 59.86,105.31);
\end{scope}
\begin{scope}
\path[clip] (  0.00,  0.00) rectangle (209.58,115.63);
\definecolor{drawColor}{RGB}{230,159,0}
\definecolor{fillColor}{RGB}{230,159,0}

\path[draw=drawColor,line width= 0.4pt,line join=round,line cap=round,fill=fillColor] ( 54.07,105.31) circle (  1.96);
\end{scope}
\begin{scope}
\path[clip] (  0.00,  0.00) rectangle (209.58,115.63);
\definecolor{drawColor}{RGB}{46,37,133}

\path[draw=drawColor,line width= 0.6pt,dash pattern=on 2pt off 2pt ,line join=round] ( 48.29, 90.86) -- ( 59.86, 90.86);
\end{scope}
\begin{scope}
\path[clip] (  0.00,  0.00) rectangle (209.58,115.63);
\definecolor{fillColor}{RGB}{46,37,133}

\path[fill=fillColor] ( 52.11, 88.89) --
	( 56.04, 88.89) --
	( 56.04, 92.82) --
	( 52.11, 92.82) --
	cycle;
\end{scope}
\begin{scope}
\path[clip] (  0.00,  0.00) rectangle (209.58,115.63);
\definecolor{drawColor}{RGB}{0,0,0}

\node[text=drawColor,anchor=base west,inner sep=0pt, outer sep=0pt, scale=  0.80] at ( 66.80,102.55) {$\alpha$-ACP};
\end{scope}
\begin{scope}
\path[clip] (  0.00,  0.00) rectangle (209.58,115.63);
\definecolor{drawColor}{RGB}{0,0,0}

\node[text=drawColor,anchor=base west,inner sep=0pt, outer sep=0pt, scale=  0.80] at ( 66.80, 88.10) {ACP};
\end{scope}
\end{tikzpicture}}
	\resizebox{0.49\textwidth}{!}{\input{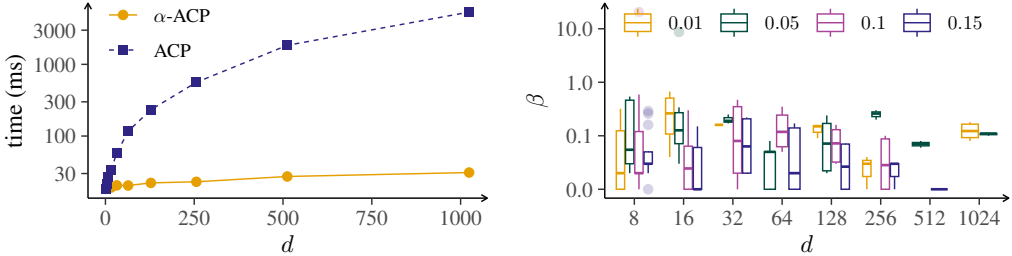}}
	\caption{Average query times of \acl{lve} on the output of \ac{acp} and \ac{aacp} (left) and the average number $\beta$ of queries after which the offline overhead of \ac{aacp} amortises (right).}
	\label{fig:aacp_plot_main}
\end{figure}
For our experiments, we generate \acp{fg} containing between $2d+1$ and $d\cdot\lfloor\log_2(d)\rfloor+2d+1$ \acp{rv} as well as between $2d$ and $d\cdot\lfloor\log_2(d)\rfloor+d+1$ factors, where the parameter $d \in \{2, 4, 8, 16, 32, 64, 128, 256, 512, 1024\}$ controls the size of the \ac{fg}.
In every \ac{fg}, a proportion of $p \in \{0.01, 0.05, 0.1, 0.15\}$ of the factors is scaled by a scalar $\alpha \in \{1, \ldots, 10\}$ (chosen uniformly at random).
For each choice of $d$, we pose three to four queries to each \ac{fg} and report the average run time over all queries.
\Cref{fig:aacp_plot_main} displays the results.
The left plot shows the average run times of \acl{lve} on the output of \ac{acp} and \ac{aacp}, respectively.
As expected, \acl{lve} runs significantly faster and is able to handle larger values of $d$ if \ac{aacp} instead of \ac{acp} is applied.
Since \ac{aacp} is able to detect exchangeable factors that are scaled differently, it detects strictly more symmetries than \ac{acp} resulting in a more compact model and hence in faster inference times.
Clearly, the speedup depends on the proportion $p$ of scaled factors and thus, we provide additional experimental results for each individual choice of $p$ in \cref{appendix:more_eval}.

The boxplot on the right in \cref{fig:aacp_plot_main} displays the average number $\beta$ of queries after which the additional offline overhead of \ac{aacp} compared to \ac{acp} amortises.\footnote{The dots in the boxplot on the right in \cref{fig:aacp_plot_main} represent outliers.}
In particular, it holds that $\beta = \Delta_o \mathbin{/} \Delta_g$, where $\Delta_o$ denotes the offline overhead of \ac{aacp} (i.e., the difference of the offline run times required by \ac{aacp} and \ac{acp}) and $\Delta_g$ denotes the online gain (i.e., the difference of the times required by \acl{lve} run on the output of \ac{acp} and \ac{aacp}).
In other words, after $\beta$ queries, the additional time needed by \ac{aacp} to construct the \ac{pfg} is saved due to faster inference times.
Negative values for $\beta$ are not displayed (hence the missing boxes for some $p$) as there is no overhead in these cases.
The boxplot shows a box for each choice of $p$ for every $d \geq 8$ and it becomes clear that the median value for $\beta$ is always smaller than one.
Apart from a single outlier, all values of $\beta$ are smaller than ten.
Thus, after a maximum of ten queries, the additional offline overhead of \ac{aacp} amortises, showing that \ac{aacp} works efficiently as it introduces almost no overhead.

\section{Conclusion}
In this paper, we generalise the \ac{acp} algorithm to detect exchangeable factors independent of the scale of their potentials without the requirement of normalising the potentials.
Our proposed approach allows for arbitrary scalars and makes use of vector representations in combination with collinearity checks to efficiently detect exchangeable factors independent of their scale.
We show that our approach maintains equivalent semantics and at the same time yields a more compact representation by detecting strictly more symmetries than the original \ac{acp} algorithm, thereby speeding up inference.

\section*{Acknowledgements}
This work is funded by the BMBF project AnoMed 16KISA057.

\bibliographystyle{unsrtnat}
\bibliography{references.bib}

\clearpage
\appendix

\section{Probabilistic Inference in More Detail} \label{appendix:example_inference}
The task of probabilistic inference describes the computation of marginal distributions of \acp{rv} given observations for other \acp{rv}.
In other words, probabilistic inference refers to query answering, where a query is defined as follows.
\begin{definition}[Query]
	A \emph{query} $P(Q \mid E_1 = e_1, \ldots, E_k = e_k)$ consists of a query term $Q$ and a set of events $\{E_j = e_j\}_{j=1}^{k}$ (called evidence), where $Q$ and $E_1, \ldots, E_k$ are \acp{rv}.
	To query a specific probability instead of a probability distribution, the query term is an event $Q = q$.
\end{definition}
\begin{example}[Probabilistic Inference] \label{ex:eacp_lifting_idea}
	Take a look again at the \ac{fg} depicted in \cref{fig:example_fg} and assume we want to answer the query $P(B = \true)$.
	\begin{align}
		P(B = \true)
		&= \sum_{a \in \range{A}} \sum_{c \in \range{C}} P(A = a, B = \true, C = c) \\
		&= \frac{1}{Z} \sum_{a \in \range{A}} \sum_{c \in \range{C}} \phi_1(a, \true) \cdot \phi_2(c, \true) \\
		&= \frac{1}{Z} \Big( \varphi_1 \varphi_1 + \varphi_1 \varphi_3 + \varphi_3 \varphi_1 + \varphi_3 \varphi_3 \Big).
	\end{align}
	Since $\phi_1(A,B)$ and $\phi_2(C,B)$ are exchangeable (i.e., it holds that $\phi_1(a, \true) = \phi_2(c, \true)$ for all assignments where $a = c$), we can exploit this symmetry to simplify the computation and obtain
	\begin{align}
		P(B = \true)
		&= \frac{1}{Z} \sum_{a \in \range{A}} \sum_{c \in \range{C}} \phi_1(a, \true) \cdot \phi_2(c, \true) \\
		&= \frac{1}{Z} \sum_{a \in \range{A}} \phi_1(a, \true) \sum_{c \in \range{C}} \phi_2(c, \true) \\
		&= \frac{1}{Z} \Bigg( \sum_{a \in \range{A}} \phi_1(a, \true) \Bigg)^2 \\
		&= \frac{1}{Z} \Bigg( \sum_{c \in \range{C}} \phi_2(c, \true) \Bigg)^2 \\
		&= \frac{1}{Z} \Big( \varphi_1 + \varphi_3 \Big)^2.
	\end{align}
	This example illustrates the idea of using a representative of indistinguishable objects for computations (here, either $A$ or $C$ can be chosen as a representative for the group consisting of $A$ and $C$).
	This idea can be generalised to groups of $k$ indistinguishable objects to significantly reduce the computational effort when answering queries.
\end{example}

\section{Formal Description of the Advanced Colour Passing Algorithm} \label{appendix:acp}
The \ac{acp} algorithm~\citep{Luttermann2024a} extends the \acl{cp} algorithm~\citep{Kersting2009a,Ahmadi2013a}, thereby solving the problem of constructing a lifted representation in form of a \ac{pfg} from a given \ac{fg}.
The idea of \ac{acp} is to first find symmetric subgraphs in a propositional \ac{fg} and then group together these symmetric subgraphs.
\Ac{acp} searches for symmetries based on potentials of factors, on ranges and evidence of \acp{rv}, as well as on the graph structure by employing a colour passing routine.
A formal description of the \ac{acp} algorithm is depicted in \cref{alg:acp}.
We next explain the steps undertaken by \ac{acp} in more detail.

\begin{algorithm}
	\caption{Advanced Colour Passing (reprinted from \citealp{Luttermann2024a})}
	\label{alg:acp}
	\alginput{An \ac{fg} $G$ with \acp{rv} $\boldsymbol R = \{R_1, \ldots, R_n\}$, factors $\boldsymbol \Phi = \{\phi_1, \ldots, \phi_m\}$, and evidence \\\hspace*{\algorithmicindent} $\boldsymbol E = \{R_1 = r_1, \ldots, R_k = r_k\}$.} \\
	\algoutput{A lifted representation $G'$ in form of a \ac{pfg} entailing equivalent semantics as $G$.}
	\begin{algorithmic}[1]
		\State Assign each $R_i$ a colour according to $\mathcal R(R_i)$ and $\boldsymbol E$\;
		\State Assign each $\phi_i$ a colour according to order-independent potentials and rearrange arguments accordingly\; \label{line:acp_factor_colour_init}
		\Repeat
			\For{each factor $\phi \in \boldsymbol \Phi$}
				\State $signature_{\phi} \gets [\,]$\;
				\For{each \ac{rv} $R \in neighbours(G, \phi)$}
					\Comment{In order of appearance in $\phi$}\;
					\State $append(signature_{\phi}, R.colour)$\;
				\EndFor
				\State $append(signature_{\phi}, \phi.colour)$\;
			\EndFor
			\State Group together all $\phi$s with the same signature\;
			\State Assign each such cluster a unique colour\;
			\State Set $\phi.colour$ correspondingly for all $\phi$s\;
			\For{each \ac{rv} $R \in \boldsymbol R$}
				\State $signature_{R} \gets [\,]$\;
				\For{each factor $\phi \in neighbours(G, R)$}
					\If{$\phi$ is commutative w.r.t.\ $\boldsymbol S$ and $R \in \boldsymbol S$}
						\State $append(signature_{R}, (\phi.colour, 0))$\;
					\Else
						\State $append(signature_{R}, (\phi.colour, p(R, \phi)))$\;
					\EndIf
				\EndFor
				\State Sort $signature_{R}$ according to colour\;
				\State $append(signature_{R}, R.colour)$\;
			\EndFor
			\State Group together all $R$s with the same signature\;
			\State Assign each such cluster a unique colour\;
			\State Set $R.colour$ correspondingly for all $R$s\;
		\Until{grouping does not change}
		\State $G' \gets$ construct \acs{pfg} from groupings\;
	\end{algorithmic}
\end{algorithm}

\Ac{acp} begins with the colour assignment to variable nodes, meaning that all \acp{rv} having the same range and observed event are assigned the same colour.
\Ac{rv} with different ranges or different observed events are assigned distinct colours since those \ac{rv} do not \enquote{behave in the same way}.
Then, in \cref{line:acp_factor_colour_init}, \ac{acp} assigns colours to factor nodes such that exchangeable factors encoding equivalent semantics according to \cref{def:exchangeable_scaled} are assigned the same colour.
In its original form, \ac{acp} uses the \ac{deft} algorithm~\citep{Luttermann2024d} to efficiently detect exchangeable factors and to assign colours accordingly.
The \ac{deft} algorithm deployed within \ac{acp}, however, is not able to handle factors with differently scaled potentials (i.e., it detects exchangeable factors according to \cref{def:exchangeable_scaled} only for $\alpha = 1$) and thus, our extension replaces the \ac{deft} algorithm by \cref{alg:exchangeable_vector} to detect exchangeable factors independent of the scale of their potentials in \cref{line:acp_factor_colour_init} of \cref{alg:acp}.

After the initial colour assignments, \ac{acp} runs a colour passing routine.
\Ac{acp} first passes the colours from each variable node to its neighbouring factor nodes and after a recolouring step to reduce communication overhead, each factor node $\phi$ sends its colour as well as the position $p(R, \phi)$ of $R$ in $\phi$'s argument list to all of its neighbouring variable nodes $R$, again followed by a recolouring step.
The procedure is then iterated until the identified groups do not change anymore.
In the end, the determined groups are then used to construct a \ac{pfg} entailing equivalent semantics as the input \ac{fg} $G$.

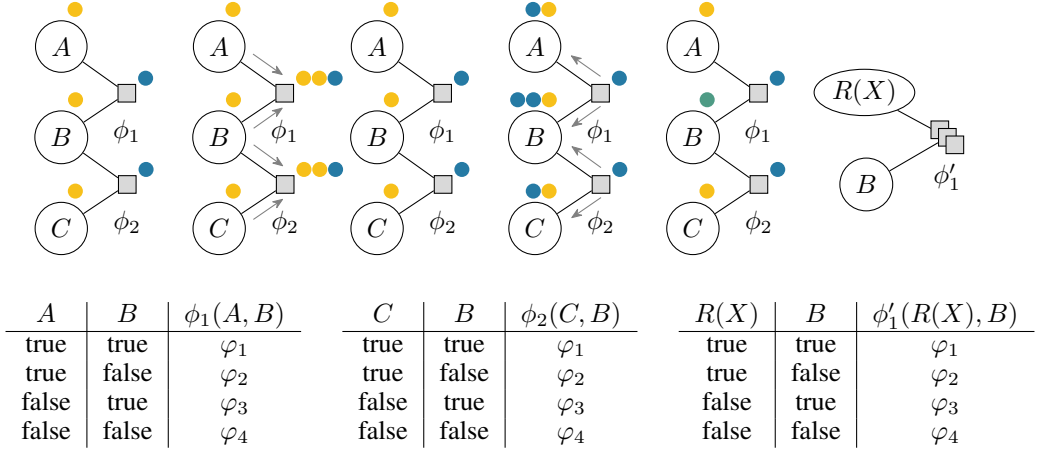
\begin{figure*}[t]
	\centering
	\begin{tikzpicture}[label distance=1mm]
	\node[circle, draw] (A) {$A$};
	\node[circle, draw] (B) [below = 0.5cm of A] {$B$};
	\node[circle, draw] (C) [below = 0.5cm of B] {$C$};
	\factor{below right}{A}{0.25cm and 0.5cm}{270}{$\phi_1$}{f1}
	\factor{below right}{B}{0.25cm and 0.5cm}{270}{$\phi_2$}{f2}

	\nodecolorshift{myyellow}{A}{Acol}{-2.1mm}{1mm}
	\nodecolorshift{myyellow}{B}{Bcol}{-2.1mm}{1mm}
	\nodecolorshift{myyellow}{C}{Ccol}{-2.1mm}{1mm}

	\factorcolor{myblue}{f1}{f1col}
	\factorcolor{myblue}{f2}{f2col}

	\draw (A) -- (f1);
	\draw (B) -- (f1);
	\draw (B) -- (f2);
	\draw (C) -- (f2);

	\node[circle, draw, right = 1.4cm of A] (A1) {$A$};
	\node[circle, draw, below = 0.5cm of A1] (B1) {$B$};
	\node[circle, draw, below = 0.5cm of B1] (C1) {$C$};
	\factor{below right}{A1}{0.25cm and 0.5cm}{270}{$\phi_1$}{f1_1}
	\factor{below right}{B1}{0.25cm and 0.5cm}{270}{$\phi_2$}{f2_1}

	\nodecolorshift{myyellow}{A1}{A1col}{-2.1mm}{1mm}
	\nodecolorshift{myyellow}{B1}{B1col}{-2.1mm}{1mm}
	\nodecolorshift{myyellow}{C1}{C1col}{-2.1mm}{1mm}

	\factorcolor{myyellow}{f1_1}{f1_1col1}
	\factorcolorshift{myyellow}{f1_1}{f1_1col2}{2.1mm}
	\factorcolorshift{myblue}{f1_1}{f1_1col3}{4.2mm}
	\factorcolor{myyellow}{f2_1}{f2_1col1}
	\factorcolorshift{myyellow}{f2_1}{f2_1col2}{2.1mm}
	\factorcolorshift{myblue}{f2_1}{f2_1col3}{4.2mm}

	\coordinate[right=0.1cm of A1, yshift=-0.1cm] (CA1);
	\coordinate[above=0.2cm of f1_1, yshift=-0.1cm] (Cf1_1);
	\coordinate[right=0.1cm of B1, yshift=0.12cm] (CB1);
	\coordinate[right=0.1cm of B1, yshift=-0.1cm] (CB1_1);
	\coordinate[below=0.2cm of f1_1, yshift=0.15cm] (Cf1_1b);
	\coordinate[above=0.2cm of f2_1, yshift=-0.1cm] (Cf2_1);
	\coordinate[right=0.1cm of C1, yshift=0.12cm] (CC1);
	\coordinate[below=0.2cm of f2_1, yshift=0.15cm] (Cf2_1b);

	\begin{pgfonlayer}{bg}
		\draw (A1) -- (f1_1);
		\draw [arc, gray] (CA1) -- (Cf1_1);
		\draw (B1) -- (f1_1);
		\draw [arc, gray] (CB1) -- (Cf1_1b);
		\draw (B1) -- (f2_1);
		\draw [arc, gray] (CB1_1) -- (Cf2_1);
		\draw (C1) -- (f2_1);
		\draw [arc, gray] (CC1) -- (Cf2_1b);
	\end{pgfonlayer}

	\node[circle, draw, right = 1.4cm of A1] (A2) {$A$};
	\node[circle, draw, below = 0.5cm of A2] (B2) {$B$};
	\node[circle, draw, below = 0.5cm of B2] (C2) {$C$};
	\factor{below right}{A2}{0.25cm and 0.5cm}{270}{$\phi_1$}{f1_2}
	\factor{below right}{B2}{0.25cm and 0.5cm}{270}{$\phi_2$}{f2_2}

	\nodecolorshift{myyellow}{A2}{A2col}{-2.1mm}{1mm}
	\nodecolorshift{myyellow}{B2}{B2col}{-2.1mm}{1mm}
	\nodecolorshift{myyellow}{C2}{C2col}{-2.1mm}{1mm}

	\factorcolor{myblue}{f1_2}{f1_2col1}
	\factorcolor{myblue}{f2_2}{f2_2col1}

	\draw (A2) -- (f1_2);
	\draw (B2) -- (f1_2);
	\draw (B2) -- (f2_2);
	\draw (C2) -- (f2_2);

	\node[circle, draw, right = 1.4cm of A2] (A3) {$A$};
	\node[circle, draw, below = 0.5cm of A3] (B3) {$B$};
	\node[circle, draw, below = 0.5cm of B3] (C3) {$C$};
	\factor{below right}{A3}{0.25cm and 0.5cm}{270}{$\phi_1$}{f1_3}
	\factor{below right}{B3}{0.25cm and 0.5cm}{270}{$\phi_2$}{f2_3}

	\nodecolorshift{myblue}{A3}{A3col1}{-4.2mm}{1mm}
	\nodecolorshift{myyellow}{A3}{A3col2}{-2.1mm}{1mm}
	\nodecolorshift{myblue}{B3}{B3col1}{-6.3mm}{1mm}
	\nodecolorshift{myblue}{B3}{B3col2}{-4.2mm}{1mm}
	\nodecolorshift{myyellow}{B3}{B3col3}{-2.1mm}{1mm}
	\nodecolorshift{myblue}{C3}{C3col1}{-4.2mm}{1mm}
	\nodecolorshift{myyellow}{C3}{C3col2}{-2.1mm}{1mm}

	\factorcolor{myblue}{f1_3}{f1_3col1}
	\factorcolor{myblue}{f2_3}{f2_3col1}

	\coordinate[right=0.1cm of A3, yshift=-0.1cm] (CA3);
	\coordinate[above=0.2cm of f1_3, yshift=-0.1cm] (Cf1_3);
	\coordinate[right=0.1cm of B3, yshift=0.12cm] (CB3);
	\coordinate[right=0.1cm of B3, yshift=-0.1cm] (CB1_3);
	\coordinate[below=0.2cm of f1_3, yshift=0.15cm] (Cf1_3b);
	\coordinate[above=0.2cm of f2_3, yshift=-0.1cm] (Cf2_3);
	\coordinate[right=0.1cm of C3, yshift=0.12cm] (CC3);
	\coordinate[below=0.2cm of f2_3, yshift=0.15cm] (Cf2_3b);

	\begin{pgfonlayer}{bg}
		\draw (A3) -- (f1_3);
		\draw [arc, gray] (Cf1_3) -- (CA3);
		\draw (B3) -- (f1_3);
		\draw [arc, gray] (Cf1_3b) -- (CB3);
		\draw (B3) -- (f2_3);
		\draw [arc, gray] (Cf2_3) -- (CB1_3);
		\draw (C3) -- (f2_3);
		\draw [arc, gray] (Cf2_3b) -- (CC3);
	\end{pgfonlayer}

	\node[circle, draw, right = 1.4cm of A3] (A4) {$A$};
	\node[circle, draw, below = 0.5cm of A4] (B4) {$B$};
	\node[circle, draw, below = 0.5cm of B4] (C4) {$C$};
	\factor{below right}{A4}{0.25cm and 0.5cm}{270}{$\phi_1$}{f1_4}
	\factor{below right}{B4}{0.25cm and 0.5cm}{270}{$\phi_2$}{f2_4}

	\nodecolorshift{myyellow}{A4}{A4col}{-2.1mm}{1mm}
	\nodecolorshift{mygreen}{B4}{B4col}{-2.1mm}{1mm}
	\nodecolorshift{myyellow}{C4}{C4col}{-2.1mm}{1mm}

	\factorcolor{myblue}{f1_4}{f1_4col1}
	\factorcolor{myblue}{f2_4}{f2_4col1}

	\draw (A4) -- (f1_4);
	\draw (B4) -- (f1_4);
	\draw (B4) -- (f2_4);
	\draw (C4) -- (f2_4);

	\pfs{right}{B4}{2.9cm}{270}{$\phi'_1$}{f12a}{f12}{f12b}

	\node[ellipse, inner sep = 1.2pt, draw, above left = 0.25cm and 0.5cm of f12] (AC) {$R(X)$};
	\node[circle, draw] (B) [below left = 0.25cm and 0.7cm of f12] {$B$};

	\begin{pgfonlayer}{bg}
		\draw (AC) -- (f12);
		\draw (B) -- (f12);
	\end{pgfonlayer}

	\node[below = 0.5cm of C2, xshift=1.5cm] (tab_f2) {
		\begin{tabular}{c|c|c}
			$C$   & $B$   & $\phi_2(C,B)$ \\ \hline
			true  & true  & $\varphi_1$ \\
			true  & false & $\varphi_2$ \\
			false & true  & $\varphi_3$ \\
			false & false & $\varphi_4$ \\
		\end{tabular}
	};

	\node[left = 0.3cm of tab_f2] (tab_f1) {
		\begin{tabular}{c|c|c}
			$A$   & $B$   & $\phi_1(A,B)$ \\ \hline
			true  & true  & $\varphi_1$ \\
			true  & false & $\varphi_2$ \\
			false & true  & $\varphi_3$ \\
			false & false & $\varphi_4$ \\
		\end{tabular}
	};

	\node[right = 0.3cm of tab_f2] (tab_f12) {
		\begin{tabular}{c|c|c}
			$R(X)$   & $B$   & $\phi'_1(R(X),B)$ \\ \hline
			true  & true  & $\varphi_1$ \\
			true  & false & $\varphi_2$ \\
			false & true  & $\varphi_3$ \\
			false & false & $\varphi_4$ \\
		\end{tabular}
	};
\end{tikzpicture}
	\caption{A visualisation of the steps undertaken by \cref{alg:acp} on an input \ac{fg} with only Boolean \acp{rv} and no evidence (left). Colours are first passed from variable nodes to factor nodes, followed by a recolouring, and then passed back from factor nodes to variable nodes, again followed by a recolouring. The colour passing procedure is iterated until convergence and the resulting \ac{pfg} is depicted on the right. This figure is reprinted from \citep{Luttermann2024a}.}
	\label{fig:acp_example}
\end{figure*}

\Cref{fig:acp_example} illustrates \ac{acp} on an example input \ac{fg}.
In this example, all \acp{rv} are Boolean and there is no evidence available (i.e., $\boldsymbol E = \emptyset$).
Initially, \ac{acp} assigns all \acp{rv} the same colour (e.g., $\mathrm{yellow}$) because they have the same range (Boolean) and evidence (no evidence at all).
As $\phi_1$ and $\phi_2$ encode equivalent semantics (they represent identical potentials), they are assigned the same colour (e.g., $\mathrm{blue}$).
The colours are then passed from variable nodes to factor nodes and as each factor has two neighbouring \acp{rv}, all factors receive the same messages.
After recolouring the factors, their colour assignments remain identical to their initial assignments as they all received the same message.
The purpose of the recolouring is mainly to reduce communication overhead.
Afterwards, the factor nodes send their colours to their neighbouring variable nodes.
Each message from a factor to a \ac{rv} contains the position of the \ac{rv} in the factor if the factor is not commutative\footnote{A commutative factor is a factor which maps its arguments to the same potential value independent of the order of a (sub)set of its assigned values~\citep{Luttermann2024f}.}, else the position is replaced by zero.
For simplicity, there is no commutative factor in the example shown in \cref{fig:acp_example}.
Thus, $A$ receives a message $(\mathrm{blue},1)$ from $\phi_1$, $B$ receives a message $(\mathrm{blue},2)$ from $\phi_1$ as well as a message $(\mathrm{blue},2)$ from $\phi_2$, and $C$ receives a message $(\mathrm{blue},1)$ from $\phi_2$.
Consequently, $A$ and $C$ receive identical messages (positions are not shown in \cref{fig:acp_example}) and after the recolouring step, $A$ and $C$ share the same colour while $B$ is assigned a different colour.
The groupings do not change in further iterations and the resulting \ac{pfg} is shown on the right (where $X$ has domain $\{A,C\}$).

For more details about the colour passing routine and the grouping of nodes, we refer the reader to \citep{Luttermann2024a}.
The authors also demonstrate the benefits of a lifted representation in terms of speedup for probabilistic inference.

\section{Missing Proofs} \label{appendix:missing_proofs}
\cosineDistanceMeasureTheorem*
\begin{proof}
	If $\phi_1$ and $\phi_2$ are exchangeable, there exists a scalar $\alpha \in \mathbb{R}^+$ and a permutation $\pi$ of $\{1, \ldots, n\}$ such that for all $r_1, \ldots, r_n \in \times_{i=1}^n \range{R_i}$ it holds that $\phi_1(r_1, \ldots, r_n) = \alpha \cdot \phi_2(r_{\pi(1)}, \ldots, r_{\pi(n)})$.
	Without loss of generality, assume that the arguments of $\phi_2$ are rearranged such that for all $r_1, \ldots, r_n \in \times_{i=1}^n \range{R_i}$ it holds that $\phi_1(r_1, \ldots, r_n) = \alpha \cdot \phi_2(r_1, \ldots, r_n)$.
	Then, entering $\phi_1(r_1, \ldots, r_n) = \alpha \cdot \phi_2(r_1, \ldots, r_n)$ into \cref{eq:cosine_distance} yields
	\begingroup
	\allowdisplaybreaks
	\begin{align}
		D_{\cos}(\phi_1, \phi_2)
		&= 1 - \frac{\alpha \cdot \sum\limits_{\boldsymbol a \in \times_{i=1}^n \range{R_i}} \phi_2(\boldsymbol a)^2}{\sqrt{\alpha^2 \cdot \sum\limits_{\boldsymbol a \in \times_{i=1}^n \range{R_i}} \phi_2(\boldsymbol a)^2} \cdot \sqrt{\sum\limits_{\boldsymbol a \in \times_{i=1}^n \range{R_i}} \phi_2(\boldsymbol a)^2}} \\
		&= 1 - \frac{\alpha \cdot \sum\limits_{\boldsymbol a \in \times_{i=1}^n \range{R_i}} \phi_2(\boldsymbol a)^2}{\sqrt{\alpha^2} \cdot \sqrt{\sum\limits_{\boldsymbol a \in \times_{i=1}^n \range{R_i}} \phi_2(\boldsymbol a)^2} \cdot \sqrt{\sum\limits_{\boldsymbol a \in \times_{i=1}^n \range{R_i}} \phi_2(\boldsymbol a)^2}} \\
		&= 1 - \frac{\alpha \cdot \sum\limits_{\boldsymbol a \in \times_{i=1}^n \range{R_i}} \phi_2(\boldsymbol a)^2}{\alpha \cdot \sum\limits_{\boldsymbol a \in \times_{i=1}^n \range{R_i}} \phi_2(\boldsymbol a)^2} \\
		&= 0.
	\end{align}
	\endgroup
\end{proof}

\section{Checking Collinearity of Vectors without Division Operations} \label{appendix:collinearity_wo_division}
Let $\vec \phi_1 = (\varphi_1, \ldots, \varphi_n)$ and $\vec \phi_2 = (\psi_1, \ldots, \psi_n)$ denote two vector representations of factors $\phi_1$ and $\phi_2$.
Our goal is to check whether $\vec \phi_1$ and $\vec \phi_2$ are collinear, that is, whether there exists a scalar $\alpha$ such that $\vec \phi_1 = \alpha \cdot \vec \phi_2$, without using division or square root operations.
By doing so, we avoid floating point arithmetic issues.
In particular, if the potential values $\varphi_1, \ldots, \varphi_n$ and $\psi_1, \ldots, \psi_n$ are integers (which is the case if the initial \ac{fg} is learned from data by counting occurrences of specific assignments), then no floating point numbers are involved in our calculations during exchangeability checks.
Note that, even if we apply the cosine distance to check for collinearity, the model itself can still consist only of integer potential values and only the check for exchangeability involves floating point arithmetics in this case (as opposed to a model where potential values are normalised at the beginning, which results in floating point numbers for its potential values).

We next provide the technical details to check whether $\vec \phi_1 = \alpha \cdot \vec \phi_2$ holds by using only multiplication operations.
If $\vec \phi_1$ and $\vec \phi_2$ are collinear, we have $\varphi_1 = \alpha \cdot \psi_1, \ldots, \varphi_n = \alpha \cdot \psi_n$ and hence $\alpha = \varphi_i \mathbin{/} \psi_i$ for all $i \in \{1, \ldots, n\}$.
Since $\alpha = \varphi_i \mathbin{/} \psi_i$ holds for all $i \in \{1, \ldots, n\}$, we can enter $\alpha = \varphi_1 \mathbin{/} \psi_1$ into the equation $\varphi_i = \alpha \cdot \psi_i$ and obtain
\begingroup
\allowdisplaybreaks
\begin{align}
	\varphi_i &= \frac{\varphi_1}{\psi_1} \cdot \psi_i \\
	\Leftrightarrow \varphi_i \cdot \psi_1 &= \varphi_1 \cdot \psi_i.
\end{align}
\endgroup
In consequence, we can check whether $\vec \phi_1$ and $\vec \phi_2$ are collinear by verifying whether $\varphi_i \cdot \psi_1 = \varphi_1 \cdot \psi_i$ holds for all $i \in \{1, \ldots, n\}$.
Verifying this equation involves only multiplication operations and hence avoids floating point arithmetics as much as possible.

To supplement \cref{sec:permuted_arguments_scaled}, we next show that the aforementioned approach can also be extended to handle permutations of arguments.
In case there exists a permutation $\pi$ of $\{1, \ldots, n\}$ such that for all $r_1, \ldots, r_n \in \times_{i=1}^n \range{R_i}$ it holds that $\phi_1(r_1, \ldots, r_n) = \alpha \cdot \phi_2(r_{\pi(1)}, \ldots, r_{\pi(n)})$, it does not necessarily hold that $\alpha = \varphi_i \mathbin{/} \psi_i$ for $i \in \{1, \ldots, n\}$.
However, we know that $\alpha = \max_{i \in \{1, \ldots, n\}} \varphi_i \mathbin{/} \max_{i \in \{1, \ldots, n\}} \psi_i$ (analogously for $\min$ instead of $\max$).
At the same time, we know that in the first bucket, there is only a single potential value, which remains the same for all possible permutations $\pi$ of $\{1, \ldots, n\}$ because the corresponding assignment assigns all arguments the same range value (e.g., $\mathrm{true}$).
We thus have $\alpha = \varphi_1 \mathbin{/} \psi_1$ again, regardless of the order of arguments.
Consequently, when searching for possible swaps to obtain $\phi_1^{\succ}(b) = \alpha \cdot \phi_2^{\succ}(b)$ in \cref{line:swap} of \cref{alg:exchangeable_vector}, we can find identical potential values up to the scalar $\alpha$ for each $\varphi_i \in \phi_1^{\succ}(b)$ in $\phi_2^{\succ}(b)$ by iterating over all potential values $\psi_i \in \phi_2^{\succ}(b)$ and checking whether $\varphi_i \cdot \psi_1 = \varphi_1 \cdot \psi_i$ holds.
Having found identical potential values up to $\alpha$, possible swaps to obtain $\phi_1^{\succ}(b) = \alpha \cdot \phi_2^{\succ}(b)$ are again determined by the positions of the identical potential values (up to $\alpha$) in the ordered multisets.
In particular, if $\varphi_i = \alpha \cdot \psi_i$ holds for $\varphi_i \in \phi_1^{\succ}(b)$ and $\psi_i \in \phi_2^{\succ}(b)$, then $\varphi_i$ and $\psi_i$ must be located at the same position in $\phi_1^{\succ}(b)$ and $\phi_2^{\succ}(b)$, respectively, to achieve that $\phi_1^{\succ}(b) = \alpha \cdot \phi_2^{\succ}(b)$.

Finally, we remark that this approach can be further extended to allow for a small deviation between potential values that are considered identical.
More specifically, instead of requiring that the equality $\varphi_i = \alpha \cdot \psi_i$ holds for all $i \in \{1, \ldots, n\}$, we can allow for a small deviation of factor $(1 + \varepsilon)$ between potential values and require that $\varphi_i \in [\alpha \cdot \psi_i \cdot (1 - \varepsilon), \alpha \cdot \psi_i \cdot (1 + \varepsilon)]$ holds for all $i \in \{1, \ldots, n\}$.
To check whether $\varphi_i$ lies in the specified interval, we need to ensure that the following inequalities hold:
\begin{align}
	\varphi_i &\geq \alpha \cdot \psi_i \cdot (1 - \varepsilon), \text{ and} \label{eq:int_interval_check_lower} \\
	\varphi_i &\leq \alpha \cdot \psi_i \cdot (1 + \varepsilon). \label{eq:int_interval_check_higher}
\end{align}
By entering $\alpha = \varphi_1 \mathbin{/} \psi_1$ into \cref{eq:int_interval_check_lower,eq:int_interval_check_higher}, we obtain
\begin{align}
	\varphi_i \cdot \psi_1 &\geq \varphi_1 \cdot \psi_i \cdot (1 - \varepsilon), \text{ and} \label{eq:int_interval_check_ineq_1} \\
	\varphi_i \cdot \psi_1 &\leq \varphi_1 \cdot \psi_i \cdot (1 + \varepsilon). \label{eq:int_interval_check_ineq_2}
\end{align}
Note that in general, it holds that $\varepsilon \in [0, 1]$ is a small floating point number.
To avoid floating point arithmetics, we can restrict $\varepsilon$ to be a rational number, which does not limit the practical applicability of this approach, as $\varepsilon$ can still be chosen arbitrarily small (e.g., $1 \mathbin{/} q$ for some arbitrary $q \in \mathbb{Z}$).
In particular, if $\varepsilon$ is a rational number, it can be represented by a fraction $\varepsilon = p \mathbin{/} q$ where $p \in \mathbb{Z}$ and $q \in \mathbb{Z}$ are integers.
Making use of this property, entering $\varepsilon = p \mathbin{/} q$ into \cref{eq:int_interval_check_ineq_1} yields
\begin{align}
	&\phantom{~\Leftrightarrow} \varphi_i \cdot \psi_1 &&\geq \varphi_1 \cdot \psi_i \cdot (1 - \varepsilon) \\
	&\Leftrightarrow \varphi_i \cdot \psi_1 &&\geq \varphi_1 \cdot (\psi_i - \psi_i \cdot \varepsilon) \\
	&\Leftrightarrow \varphi_i \cdot \psi_1 &&\geq \varphi_1 \cdot \psi_i - \varphi_1 \cdot \psi_i \cdot \varepsilon \\
	&\Leftrightarrow \varphi_i \cdot \psi_1 &&\geq \varphi_1 \cdot \psi_i - \varphi_1 \cdot \psi_i \cdot \frac{p}{q} \\
	&\Leftrightarrow \varphi_i \cdot \psi_1 \cdot q &&\geq \varphi_1 \cdot \psi_i \cdot q - \varphi_1 \cdot \psi_i \cdot p
\end{align}
and analogously for \cref{eq:int_interval_check_ineq_2}, we get
\begin{align}
	\varphi_i \cdot \psi_1 \cdot q &\leq \varphi_1 \cdot \psi_i \cdot q + \varphi_1 \cdot \psi_i \cdot p.
\end{align}
Again, these inequalities can be checked by using only multiplication operations, thereby allowing us to check for exchangeable factors independent of the scale of their potentials while at the same time allowing for arbitrary permutations of arguments and even a small deviation between potential values without using any floating point arithmetics (if the potential values themselves are integers).

\section{Additional Experimental Results} \label{appendix:more_eval}
In addition to the experimental results provided in \cref{sec:aacp_experiments}, we give further experimental results in this section.
We again evaluate the run times of running \acl{lve} on the output of \ac{acp} as well as of running \acl{lve} on the output of \ac{aacp} and also investigate the average number $\beta$ of queries after which the additional offline overhead of \ac{aacp} compared to \ac{acp} amortises.
The instances used in this section are identical to those used in \cref{sec:aacp_experiments} but we do not average the results over the proportion $p \in \{0.01, 0.05, 0.1, 0.15\}$ of scaled factors.
Instead, we present separate results for each individual choice of $p$ to highlight the effect of $p$ on both the run times for online query answering as well as on the offline overhead for constructing the \ac{pfg}.

\begin{figure}
	\centering
	\resizebox{0.49\textwidth}{!}{
\begin{tikzpicture}[x=1pt,y=1pt]
\definecolor{fillColor}{RGB}{255,255,255}
\path[use as bounding box,fill=fillColor,fill opacity=0.00] (0,0) rectangle (209.58,115.63);
\begin{scope}
\path[clip] (  0.00,  0.00) rectangle (209.58,115.63);
\definecolor{drawColor}{RGB}{255,255,255}
\definecolor{fillColor}{RGB}{255,255,255}

\path[draw=drawColor,line width= 0.6pt,line join=round,line cap=round,fill=fillColor] (  0.00,  0.00) rectangle (209.58,115.63);
\end{scope}
\begin{scope}
\path[clip] ( 35.23, 29.80) rectangle (204.08,110.13);
\definecolor{fillColor}{RGB}{255,255,255}

\path[fill=fillColor] ( 35.23, 29.80) rectangle (204.08,110.13);
\definecolor{drawColor}{RGB}{230,159,0}

\path[draw=drawColor,line width= 0.6pt,line join=round] ( 42.90, 34.35) --
	( 43.20, 35.85) --
	( 43.80, 37.73) --
	( 45.01, 35.35) --
	( 47.41, 38.92) --
	( 52.22, 36.67) --
	( 61.83, 40.68) --
	( 81.05, 41.41) --
	(119.51, 48.04) --
	(196.41, 50.51);
\definecolor{drawColor}{RGB}{46,37,133}

\path[draw=drawColor,line width= 0.6pt,dash pattern=on 2pt off 2pt ,line join=round] ( 42.90, 33.45) --
	( 43.20, 35.38) --
	( 43.80, 37.03) --
	( 45.01, 39.88) --
	( 47.41, 41.84) --
	( 52.22, 57.88) --
	( 61.83, 66.49) --
	( 81.05, 74.61) --
	(119.51, 84.00) --
	(196.41,106.48);
\definecolor{drawColor}{RGB}{230,159,0}
\definecolor{fillColor}{RGB}{230,159,0}

\path[draw=drawColor,line width= 0.4pt,line join=round,line cap=round,fill=fillColor] ( 47.41, 38.92) circle (  1.96);

\path[draw=drawColor,line width= 0.4pt,line join=round,line cap=round,fill=fillColor] ( 43.20, 35.85) circle (  1.96);

\path[draw=drawColor,line width= 0.4pt,line join=round,line cap=round,fill=fillColor] ( 45.01, 35.35) circle (  1.96);

\path[draw=drawColor,line width= 0.4pt,line join=round,line cap=round,fill=fillColor] ( 52.22, 36.67) circle (  1.96);

\path[draw=drawColor,line width= 0.4pt,line join=round,line cap=round,fill=fillColor] ( 42.90, 34.35) circle (  1.96);

\path[draw=drawColor,line width= 0.4pt,line join=round,line cap=round,fill=fillColor] (196.41, 50.51) circle (  1.96);

\path[draw=drawColor,line width= 0.4pt,line join=round,line cap=round,fill=fillColor] ( 43.80, 37.73) circle (  1.96);

\path[draw=drawColor,line width= 0.4pt,line join=round,line cap=round,fill=fillColor] (119.51, 48.04) circle (  1.96);

\path[draw=drawColor,line width= 0.4pt,line join=round,line cap=round,fill=fillColor] ( 61.83, 40.68) circle (  1.96);

\path[draw=drawColor,line width= 0.4pt,line join=round,line cap=round,fill=fillColor] ( 81.05, 41.41) circle (  1.96);
\definecolor{fillColor}{RGB}{46,37,133}

\path[fill=fillColor] ( 45.45, 39.88) --
	( 49.37, 39.88) --
	( 49.37, 43.81) --
	( 45.45, 43.81) --
	cycle;

\path[fill=fillColor] ( 41.24, 33.42) --
	( 45.17, 33.42) --
	( 45.17, 37.34) --
	( 41.24, 37.34) --
	cycle;

\path[fill=fillColor] ( 43.04, 37.92) --
	( 46.97, 37.92) --
	( 46.97, 41.84) --
	( 43.04, 41.84) --
	cycle;

\path[fill=fillColor] ( 50.25, 55.92) --
	( 54.18, 55.92) --
	( 54.18, 59.85) --
	( 50.25, 59.85) --
	cycle;

\path[fill=fillColor] ( 40.94, 31.49) --
	( 44.87, 31.49) --
	( 44.87, 35.42) --
	( 40.94, 35.42) --
	cycle;

\path[fill=fillColor] (194.45,104.52) --
	(198.37,104.52) --
	(198.37,108.44) --
	(194.45,108.44) --
	cycle;

\path[fill=fillColor] ( 41.84, 35.07) --
	( 45.77, 35.07) --
	( 45.77, 38.99) --
	( 41.84, 38.99) --
	cycle;

\path[fill=fillColor] (117.54, 82.04) --
	(121.47, 82.04) --
	(121.47, 85.97) --
	(117.54, 85.97) --
	cycle;

\path[fill=fillColor] ( 59.87, 64.53) --
	( 63.79, 64.53) --
	( 63.79, 68.46) --
	( 59.87, 68.46) --
	cycle;

\path[fill=fillColor] ( 79.09, 72.65) --
	( 83.02, 72.65) --
	( 83.02, 76.58) --
	( 79.09, 76.58) --
	cycle;
\end{scope}
\begin{scope}
\path[clip] (  0.00,  0.00) rectangle (209.58,115.63);
\definecolor{drawColor}{RGB}{0,0,0}

\path[draw=drawColor,line width= 0.6pt,line join=round] ( 35.23, 29.80) --
	( 35.23,110.13);

\path[draw=drawColor,line width= 0.6pt,line join=round] ( 36.65,107.67) --
	( 35.23,110.13) --
	( 33.81,107.67);
\end{scope}
\begin{scope}
\path[clip] (  0.00,  0.00) rectangle (209.58,115.63);
\definecolor{drawColor}{gray}{0.30}

\node[text=drawColor,anchor=base east,inner sep=0pt, outer sep=0pt, scale=  0.88] at ( 30.28, 46.77) {30};

\node[text=drawColor,anchor=base east,inner sep=0pt, outer sep=0pt, scale=  0.88] at ( 30.28, 62.72) {50};

\node[text=drawColor,anchor=base east,inner sep=0pt, outer sep=0pt, scale=  0.88] at ( 30.28, 84.36) {100};
\end{scope}
\begin{scope}
\path[clip] (  0.00,  0.00) rectangle (209.58,115.63);
\definecolor{drawColor}{gray}{0.20}

\path[draw=drawColor,line width= 0.6pt,line join=round] ( 32.48, 49.80) --
	( 35.23, 49.80);

\path[draw=drawColor,line width= 0.6pt,line join=round] ( 32.48, 65.75) --
	( 35.23, 65.75);

\path[draw=drawColor,line width= 0.6pt,line join=round] ( 32.48, 87.39) --
	( 35.23, 87.39);
\end{scope}
\begin{scope}
\path[clip] (  0.00,  0.00) rectangle (209.58,115.63);
\definecolor{drawColor}{RGB}{0,0,0}

\path[draw=drawColor,line width= 0.6pt,line join=round] ( 35.23, 29.80) --
	(204.08, 29.80);

\path[draw=drawColor,line width= 0.6pt,line join=round] (201.62, 28.38) --
	(204.08, 29.80) --
	(201.62, 31.23);
\end{scope}
\begin{scope}
\path[clip] (  0.00,  0.00) rectangle (209.58,115.63);
\definecolor{drawColor}{gray}{0.20}

\path[draw=drawColor,line width= 0.6pt,line join=round] ( 42.60, 27.05) --
	( 42.60, 29.80);

\path[draw=drawColor,line width= 0.6pt,line join=round] ( 80.15, 27.05) --
	( 80.15, 29.80);

\path[draw=drawColor,line width= 0.6pt,line join=round] (117.70, 27.05) --
	(117.70, 29.80);

\path[draw=drawColor,line width= 0.6pt,line join=round] (155.25, 27.05) --
	(155.25, 29.80);

\path[draw=drawColor,line width= 0.6pt,line join=round] (192.80, 27.05) --
	(192.80, 29.80);
\end{scope}
\begin{scope}
\path[clip] (  0.00,  0.00) rectangle (209.58,115.63);
\definecolor{drawColor}{gray}{0.30}

\node[text=drawColor,anchor=base,inner sep=0pt, outer sep=0pt, scale=  0.88] at ( 42.60, 18.79) {0};

\node[text=drawColor,anchor=base,inner sep=0pt, outer sep=0pt, scale=  0.88] at ( 80.15, 18.79) {250};

\node[text=drawColor,anchor=base,inner sep=0pt, outer sep=0pt, scale=  0.88] at (117.70, 18.79) {500};

\node[text=drawColor,anchor=base,inner sep=0pt, outer sep=0pt, scale=  0.88] at (155.25, 18.79) {750};

\node[text=drawColor,anchor=base,inner sep=0pt, outer sep=0pt, scale=  0.88] at (192.80, 18.79) {1000};
\end{scope}
\begin{scope}
\path[clip] (  0.00,  0.00) rectangle (209.58,115.63);
\definecolor{drawColor}{RGB}{0,0,0}

\node[text=drawColor,anchor=base,inner sep=0pt, outer sep=0pt, scale=  1.00] at (119.66,  7.44) {$d$};
\end{scope}
\begin{scope}
\path[clip] (  0.00,  0.00) rectangle (209.58,115.63);
\definecolor{drawColor}{RGB}{0,0,0}

\node[text=drawColor,rotate= 90.00,anchor=base,inner sep=0pt, outer sep=0pt, scale=  1.00] at ( 12.39, 69.97) {time (ms)};
\end{scope}
\begin{scope}
\path[clip] (  0.00,  0.00) rectangle (209.58,115.63);

\path[] ( 37.74, 78.13) rectangle ( 93.50,118.04);
\end{scope}
\begin{scope}
\path[clip] (  0.00,  0.00) rectangle (209.58,115.63);
\definecolor{drawColor}{RGB}{230,159,0}

\path[draw=drawColor,line width= 0.6pt,line join=round] ( 44.69,105.31) -- ( 56.25,105.31);
\end{scope}
\begin{scope}
\path[clip] (  0.00,  0.00) rectangle (209.58,115.63);
\definecolor{drawColor}{RGB}{230,159,0}
\definecolor{fillColor}{RGB}{230,159,0}

\path[draw=drawColor,line width= 0.4pt,line join=round,line cap=round,fill=fillColor] ( 50.47,105.31) circle (  1.96);
\end{scope}
\begin{scope}
\path[clip] (  0.00,  0.00) rectangle (209.58,115.63);
\definecolor{drawColor}{RGB}{46,37,133}

\path[draw=drawColor,line width= 0.6pt,dash pattern=on 2pt off 2pt ,line join=round] ( 44.69, 90.86) -- ( 56.25, 90.86);
\end{scope}
\begin{scope}
\path[clip] (  0.00,  0.00) rectangle (209.58,115.63);
\definecolor{fillColor}{RGB}{46,37,133}

\path[fill=fillColor] ( 48.51, 88.89) --
	( 52.43, 88.89) --
	( 52.43, 92.82) --
	( 48.51, 92.82) --
	cycle;
\end{scope}
\begin{scope}
\path[clip] (  0.00,  0.00) rectangle (209.58,115.63);
\definecolor{drawColor}{RGB}{0,0,0}

\node[text=drawColor,anchor=base west,inner sep=0pt, outer sep=0pt, scale=  0.80] at ( 63.19,102.55) {$\alpha$-ACP};
\end{scope}
\begin{scope}
\path[clip] (  0.00,  0.00) rectangle (209.58,115.63);
\definecolor{drawColor}{RGB}{0,0,0}

\node[text=drawColor,anchor=base west,inner sep=0pt, outer sep=0pt, scale=  0.80] at ( 63.19, 88.10) {ACP};
\end{scope}
\end{tikzpicture}}
	\resizebox{0.49\textwidth}{!}{
\begin{tikzpicture}[x=1pt,y=1pt]
\definecolor{fillColor}{RGB}{255,255,255}
\path[use as bounding box,fill=fillColor,fill opacity=0.00] (0,0) rectangle (209.58,115.63);
\begin{scope}
\path[clip] (  0.00,  0.00) rectangle (209.58,115.63);
\definecolor{drawColor}{RGB}{255,255,255}
\definecolor{fillColor}{RGB}{255,255,255}

\path[draw=drawColor,line width= 0.6pt,line join=round,line cap=round,fill=fillColor] (  0.00,  0.00) rectangle (209.58,115.63);
\end{scope}
\begin{scope}
\path[clip] ( 33.27, 29.80) rectangle (204.08,110.13);
\definecolor{fillColor}{RGB}{255,255,255}

\path[fill=fillColor] ( 33.27, 29.80) rectangle (204.08,110.13);
\definecolor{drawColor}{RGB}{46,37,133}

\path[draw=drawColor,line width= 0.6pt,line join=round] ( 43.32, 86.29) -- ( 43.32,106.48);

\path[draw=drawColor,line width= 0.6pt,line join=round] ( 43.32, 54.49) -- ( 43.32, 33.45);
\definecolor{fillColor}{RGB}{46,37,133}

\path[draw=drawColor,line width= 0.6pt,fill=fillColor,fill opacity=0.20] ( 37.04, 86.29) --
	( 37.04, 54.49) --
	( 49.60, 54.49) --
	( 49.60, 86.29) --
	( 37.04, 86.29) --
	cycle;

\path[draw=drawColor,line width= 1.1pt] ( 37.04, 72.94) -- ( 49.60, 72.94);

\path[draw=drawColor,line width= 0.6pt,line join=round] ( 60.07, 53.07) -- ( 60.07, 66.50);

\path[draw=drawColor,line width= 0.6pt,line join=round] ( 60.07, 40.62) -- ( 60.07, 33.45);

\path[draw=drawColor,line width= 0.6pt,fill=fillColor,fill opacity=0.20] ( 53.79, 53.07) --
	( 53.79, 40.62) --
	( 66.35, 40.62) --
	( 66.35, 53.07) --
	( 53.79, 53.07) --
	cycle;

\path[draw=drawColor,line width= 1.1pt] ( 53.79, 45.80) -- ( 66.35, 45.80);

\path[draw=drawColor,line width= 0.6pt,line join=round] ( 76.81, 68.17) -- ( 76.81, 81.22);

\path[draw=drawColor,line width= 0.6pt,line join=round] ( 76.81, 33.45) -- ( 76.81, 33.45);

\path[draw=drawColor,line width= 0.6pt,fill=fillColor,fill opacity=0.20] ( 70.53, 68.17) --
	( 70.53, 33.45) --
	( 83.09, 33.45) --
	( 83.09, 68.17) --
	( 70.53, 68.17) --
	cycle;

\path[draw=drawColor,line width= 1.1pt] ( 70.53, 43.01) -- ( 83.09, 43.01);

\path[draw=drawColor,line width= 0.6pt,line join=round] ( 93.56, 87.52) -- ( 93.56, 91.41);

\path[draw=drawColor,line width= 0.6pt,line join=round] ( 93.56, 66.22) -- ( 93.56, 52.56);

\path[draw=drawColor,line width= 0.6pt,fill=fillColor,fill opacity=0.20] ( 87.28, 87.52) --
	( 87.28, 66.22) --
	( 99.84, 66.22) --
	( 99.84, 87.52) --
	( 87.28, 87.52) --
	cycle;

\path[draw=drawColor,line width= 1.1pt] ( 87.28, 78.50) -- ( 99.84, 78.50);

\path[draw=drawColor,line width= 0.6pt,line join=round] (110.30, 71.67) -- (110.30, 71.67);

\path[draw=drawColor,line width= 0.6pt,line join=round] (110.30, 71.67) -- (110.30, 71.67);

\path[draw=drawColor,line width= 0.6pt,fill=fillColor,fill opacity=0.20] (104.03, 71.67) --
	(104.03, 71.67) --
	(116.58, 71.67) --
	(116.58, 71.67) --
	(104.03, 71.67) --
	cycle;

\path[draw=drawColor,line width= 1.1pt] (104.03, 71.67) -- (116.58, 71.67);

\path[draw=drawColor,line width= 0.6pt,line join=round] (143.80, 71.22) -- (143.80, 71.67);

\path[draw=drawColor,line width= 0.6pt,line join=round] (143.80, 67.26) -- (143.80, 63.74);

\path[draw=drawColor,line width= 0.6pt,fill=fillColor,fill opacity=0.20] (137.52, 71.22) --
	(137.52, 67.26) --
	(150.08, 67.26) --
	(150.08, 71.22) --
	(137.52, 71.22) --
	cycle;

\path[draw=drawColor,line width= 1.1pt] (137.52, 70.78) -- (150.08, 70.78);

\path[draw=drawColor,line width= 0.6pt,line join=round] (160.54, 50.58) -- (160.54, 52.56);

\path[draw=drawColor,line width= 0.6pt,line join=round] (160.54, 41.03) -- (160.54, 33.45);

\path[draw=drawColor,line width= 0.6pt,fill=fillColor,fill opacity=0.20] (154.26, 50.58) --
	(154.26, 41.03) --
	(166.82, 41.03) --
	(166.82, 50.58) --
	(154.26, 50.58) --
	cycle;

\path[draw=drawColor,line width= 1.1pt] (154.26, 48.60) -- (166.82, 48.60);

\path[draw=drawColor,line width= 0.6pt,line join=round] (194.04, 72.07) -- (194.04, 73.29);

\path[draw=drawColor,line width= 0.6pt,line join=round] (194.04, 64.10) -- (194.04, 62.11);

\path[draw=drawColor,line width= 0.6pt,fill=fillColor,fill opacity=0.20] (187.76, 72.07) --
	(187.76, 64.10) --
	(200.32, 64.10) --
	(200.32, 72.07) --
	(187.76, 72.07) --
	cycle;

\path[draw=drawColor,line width= 1.1pt] (187.76, 67.98) -- (200.32, 67.98);
\end{scope}
\begin{scope}
\path[clip] (  0.00,  0.00) rectangle (209.58,115.63);
\definecolor{drawColor}{RGB}{0,0,0}

\path[draw=drawColor,line width= 0.6pt,line join=round] ( 33.27, 29.80) --
	( 33.27,110.13);

\path[draw=drawColor,line width= 0.6pt,line join=round] ( 34.70,107.67) --
	( 33.27,110.13) --
	( 31.85,107.67);
\end{scope}
\begin{scope}
\path[clip] (  0.00,  0.00) rectangle (209.58,115.63);
\definecolor{drawColor}{gray}{0.30}

\node[text=drawColor,anchor=base east,inner sep=0pt, outer sep=0pt, scale=  0.88] at ( 28.32, 30.42) {0.0};

\node[text=drawColor,anchor=base east,inner sep=0pt, outer sep=0pt, scale=  0.88] at ( 28.32, 62.16) {0.1};

\node[text=drawColor,anchor=base east,inner sep=0pt, outer sep=0pt, scale=  0.88] at ( 28.32, 93.90) {1.0};
\end{scope}
\begin{scope}
\path[clip] (  0.00,  0.00) rectangle (209.58,115.63);
\definecolor{drawColor}{gray}{0.20}

\path[draw=drawColor,line width= 0.6pt,line join=round] ( 30.52, 33.45) --
	( 33.27, 33.45);

\path[draw=drawColor,line width= 0.6pt,line join=round] ( 30.52, 65.19) --
	( 33.27, 65.19);

\path[draw=drawColor,line width= 0.6pt,line join=round] ( 30.52, 96.93) --
	( 33.27, 96.93);
\end{scope}
\begin{scope}
\path[clip] (  0.00,  0.00) rectangle (209.58,115.63);
\definecolor{drawColor}{RGB}{0,0,0}

\path[draw=drawColor,line width= 0.6pt,line join=round] ( 33.27, 29.80) --
	(204.08, 29.80);

\path[draw=drawColor,line width= 0.6pt,line join=round] (201.62, 28.38) --
	(204.08, 29.80) --
	(201.62, 31.23);
\end{scope}
\begin{scope}
\path[clip] (  0.00,  0.00) rectangle (209.58,115.63);
\definecolor{drawColor}{gray}{0.20}

\path[draw=drawColor,line width= 0.6pt,line join=round] ( 43.32, 27.05) --
	( 43.32, 29.80);

\path[draw=drawColor,line width= 0.6pt,line join=round] ( 60.07, 27.05) --
	( 60.07, 29.80);

\path[draw=drawColor,line width= 0.6pt,line join=round] ( 76.81, 27.05) --
	( 76.81, 29.80);

\path[draw=drawColor,line width= 0.6pt,line join=round] ( 93.56, 27.05) --
	( 93.56, 29.80);

\path[draw=drawColor,line width= 0.6pt,line join=round] (110.30, 27.05) --
	(110.30, 29.80);

\path[draw=drawColor,line width= 0.6pt,line join=round] (127.05, 27.05) --
	(127.05, 29.80);

\path[draw=drawColor,line width= 0.6pt,line join=round] (143.80, 27.05) --
	(143.80, 29.80);

\path[draw=drawColor,line width= 0.6pt,line join=round] (160.54, 27.05) --
	(160.54, 29.80);

\path[draw=drawColor,line width= 0.6pt,line join=round] (177.29, 27.05) --
	(177.29, 29.80);

\path[draw=drawColor,line width= 0.6pt,line join=round] (194.04, 27.05) --
	(194.04, 29.80);
\end{scope}
\begin{scope}
\path[clip] (  0.00,  0.00) rectangle (209.58,115.63);
\definecolor{drawColor}{gray}{0.30}

\node[text=drawColor,anchor=base,inner sep=0pt, outer sep=0pt, scale=  0.88] at ( 43.32, 18.79) {2};

\node[text=drawColor,anchor=base,inner sep=0pt, outer sep=0pt, scale=  0.88] at ( 60.07, 18.79) {4};

\node[text=drawColor,anchor=base,inner sep=0pt, outer sep=0pt, scale=  0.88] at ( 76.81, 18.79) {8};

\node[text=drawColor,anchor=base,inner sep=0pt, outer sep=0pt, scale=  0.88] at ( 93.56, 18.79) {16};

\node[text=drawColor,anchor=base,inner sep=0pt, outer sep=0pt, scale=  0.88] at (110.30, 18.79) {32};

\node[text=drawColor,anchor=base,inner sep=0pt, outer sep=0pt, scale=  0.88] at (127.05, 18.79) {64};

\node[text=drawColor,anchor=base,inner sep=0pt, outer sep=0pt, scale=  0.88] at (143.80, 18.79) {128};

\node[text=drawColor,anchor=base,inner sep=0pt, outer sep=0pt, scale=  0.88] at (160.54, 18.79) {256};

\node[text=drawColor,anchor=base,inner sep=0pt, outer sep=0pt, scale=  0.88] at (177.29, 18.79) {512};

\node[text=drawColor,anchor=base,inner sep=0pt, outer sep=0pt, scale=  0.88] at (194.04, 18.79) {1024};
\end{scope}
\begin{scope}
\path[clip] (  0.00,  0.00) rectangle (209.58,115.63);
\definecolor{drawColor}{RGB}{0,0,0}

\node[text=drawColor,anchor=base,inner sep=0pt, outer sep=0pt, scale=  1.00] at (118.68,  7.44) {$d$};
\end{scope}
\begin{scope}
\path[clip] (  0.00,  0.00) rectangle (209.58,115.63);
\definecolor{drawColor}{RGB}{0,0,0}

\node[text=drawColor,rotate= 90.00,anchor=base,inner sep=0pt, outer sep=0pt, scale=  1.00] at ( 12.39, 69.97) {$\beta$};
\end{scope}
\end{tikzpicture}}
	\caption{Average query times of \acl{lve} run on the output of \ac{acp} and \ac{aacp} where the input \acp{fg} contain a proportion of $p=0.01$ scaled factors (left), and the distribution of the number $\beta$ of queries after which the offline overhead of \ac{aacp} amortises on input \acp{fg} containing a proportion of $p=0.01$ scaled factors (right).}
	\label{fig:aacp_plot_app_p=0.01}
\end{figure}
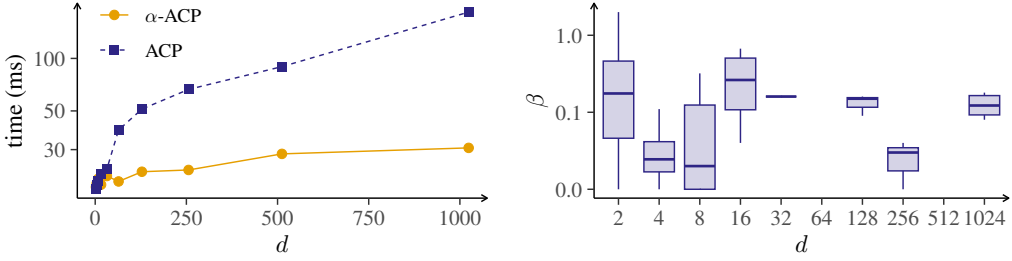
\begin{figure}
	\centering
	\resizebox{0.49\textwidth}{!}{
\begin{tikzpicture}[x=1pt,y=1pt]
\definecolor{fillColor}{RGB}{255,255,255}
\path[use as bounding box,fill=fillColor,fill opacity=0.00] (0,0) rectangle (209.58,115.63);
\begin{scope}
\path[clip] (  0.00,  0.00) rectangle (209.58,115.63);
\definecolor{drawColor}{RGB}{255,255,255}
\definecolor{fillColor}{RGB}{255,255,255}

\path[draw=drawColor,line width= 0.6pt,line join=round,line cap=round,fill=fillColor] (  0.00,  0.00) rectangle (209.58,115.63);
\end{scope}
\begin{scope}
\path[clip] ( 39.63, 29.80) rectangle (204.08,110.13);
\definecolor{fillColor}{RGB}{255,255,255}

\path[fill=fillColor] ( 39.63, 29.80) rectangle (204.08,110.13);
\definecolor{drawColor}{RGB}{230,159,0}

\path[draw=drawColor,line width= 0.6pt,line join=round] ( 47.10, 33.95) --
	( 47.39, 33.75) --
	( 47.98, 33.82) --
	( 49.15, 33.45) --
	( 51.49, 36.76) --
	( 56.17, 35.78) --
	( 65.53, 36.99) --
	( 84.26, 37.50) --
	(121.71, 41.15) --
	(196.61, 43.11);
\definecolor{drawColor}{RGB}{46,37,133}

\path[draw=drawColor,line width= 0.6pt,dash pattern=on 2pt off 2pt ,line join=round] ( 47.10, 33.80) --
	( 47.39, 36.61) --
	( 47.98, 37.98) --
	( 49.15, 38.86) --
	( 51.49, 47.95) --
	( 56.17, 56.96) --
	( 65.53, 64.83) --
	( 84.26, 74.49) --
	(121.71, 91.64) --
	(196.61,106.48);
\definecolor{drawColor}{RGB}{230,159,0}
\definecolor{fillColor}{RGB}{230,159,0}

\path[draw=drawColor,line width= 0.4pt,line join=round,line cap=round,fill=fillColor] ( 51.49, 36.76) circle (  1.96);

\path[draw=drawColor,line width= 0.4pt,line join=round,line cap=round,fill=fillColor] ( 47.39, 33.75) circle (  1.96);

\path[draw=drawColor,line width= 0.4pt,line join=round,line cap=round,fill=fillColor] ( 49.15, 33.45) circle (  1.96);

\path[draw=drawColor,line width= 0.4pt,line join=round,line cap=round,fill=fillColor] ( 56.17, 35.78) circle (  1.96);

\path[draw=drawColor,line width= 0.4pt,line join=round,line cap=round,fill=fillColor] ( 47.10, 33.95) circle (  1.96);

\path[draw=drawColor,line width= 0.4pt,line join=round,line cap=round,fill=fillColor] (196.61, 43.11) circle (  1.96);

\path[draw=drawColor,line width= 0.4pt,line join=round,line cap=round,fill=fillColor] ( 47.98, 33.82) circle (  1.96);

\path[draw=drawColor,line width= 0.4pt,line join=round,line cap=round,fill=fillColor] (121.71, 41.15) circle (  1.96);

\path[draw=drawColor,line width= 0.4pt,line join=round,line cap=round,fill=fillColor] ( 65.53, 36.99) circle (  1.96);

\path[draw=drawColor,line width= 0.4pt,line join=round,line cap=round,fill=fillColor] ( 84.26, 37.50) circle (  1.96);
\definecolor{fillColor}{RGB}{46,37,133}

\path[fill=fillColor] ( 49.53, 45.98) --
	( 53.45, 45.98) --
	( 53.45, 49.91) --
	( 49.53, 49.91) --
	cycle;

\path[fill=fillColor] ( 45.43, 34.64) --
	( 49.36, 34.64) --
	( 49.36, 38.57) --
	( 45.43, 38.57) --
	cycle;

\path[fill=fillColor] ( 47.19, 36.89) --
	( 51.11, 36.89) --
	( 51.11, 40.82) --
	( 47.19, 40.82) --
	cycle;

\path[fill=fillColor] ( 54.21, 54.99) --
	( 58.13, 54.99) --
	( 58.13, 58.92) --
	( 54.21, 58.92) --
	cycle;

\path[fill=fillColor] ( 45.14, 31.83) --
	( 49.06, 31.83) --
	( 49.06, 35.76) --
	( 45.14, 35.76) --
	cycle;

\path[fill=fillColor] (194.65,104.52) --
	(198.57,104.52) --
	(198.57,108.44) --
	(194.65,108.44) --
	cycle;

\path[fill=fillColor] ( 46.02, 36.02) --
	( 49.94, 36.02) --
	( 49.94, 39.94) --
	( 46.02, 39.94) --
	cycle;

\path[fill=fillColor] (119.75, 89.68) --
	(123.67, 89.68) --
	(123.67, 93.61) --
	(119.75, 93.61) --
	cycle;

\path[fill=fillColor] ( 63.57, 62.87) --
	( 67.50, 62.87) --
	( 67.50, 66.79) --
	( 63.57, 66.79) --
	cycle;

\path[fill=fillColor] ( 82.30, 72.52) --
	( 86.22, 72.52) --
	( 86.22, 76.45) --
	( 82.30, 76.45) --
	cycle;
\end{scope}
\begin{scope}
\path[clip] (  0.00,  0.00) rectangle (209.58,115.63);
\definecolor{drawColor}{RGB}{0,0,0}

\path[draw=drawColor,line width= 0.6pt,line join=round] ( 39.63, 29.80) --
	( 39.63,110.13);

\path[draw=drawColor,line width= 0.6pt,line join=round] ( 41.05,107.67) --
	( 39.63,110.13) --
	( 38.20,107.67);
\end{scope}
\begin{scope}
\path[clip] (  0.00,  0.00) rectangle (209.58,115.63);
\definecolor{drawColor}{gray}{0.30}

\node[text=drawColor,anchor=base east,inner sep=0pt, outer sep=0pt, scale=  0.88] at ( 34.68, 39.32) {30};

\node[text=drawColor,anchor=base east,inner sep=0pt, outer sep=0pt, scale=  0.88] at ( 34.68, 61.39) {100};

\node[text=drawColor,anchor=base east,inner sep=0pt, outer sep=0pt, scale=  0.88] at ( 34.68, 81.52) {300};

\node[text=drawColor,anchor=base east,inner sep=0pt, outer sep=0pt, scale=  0.88] at ( 34.68,103.59) {1000};
\end{scope}
\begin{scope}
\path[clip] (  0.00,  0.00) rectangle (209.58,115.63);
\definecolor{drawColor}{gray}{0.20}

\path[draw=drawColor,line width= 0.6pt,line join=round] ( 36.88, 42.35) --
	( 39.63, 42.35);

\path[draw=drawColor,line width= 0.6pt,line join=round] ( 36.88, 64.42) --
	( 39.63, 64.42);

\path[draw=drawColor,line width= 0.6pt,line join=round] ( 36.88, 84.55) --
	( 39.63, 84.55);

\path[draw=drawColor,line width= 0.6pt,line join=round] ( 36.88,106.62) --
	( 39.63,106.62);
\end{scope}
\begin{scope}
\path[clip] (  0.00,  0.00) rectangle (209.58,115.63);
\definecolor{drawColor}{RGB}{0,0,0}

\path[draw=drawColor,line width= 0.6pt,line join=round] ( 39.63, 29.80) --
	(204.08, 29.80);

\path[draw=drawColor,line width= 0.6pt,line join=round] (201.62, 28.38) --
	(204.08, 29.80) --
	(201.62, 31.23);
\end{scope}
\begin{scope}
\path[clip] (  0.00,  0.00) rectangle (209.58,115.63);
\definecolor{drawColor}{gray}{0.20}

\path[draw=drawColor,line width= 0.6pt,line join=round] ( 46.81, 27.05) --
	( 46.81, 29.80);

\path[draw=drawColor,line width= 0.6pt,line join=round] ( 83.38, 27.05) --
	( 83.38, 29.80);

\path[draw=drawColor,line width= 0.6pt,line join=round] (119.95, 27.05) --
	(119.95, 29.80);

\path[draw=drawColor,line width= 0.6pt,line join=round] (156.53, 27.05) --
	(156.53, 29.80);

\path[draw=drawColor,line width= 0.6pt,line join=round] (193.10, 27.05) --
	(193.10, 29.80);
\end{scope}
\begin{scope}
\path[clip] (  0.00,  0.00) rectangle (209.58,115.63);
\definecolor{drawColor}{gray}{0.30}

\node[text=drawColor,anchor=base,inner sep=0pt, outer sep=0pt, scale=  0.88] at ( 46.81, 18.79) {0};

\node[text=drawColor,anchor=base,inner sep=0pt, outer sep=0pt, scale=  0.88] at ( 83.38, 18.79) {250};

\node[text=drawColor,anchor=base,inner sep=0pt, outer sep=0pt, scale=  0.88] at (119.95, 18.79) {500};

\node[text=drawColor,anchor=base,inner sep=0pt, outer sep=0pt, scale=  0.88] at (156.53, 18.79) {750};

\node[text=drawColor,anchor=base,inner sep=0pt, outer sep=0pt, scale=  0.88] at (193.10, 18.79) {1000};
\end{scope}
\begin{scope}
\path[clip] (  0.00,  0.00) rectangle (209.58,115.63);
\definecolor{drawColor}{RGB}{0,0,0}

\node[text=drawColor,anchor=base,inner sep=0pt, outer sep=0pt, scale=  1.00] at (121.85,  7.44) {$d$};
\end{scope}
\begin{scope}
\path[clip] (  0.00,  0.00) rectangle (209.58,115.63);
\definecolor{drawColor}{RGB}{0,0,0}

\node[text=drawColor,rotate= 90.00,anchor=base,inner sep=0pt, outer sep=0pt, scale=  1.00] at ( 12.39, 69.97) {time (ms)};
\end{scope}
\begin{scope}
\path[clip] (  0.00,  0.00) rectangle (209.58,115.63);

\path[] ( 41.35, 78.13) rectangle ( 97.11,118.04);
\end{scope}
\begin{scope}
\path[clip] (  0.00,  0.00) rectangle (209.58,115.63);
\definecolor{drawColor}{RGB}{230,159,0}

\path[draw=drawColor,line width= 0.6pt,line join=round] ( 48.29,105.31) -- ( 59.86,105.31);
\end{scope}
\begin{scope}
\path[clip] (  0.00,  0.00) rectangle (209.58,115.63);
\definecolor{drawColor}{RGB}{230,159,0}
\definecolor{fillColor}{RGB}{230,159,0}

\path[draw=drawColor,line width= 0.4pt,line join=round,line cap=round,fill=fillColor] ( 54.07,105.31) circle (  1.96);
\end{scope}
\begin{scope}
\path[clip] (  0.00,  0.00) rectangle (209.58,115.63);
\definecolor{drawColor}{RGB}{46,37,133}

\path[draw=drawColor,line width= 0.6pt,dash pattern=on 2pt off 2pt ,line join=round] ( 48.29, 90.86) -- ( 59.86, 90.86);
\end{scope}
\begin{scope}
\path[clip] (  0.00,  0.00) rectangle (209.58,115.63);
\definecolor{fillColor}{RGB}{46,37,133}

\path[fill=fillColor] ( 52.11, 88.89) --
	( 56.04, 88.89) --
	( 56.04, 92.82) --
	( 52.11, 92.82) --
	cycle;
\end{scope}
\begin{scope}
\path[clip] (  0.00,  0.00) rectangle (209.58,115.63);
\definecolor{drawColor}{RGB}{0,0,0}

\node[text=drawColor,anchor=base west,inner sep=0pt, outer sep=0pt, scale=  0.80] at ( 66.80,102.55) {$\alpha$-ACP};
\end{scope}
\begin{scope}
\path[clip] (  0.00,  0.00) rectangle (209.58,115.63);
\definecolor{drawColor}{RGB}{0,0,0}

\node[text=drawColor,anchor=base west,inner sep=0pt, outer sep=0pt, scale=  0.80] at ( 66.80, 88.10) {ACP};
\end{scope}
\end{tikzpicture}}
	\resizebox{0.49\textwidth}{!}{
\begin{tikzpicture}[x=1pt,y=1pt]
\definecolor{fillColor}{RGB}{255,255,255}
\path[use as bounding box,fill=fillColor,fill opacity=0.00] (0,0) rectangle (209.58,115.63);
\begin{scope}
\path[clip] (  0.00,  0.00) rectangle (209.58,115.63);
\definecolor{drawColor}{RGB}{255,255,255}
\definecolor{fillColor}{RGB}{255,255,255}

\path[draw=drawColor,line width= 0.6pt,line join=round,line cap=round,fill=fillColor] (  0.00,  0.00) rectangle (209.58,115.63);
\end{scope}
\begin{scope}
\path[clip] ( 37.67, 29.80) rectangle (204.08,110.13);
\definecolor{fillColor}{RGB}{255,255,255}

\path[fill=fillColor] ( 37.67, 29.80) rectangle (204.08,110.13);
\definecolor{drawColor}{RGB}{46,37,133}

\path[draw=drawColor,line width= 0.6pt,line join=round] ( 47.46, 52.15) -- ( 47.46, 69.40);

\path[draw=drawColor,line width= 0.6pt,line join=round] ( 47.46, 35.32) -- ( 47.46, 33.45);
\definecolor{fillColor}{RGB}{46,37,133}

\path[draw=drawColor,line width= 0.6pt,fill=fillColor,fill opacity=0.20] ( 41.34, 52.15) --
	( 41.34, 35.32) --
	( 53.58, 35.32) --
	( 53.58, 52.15) --
	( 41.34, 52.15) --
	cycle;

\path[draw=drawColor,line width= 1.1pt] ( 41.34, 40.93) -- ( 53.58, 40.93);

\path[draw=drawColor,line width= 0.6pt,line join=round] ( 63.78, 51.31) -- ( 63.78, 52.78);

\path[draw=drawColor,line width= 0.6pt,line join=round] ( 63.78, 39.06) -- ( 63.78, 33.45);

\path[draw=drawColor,line width= 0.6pt,fill=fillColor,fill opacity=0.20] ( 57.66, 51.31) --
	( 57.66, 39.06) --
	( 69.89, 39.06) --
	( 69.89, 51.31) --
	( 57.66, 51.31) --
	cycle;

\path[draw=drawColor,line width= 1.1pt] ( 57.66, 49.61) -- ( 69.89, 49.61);

\path[draw=drawColor,line width= 0.6pt,line join=round] ( 80.09, 74.81) -- ( 80.09, 76.48);

\path[draw=drawColor,line width= 0.6pt,line join=round] ( 80.09, 45.31) -- ( 80.09, 40.93);

\path[draw=drawColor,line width= 0.6pt,fill=fillColor,fill opacity=0.20] ( 73.97, 74.81) --
	( 73.97, 45.31) --
	( 86.21, 45.31) --
	( 86.21, 74.81) --
	( 73.97, 74.81) --
	cycle;

\path[draw=drawColor,line width= 1.1pt] ( 73.97, 51.80) -- ( 86.21, 51.80);
\definecolor{drawColor}{RGB}{46,37,133}

\path[draw=drawColor,draw opacity=0.20,line width= 0.4pt,line join=round,line cap=round,fill=fillColor,fill opacity=0.20] ( 96.41,106.48) circle (  1.96);
\definecolor{drawColor}{RGB}{46,37,133}

\path[draw=drawColor,line width= 0.6pt,line join=round] ( 96.41, 69.01) -- ( 96.41, 71.49);

\path[draw=drawColor,line width= 0.6pt,line join=round] ( 96.41, 54.62) -- ( 96.41, 45.31);

\path[draw=drawColor,line width= 0.6pt,fill=fillColor,fill opacity=0.20] ( 90.29, 69.01) --
	( 90.29, 54.62) --
	(102.52, 54.62) --
	(102.52, 69.01) --
	( 90.29, 69.01) --
	cycle;

\path[draw=drawColor,line width= 1.1pt] ( 90.29, 60.83) -- (102.52, 60.83);

\path[draw=drawColor,line width= 0.6pt,line join=round] (112.72, 66.70) -- (112.72, 68.18);

\path[draw=drawColor,line width= 0.6pt,line join=round] (112.72, 64.62) -- (112.72, 64.02);

\path[draw=drawColor,line width= 0.6pt,fill=fillColor,fill opacity=0.20] (106.60, 66.70) --
	(106.60, 64.62) --
	(118.84, 64.62) --
	(118.84, 66.70) --
	(106.60, 66.70) --
	cycle;

\path[draw=drawColor,line width= 1.1pt] (106.60, 65.22) -- (118.84, 65.22);

\path[draw=drawColor,line width= 0.6pt,line join=round] (129.03, 50.82) -- (129.03, 55.89);

\path[draw=drawColor,line width= 0.6pt,line join=round] (129.03, 33.45) -- (129.03, 33.45);

\path[draw=drawColor,line width= 0.6pt,fill=fillColor,fill opacity=0.20] (122.92, 50.82) --
	(122.92, 33.45) --
	(135.15, 33.45) --
	(135.15, 50.82) --
	(122.92, 50.82) --
	cycle;

\path[draw=drawColor,line width= 1.1pt] (122.92, 50.82) -- (135.15, 50.82);

\path[draw=drawColor,line width= 0.6pt,line join=round] (145.35, 64.02) -- (145.35, 67.74);

\path[draw=drawColor,line width= 0.6pt,line join=round] (145.35, 42.02) -- (145.35, 40.93);

\path[draw=drawColor,line width= 0.6pt,fill=fillColor,fill opacity=0.20] (139.23, 64.02) --
	(139.23, 42.02) --
	(151.47, 42.02) --
	(151.47, 64.02) --
	(139.23, 64.02) --
	cycle;

\path[draw=drawColor,line width= 1.1pt] (139.23, 54.66) -- (151.47, 54.66);

\path[draw=drawColor,line width= 0.6pt,line join=round] (161.66, 69.37) -- (161.66, 70.14);

\path[draw=drawColor,line width= 0.6pt,line join=round] (161.66, 67.19) -- (161.66, 65.77);

\path[draw=drawColor,line width= 0.6pt,fill=fillColor,fill opacity=0.20] (155.55, 69.37) --
	(155.55, 67.19) --
	(167.78, 67.19) --
	(167.78, 69.37) --
	(155.55, 69.37) --
	cycle;

\path[draw=drawColor,line width= 1.1pt] (155.55, 68.60) -- (167.78, 68.60);

\path[draw=drawColor,line width= 0.6pt,line join=round] (177.98, 55.17) -- (177.98, 55.89);

\path[draw=drawColor,line width= 0.6pt,line join=round] (177.98, 53.61) -- (177.98, 52.78);

\path[draw=drawColor,line width= 0.6pt,fill=fillColor,fill opacity=0.20] (171.86, 55.17) --
	(171.86, 53.61) --
	(184.10, 53.61) --
	(184.10, 55.17) --
	(171.86, 55.17) --
	cycle;

\path[draw=drawColor,line width= 1.1pt] (171.86, 54.45) -- (184.10, 54.45);

\path[draw=drawColor,line width= 0.6pt,line join=round] (194.29, 59.32) -- (194.29, 59.32);

\path[draw=drawColor,line width= 0.6pt,line join=round] (194.29, 58.81) -- (194.29, 58.29);

\path[draw=drawColor,line width= 0.6pt,fill=fillColor,fill opacity=0.20] (188.18, 59.32) --
	(188.18, 58.81) --
	(200.41, 58.81) --
	(200.41, 59.32) --
	(188.18, 59.32) --
	cycle;

\path[draw=drawColor,line width= 1.1pt] (188.18, 59.32) -- (200.41, 59.32);
\end{scope}
\begin{scope}
\path[clip] (  0.00,  0.00) rectangle (209.58,115.63);
\definecolor{drawColor}{RGB}{0,0,0}

\path[draw=drawColor,line width= 0.6pt,line join=round] ( 37.67, 29.80) --
	( 37.67,110.13);

\path[draw=drawColor,line width= 0.6pt,line join=round] ( 39.09,107.67) --
	( 37.67,110.13) --
	( 36.25,107.67);
\end{scope}
\begin{scope}
\path[clip] (  0.00,  0.00) rectangle (209.58,115.63);
\definecolor{drawColor}{gray}{0.30}

\node[text=drawColor,anchor=base east,inner sep=0pt, outer sep=0pt, scale=  0.88] at ( 32.72, 30.42) { 0.0};

\node[text=drawColor,anchor=base east,inner sep=0pt, outer sep=0pt, scale=  0.88] at ( 32.72, 55.26) { 0.1};

\node[text=drawColor,anchor=base east,inner sep=0pt, outer sep=0pt, scale=  0.88] at ( 32.72, 80.10) { 1.0};

\node[text=drawColor,anchor=base east,inner sep=0pt, outer sep=0pt, scale=  0.88] at ( 32.72,104.94) {10.0};
\end{scope}
\begin{scope}
\path[clip] (  0.00,  0.00) rectangle (209.58,115.63);
\definecolor{drawColor}{gray}{0.20}

\path[draw=drawColor,line width= 0.6pt,line join=round] ( 34.92, 33.45) --
	( 37.67, 33.45);

\path[draw=drawColor,line width= 0.6pt,line join=round] ( 34.92, 58.29) --
	( 37.67, 58.29);

\path[draw=drawColor,line width= 0.6pt,line join=round] ( 34.92, 83.13) --
	( 37.67, 83.13);

\path[draw=drawColor,line width= 0.6pt,line join=round] ( 34.92,107.97) --
	( 37.67,107.97);
\end{scope}
\begin{scope}
\path[clip] (  0.00,  0.00) rectangle (209.58,115.63);
\definecolor{drawColor}{RGB}{0,0,0}

\path[draw=drawColor,line width= 0.6pt,line join=round] ( 37.67, 29.80) --
	(204.08, 29.80);

\path[draw=drawColor,line width= 0.6pt,line join=round] (201.62, 28.38) --
	(204.08, 29.80) --
	(201.62, 31.23);
\end{scope}
\begin{scope}
\path[clip] (  0.00,  0.00) rectangle (209.58,115.63);
\definecolor{drawColor}{gray}{0.20}

\path[draw=drawColor,line width= 0.6pt,line join=round] ( 47.46, 27.05) --
	( 47.46, 29.80);

\path[draw=drawColor,line width= 0.6pt,line join=round] ( 63.78, 27.05) --
	( 63.78, 29.80);

\path[draw=drawColor,line width= 0.6pt,line join=round] ( 80.09, 27.05) --
	( 80.09, 29.80);

\path[draw=drawColor,line width= 0.6pt,line join=round] ( 96.41, 27.05) --
	( 96.41, 29.80);

\path[draw=drawColor,line width= 0.6pt,line join=round] (112.72, 27.05) --
	(112.72, 29.80);

\path[draw=drawColor,line width= 0.6pt,line join=round] (129.03, 27.05) --
	(129.03, 29.80);

\path[draw=drawColor,line width= 0.6pt,line join=round] (145.35, 27.05) --
	(145.35, 29.80);

\path[draw=drawColor,line width= 0.6pt,line join=round] (161.66, 27.05) --
	(161.66, 29.80);

\path[draw=drawColor,line width= 0.6pt,line join=round] (177.98, 27.05) --
	(177.98, 29.80);

\path[draw=drawColor,line width= 0.6pt,line join=round] (194.29, 27.05) --
	(194.29, 29.80);
\end{scope}
\begin{scope}
\path[clip] (  0.00,  0.00) rectangle (209.58,115.63);
\definecolor{drawColor}{gray}{0.30}

\node[text=drawColor,anchor=base,inner sep=0pt, outer sep=0pt, scale=  0.88] at ( 47.46, 18.79) {2};

\node[text=drawColor,anchor=base,inner sep=0pt, outer sep=0pt, scale=  0.88] at ( 63.78, 18.79) {4};

\node[text=drawColor,anchor=base,inner sep=0pt, outer sep=0pt, scale=  0.88] at ( 80.09, 18.79) {8};

\node[text=drawColor,anchor=base,inner sep=0pt, outer sep=0pt, scale=  0.88] at ( 96.41, 18.79) {16};

\node[text=drawColor,anchor=base,inner sep=0pt, outer sep=0pt, scale=  0.88] at (112.72, 18.79) {32};

\node[text=drawColor,anchor=base,inner sep=0pt, outer sep=0pt, scale=  0.88] at (129.03, 18.79) {64};

\node[text=drawColor,anchor=base,inner sep=0pt, outer sep=0pt, scale=  0.88] at (145.35, 18.79) {128};

\node[text=drawColor,anchor=base,inner sep=0pt, outer sep=0pt, scale=  0.88] at (161.66, 18.79) {256};

\node[text=drawColor,anchor=base,inner sep=0pt, outer sep=0pt, scale=  0.88] at (177.98, 18.79) {512};

\node[text=drawColor,anchor=base,inner sep=0pt, outer sep=0pt, scale=  0.88] at (194.29, 18.79) {1024};
\end{scope}
\begin{scope}
\path[clip] (  0.00,  0.00) rectangle (209.58,115.63);
\definecolor{drawColor}{RGB}{0,0,0}

\node[text=drawColor,anchor=base,inner sep=0pt, outer sep=0pt, scale=  1.00] at (120.88,  7.44) {$d$};
\end{scope}
\begin{scope}
\path[clip] (  0.00,  0.00) rectangle (209.58,115.63);
\definecolor{drawColor}{RGB}{0,0,0}

\node[text=drawColor,rotate= 90.00,anchor=base,inner sep=0pt, outer sep=0pt, scale=  1.00] at ( 12.39, 69.97) {$\beta$};
\end{scope}
\end{tikzpicture}}
	\caption{Average query times of \acl{lve} run on the output of \ac{acp} and \ac{aacp} where the input \acp{fg} contain a proportion of $p=0.05$ scaled factors (left), and the distribution of the number $\beta$ of queries after which the offline overhead of \ac{aacp} amortises on input \acp{fg} containing a proportion of $p=0.05$ scaled factors (right).}
	\label{fig:aacp_plot_app_p=0.05}
\end{figure}
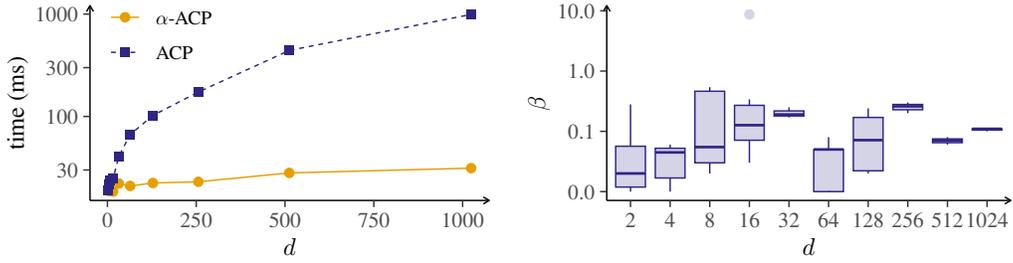
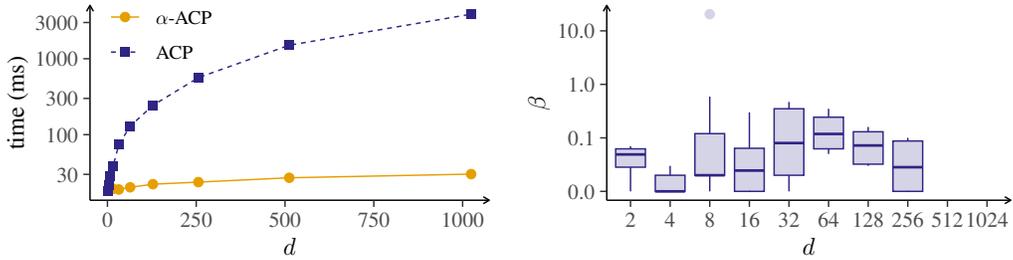
\begin{figure}
	\centering
	\resizebox{0.49\textwidth}{!}{
\begin{tikzpicture}[x=1pt,y=1pt]
\definecolor{fillColor}{RGB}{255,255,255}
\path[use as bounding box,fill=fillColor,fill opacity=0.00] (0,0) rectangle (209.58,115.63);
\begin{scope}
\path[clip] (  0.00,  0.00) rectangle (209.58,115.63);
\definecolor{drawColor}{RGB}{255,255,255}
\definecolor{fillColor}{RGB}{255,255,255}

\path[draw=drawColor,line width= 0.6pt,line join=round,line cap=round,fill=fillColor] (  0.00,  0.00) rectangle (209.58,115.63);
\end{scope}
\begin{scope}
\path[clip] ( 39.63, 29.80) rectangle (204.08,110.13);
\definecolor{fillColor}{RGB}{255,255,255}

\path[fill=fillColor] ( 39.63, 29.80) rectangle (204.08,110.13);
\definecolor{drawColor}{RGB}{230,159,0}

\path[draw=drawColor,line width= 0.6pt,line join=round] ( 47.10, 33.78) --
	( 47.39, 34.12) --
	( 47.98, 34.21) --
	( 49.15, 34.52) --
	( 51.49, 34.24) --
	( 56.17, 35.13) --
	( 65.53, 36.41) --
	( 84.26, 37.26) --
	(121.71, 39.05) --
	(196.61, 40.56);
\definecolor{drawColor}{RGB}{46,37,133}

\path[draw=drawColor,line width= 0.6pt,dash pattern=on 2pt off 2pt ,line join=round] ( 47.10, 33.45) --
	( 47.39, 35.86) --
	( 47.98, 39.91) --
	( 49.15, 43.90) --
	( 51.49, 52.66) --
	( 56.17, 60.30) --
	( 65.53, 68.91) --
	( 84.26, 80.31) --
	(121.71, 93.69) --
	(196.61,106.48);
\definecolor{drawColor}{RGB}{230,159,0}
\definecolor{fillColor}{RGB}{230,159,0}

\path[draw=drawColor,line width= 0.4pt,line join=round,line cap=round,fill=fillColor] ( 51.49, 34.24) circle (  1.96);

\path[draw=drawColor,line width= 0.4pt,line join=round,line cap=round,fill=fillColor] ( 47.39, 34.12) circle (  1.96);

\path[draw=drawColor,line width= 0.4pt,line join=round,line cap=round,fill=fillColor] ( 49.15, 34.52) circle (  1.96);

\path[draw=drawColor,line width= 0.4pt,line join=round,line cap=round,fill=fillColor] ( 56.17, 35.13) circle (  1.96);

\path[draw=drawColor,line width= 0.4pt,line join=round,line cap=round,fill=fillColor] ( 47.10, 33.78) circle (  1.96);

\path[draw=drawColor,line width= 0.4pt,line join=round,line cap=round,fill=fillColor] (196.61, 40.56) circle (  1.96);

\path[draw=drawColor,line width= 0.4pt,line join=round,line cap=round,fill=fillColor] ( 47.98, 34.21) circle (  1.96);

\path[draw=drawColor,line width= 0.4pt,line join=round,line cap=round,fill=fillColor] (121.71, 39.05) circle (  1.96);

\path[draw=drawColor,line width= 0.4pt,line join=round,line cap=round,fill=fillColor] ( 65.53, 36.41) circle (  1.96);

\path[draw=drawColor,line width= 0.4pt,line join=round,line cap=round,fill=fillColor] ( 84.26, 37.26) circle (  1.96);
\definecolor{fillColor}{RGB}{46,37,133}

\path[fill=fillColor] ( 49.53, 50.70) --
	( 53.45, 50.70) --
	( 53.45, 54.63) --
	( 49.53, 54.63) --
	cycle;

\path[fill=fillColor] ( 45.43, 33.90) --
	( 49.36, 33.90) --
	( 49.36, 37.82) --
	( 45.43, 37.82) --
	cycle;

\path[fill=fillColor] ( 47.19, 41.93) --
	( 51.11, 41.93) --
	( 51.11, 45.86) --
	( 47.19, 45.86) --
	cycle;

\path[fill=fillColor] ( 54.21, 58.34) --
	( 58.13, 58.34) --
	( 58.13, 62.26) --
	( 54.21, 62.26) --
	cycle;

\path[fill=fillColor] ( 45.14, 31.49) --
	( 49.06, 31.49) --
	( 49.06, 35.42) --
	( 45.14, 35.42) --
	cycle;

\path[fill=fillColor] (194.65,104.52) --
	(198.57,104.52) --
	(198.57,108.44) --
	(194.65,108.44) --
	cycle;

\path[fill=fillColor] ( 46.02, 37.95) --
	( 49.94, 37.95) --
	( 49.94, 41.87) --
	( 46.02, 41.87) --
	cycle;

\path[fill=fillColor] (119.75, 91.72) --
	(123.67, 91.72) --
	(123.67, 95.65) --
	(119.75, 95.65) --
	cycle;

\path[fill=fillColor] ( 63.57, 66.95) --
	( 67.50, 66.95) --
	( 67.50, 70.87) --
	( 63.57, 70.87) --
	cycle;

\path[fill=fillColor] ( 82.30, 78.35) --
	( 86.22, 78.35) --
	( 86.22, 82.27) --
	( 82.30, 82.27) --
	cycle;
\end{scope}
\begin{scope}
\path[clip] (  0.00,  0.00) rectangle (209.58,115.63);
\definecolor{drawColor}{RGB}{0,0,0}

\path[draw=drawColor,line width= 0.6pt,line join=round] ( 39.63, 29.80) --
	( 39.63,110.13);

\path[draw=drawColor,line width= 0.6pt,line join=round] ( 41.05,107.67) --
	( 39.63,110.13) --
	( 38.20,107.67);
\end{scope}
\begin{scope}
\path[clip] (  0.00,  0.00) rectangle (209.58,115.63);
\definecolor{drawColor}{gray}{0.30}

\node[text=drawColor,anchor=base east,inner sep=0pt, outer sep=0pt, scale=  0.88] at ( 34.68, 37.44) {30};

\node[text=drawColor,anchor=base east,inner sep=0pt, outer sep=0pt, scale=  0.88] at ( 34.68, 53.80) {100};

\node[text=drawColor,anchor=base east,inner sep=0pt, outer sep=0pt, scale=  0.88] at ( 34.68, 68.73) {300};

\node[text=drawColor,anchor=base east,inner sep=0pt, outer sep=0pt, scale=  0.88] at ( 34.68, 85.09) {1000};

\node[text=drawColor,anchor=base east,inner sep=0pt, outer sep=0pt, scale=  0.88] at ( 34.68,100.02) {3000};
\end{scope}
\begin{scope}
\path[clip] (  0.00,  0.00) rectangle (209.58,115.63);
\definecolor{drawColor}{gray}{0.20}

\path[draw=drawColor,line width= 0.6pt,line join=round] ( 36.88, 40.47) --
	( 39.63, 40.47);

\path[draw=drawColor,line width= 0.6pt,line join=round] ( 36.88, 56.83) --
	( 39.63, 56.83);

\path[draw=drawColor,line width= 0.6pt,line join=round] ( 36.88, 71.76) --
	( 39.63, 71.76);

\path[draw=drawColor,line width= 0.6pt,line join=round] ( 36.88, 88.12) --
	( 39.63, 88.12);

\path[draw=drawColor,line width= 0.6pt,line join=round] ( 36.88,103.05) --
	( 39.63,103.05);
\end{scope}
\begin{scope}
\path[clip] (  0.00,  0.00) rectangle (209.58,115.63);
\definecolor{drawColor}{RGB}{0,0,0}

\path[draw=drawColor,line width= 0.6pt,line join=round] ( 39.63, 29.80) --
	(204.08, 29.80);

\path[draw=drawColor,line width= 0.6pt,line join=round] (201.62, 28.38) --
	(204.08, 29.80) --
	(201.62, 31.23);
\end{scope}
\begin{scope}
\path[clip] (  0.00,  0.00) rectangle (209.58,115.63);
\definecolor{drawColor}{gray}{0.20}

\path[draw=drawColor,line width= 0.6pt,line join=round] ( 46.81, 27.05) --
	( 46.81, 29.80);

\path[draw=drawColor,line width= 0.6pt,line join=round] ( 83.38, 27.05) --
	( 83.38, 29.80);

\path[draw=drawColor,line width= 0.6pt,line join=round] (119.95, 27.05) --
	(119.95, 29.80);

\path[draw=drawColor,line width= 0.6pt,line join=round] (156.53, 27.05) --
	(156.53, 29.80);

\path[draw=drawColor,line width= 0.6pt,line join=round] (193.10, 27.05) --
	(193.10, 29.80);
\end{scope}
\begin{scope}
\path[clip] (  0.00,  0.00) rectangle (209.58,115.63);
\definecolor{drawColor}{gray}{0.30}

\node[text=drawColor,anchor=base,inner sep=0pt, outer sep=0pt, scale=  0.88] at ( 46.81, 18.79) {0};

\node[text=drawColor,anchor=base,inner sep=0pt, outer sep=0pt, scale=  0.88] at ( 83.38, 18.79) {250};

\node[text=drawColor,anchor=base,inner sep=0pt, outer sep=0pt, scale=  0.88] at (119.95, 18.79) {500};

\node[text=drawColor,anchor=base,inner sep=0pt, outer sep=0pt, scale=  0.88] at (156.53, 18.79) {750};

\node[text=drawColor,anchor=base,inner sep=0pt, outer sep=0pt, scale=  0.88] at (193.10, 18.79) {1000};
\end{scope}
\begin{scope}
\path[clip] (  0.00,  0.00) rectangle (209.58,115.63);
\definecolor{drawColor}{RGB}{0,0,0}

\node[text=drawColor,anchor=base,inner sep=0pt, outer sep=0pt, scale=  1.00] at (121.85,  7.44) {$d$};
\end{scope}
\begin{scope}
\path[clip] (  0.00,  0.00) rectangle (209.58,115.63);
\definecolor{drawColor}{RGB}{0,0,0}

\node[text=drawColor,rotate= 90.00,anchor=base,inner sep=0pt, outer sep=0pt, scale=  1.00] at ( 12.39, 69.97) {time (ms)};
\end{scope}
\begin{scope}
\path[clip] (  0.00,  0.00) rectangle (209.58,115.63);

\path[] ( 41.35, 78.13) rectangle ( 97.11,118.04);
\end{scope}
\begin{scope}
\path[clip] (  0.00,  0.00) rectangle (209.58,115.63);
\definecolor{drawColor}{RGB}{230,159,0}

\path[draw=drawColor,line width= 0.6pt,line join=round] ( 48.29,105.31) -- ( 59.86,105.31);
\end{scope}
\begin{scope}
\path[clip] (  0.00,  0.00) rectangle (209.58,115.63);
\definecolor{drawColor}{RGB}{230,159,0}
\definecolor{fillColor}{RGB}{230,159,0}

\path[draw=drawColor,line width= 0.4pt,line join=round,line cap=round,fill=fillColor] ( 54.07,105.31) circle (  1.96);
\end{scope}
\begin{scope}
\path[clip] (  0.00,  0.00) rectangle (209.58,115.63);
\definecolor{drawColor}{RGB}{46,37,133}

\path[draw=drawColor,line width= 0.6pt,dash pattern=on 2pt off 2pt ,line join=round] ( 48.29, 90.86) -- ( 59.86, 90.86);
\end{scope}
\begin{scope}
\path[clip] (  0.00,  0.00) rectangle (209.58,115.63);
\definecolor{fillColor}{RGB}{46,37,133}

\path[fill=fillColor] ( 52.11, 88.89) --
	( 56.04, 88.89) --
	( 56.04, 92.82) --
	( 52.11, 92.82) --
	cycle;
\end{scope}
\begin{scope}
\path[clip] (  0.00,  0.00) rectangle (209.58,115.63);
\definecolor{drawColor}{RGB}{0,0,0}

\node[text=drawColor,anchor=base west,inner sep=0pt, outer sep=0pt, scale=  0.80] at ( 66.80,102.55) {$\alpha$-ACP};
\end{scope}
\begin{scope}
\path[clip] (  0.00,  0.00) rectangle (209.58,115.63);
\definecolor{drawColor}{RGB}{0,0,0}

\node[text=drawColor,anchor=base west,inner sep=0pt, outer sep=0pt, scale=  0.80] at ( 66.80, 88.10) {ACP};
\end{scope}
\end{tikzpicture}}
	\resizebox{0.49\textwidth}{!}{
\begin{tikzpicture}[x=1pt,y=1pt]
\definecolor{fillColor}{RGB}{255,255,255}
\path[use as bounding box,fill=fillColor,fill opacity=0.00] (0,0) rectangle (209.58,115.63);
\begin{scope}
\path[clip] (  0.00,  0.00) rectangle (209.58,115.63);
\definecolor{drawColor}{RGB}{255,255,255}
\definecolor{fillColor}{RGB}{255,255,255}

\path[draw=drawColor,line width= 0.6pt,line join=round,line cap=round,fill=fillColor] (  0.00,  0.00) rectangle (209.58,115.63);
\end{scope}
\begin{scope}
\path[clip] ( 37.67, 29.80) rectangle (204.08,110.13);
\definecolor{fillColor}{RGB}{255,255,255}

\path[fill=fillColor] ( 37.67, 29.80) rectangle (204.08,110.13);
\definecolor{drawColor}{RGB}{46,37,133}

\path[draw=drawColor,line width= 0.6pt,line join=round] ( 47.46, 50.98) -- ( 47.46, 52.09);

\path[draw=drawColor,line width= 0.6pt,line join=round] ( 47.46, 43.41) -- ( 47.46, 33.45);
\definecolor{fillColor}{RGB}{46,37,133}

\path[draw=drawColor,line width= 0.6pt,fill=fillColor,fill opacity=0.20] ( 41.34, 50.98) --
	( 41.34, 43.41) --
	( 53.58, 43.41) --
	( 53.58, 50.98) --
	( 41.34, 50.98) --
	cycle;

\path[draw=drawColor,line width= 1.1pt] ( 41.34, 48.67) -- ( 53.58, 48.67);

\path[draw=drawColor,line width= 0.6pt,line join=round] ( 63.78, 40.09) -- ( 63.78, 43.97);

\path[draw=drawColor,line width= 0.6pt,line join=round] ( 63.78, 33.45) -- ( 63.78, 33.45);

\path[draw=drawColor,line width= 0.6pt,fill=fillColor,fill opacity=0.20] ( 57.66, 40.09) --
	( 57.66, 33.45) --
	( 69.89, 33.45) --
	( 69.89, 40.09) --
	( 57.66, 40.09) --
	cycle;

\path[draw=drawColor,line width= 1.1pt] ( 57.66, 33.45) -- ( 69.89, 33.45);
\definecolor{drawColor}{RGB}{46,37,133}

\path[draw=drawColor,draw opacity=0.20,line width= 0.4pt,line join=round,line cap=round,fill=fillColor,fill opacity=0.20] ( 80.09,106.48) circle (  1.96);
\definecolor{drawColor}{RGB}{46,37,133}

\path[draw=drawColor,line width= 0.6pt,line join=round] ( 80.09, 57.25) -- ( 80.09, 72.50);

\path[draw=drawColor,line width= 0.6pt,line join=round] ( 80.09, 40.09) -- ( 80.09, 33.45);

\path[draw=drawColor,line width= 0.6pt,fill=fillColor,fill opacity=0.20] ( 73.97, 57.25) --
	( 73.97, 40.09) --
	( 86.21, 40.09) --
	( 86.21, 57.25) --
	( 73.97, 57.25) --
	cycle;

\path[draw=drawColor,line width= 1.1pt] ( 73.97, 40.09) -- ( 86.21, 40.09);

\path[draw=drawColor,line width= 0.6pt,line join=round] ( 96.41, 51.21) -- ( 96.41, 66.03);

\path[draw=drawColor,line width= 0.6pt,line join=round] ( 96.41, 33.45) -- ( 96.41, 33.45);

\path[draw=drawColor,line width= 0.6pt,fill=fillColor,fill opacity=0.20] ( 90.29, 51.21) --
	( 90.29, 33.45) --
	(102.52, 33.45) --
	(102.52, 51.21) --
	( 90.29, 51.21) --
	cycle;

\path[draw=drawColor,line width= 1.1pt] ( 90.29, 42.03) -- (102.52, 42.03);

\path[draw=drawColor,line width= 0.6pt,line join=round] (112.72, 67.50) -- (112.72, 70.33);

\path[draw=drawColor,line width= 0.6pt,line join=round] (112.72, 40.09) -- (112.72, 33.45);

\path[draw=drawColor,line width= 0.6pt,fill=fillColor,fill opacity=0.20] (106.60, 67.50) --
	(106.60, 40.09) --
	(118.84, 40.09) --
	(118.84, 67.50) --
	(106.60, 67.50) --
	cycle;

\path[draw=drawColor,line width= 1.1pt] (106.60, 53.37) -- (118.84, 53.37);

\path[draw=drawColor,line width= 0.6pt,line join=round] (129.03, 64.03) -- (129.03, 67.50);

\path[draw=drawColor,line width= 0.6pt,line join=round] (129.03, 50.98) -- (129.03, 48.87);

\path[draw=drawColor,line width= 0.6pt,fill=fillColor,fill opacity=0.20] (122.92, 64.03) --
	(122.92, 50.98) --
	(135.15, 50.98) --
	(135.15, 64.03) --
	(122.92, 64.03) --
	cycle;

\path[draw=drawColor,line width= 1.1pt] (122.92, 57.12) -- (135.15, 57.12);

\path[draw=drawColor,line width= 0.6pt,line join=round] (145.35, 58.02) -- (145.35, 60.01);

\path[draw=drawColor,line width= 0.6pt,line join=round] (145.35, 44.66) -- (145.35, 43.97);

\path[draw=drawColor,line width= 0.6pt,fill=fillColor,fill opacity=0.20] (139.23, 58.02) --
	(139.23, 44.66) --
	(151.47, 44.66) --
	(151.47, 58.02) --
	(139.23, 58.02) --
	cycle;

\path[draw=drawColor,line width= 1.1pt] (139.23, 52.37) -- (151.47, 52.37);

\path[draw=drawColor,line width= 0.6pt,line join=round] (161.66, 54.21) -- (161.66, 55.50);

\path[draw=drawColor,line width= 0.6pt,line join=round] (161.66, 33.45) -- (161.66, 33.45);

\path[draw=drawColor,line width= 0.6pt,fill=fillColor,fill opacity=0.20] (155.55, 54.21) --
	(155.55, 33.45) --
	(167.78, 33.45) --
	(167.78, 54.21) --
	(155.55, 54.21) --
	cycle;

\path[draw=drawColor,line width= 1.1pt] (155.55, 43.41) -- (167.78, 43.41);
\end{scope}
\begin{scope}
\path[clip] (  0.00,  0.00) rectangle (209.58,115.63);
\definecolor{drawColor}{RGB}{0,0,0}

\path[draw=drawColor,line width= 0.6pt,line join=round] ( 37.67, 29.80) --
	( 37.67,110.13);

\path[draw=drawColor,line width= 0.6pt,line join=round] ( 39.09,107.67) --
	( 37.67,110.13) --
	( 36.25,107.67);
\end{scope}
\begin{scope}
\path[clip] (  0.00,  0.00) rectangle (209.58,115.63);
\definecolor{drawColor}{gray}{0.30}

\node[text=drawColor,anchor=base east,inner sep=0pt, outer sep=0pt, scale=  0.88] at ( 32.72, 30.42) { 0.0};

\node[text=drawColor,anchor=base east,inner sep=0pt, outer sep=0pt, scale=  0.88] at ( 32.72, 52.47) { 0.1};

\node[text=drawColor,anchor=base east,inner sep=0pt, outer sep=0pt, scale=  0.88] at ( 32.72, 74.53) { 1.0};

\node[text=drawColor,anchor=base east,inner sep=0pt, outer sep=0pt, scale=  0.88] at ( 32.72, 96.58) {10.0};
\end{scope}
\begin{scope}
\path[clip] (  0.00,  0.00) rectangle (209.58,115.63);
\definecolor{drawColor}{gray}{0.20}

\path[draw=drawColor,line width= 0.6pt,line join=round] ( 34.92, 33.45) --
	( 37.67, 33.45);

\path[draw=drawColor,line width= 0.6pt,line join=round] ( 34.92, 55.50) --
	( 37.67, 55.50);

\path[draw=drawColor,line width= 0.6pt,line join=round] ( 34.92, 77.56) --
	( 37.67, 77.56);

\path[draw=drawColor,line width= 0.6pt,line join=round] ( 34.92, 99.61) --
	( 37.67, 99.61);
\end{scope}
\begin{scope}
\path[clip] (  0.00,  0.00) rectangle (209.58,115.63);
\definecolor{drawColor}{RGB}{0,0,0}

\path[draw=drawColor,line width= 0.6pt,line join=round] ( 37.67, 29.80) --
	(204.08, 29.80);

\path[draw=drawColor,line width= 0.6pt,line join=round] (201.62, 28.38) --
	(204.08, 29.80) --
	(201.62, 31.23);
\end{scope}
\begin{scope}
\path[clip] (  0.00,  0.00) rectangle (209.58,115.63);
\definecolor{drawColor}{gray}{0.20}

\path[draw=drawColor,line width= 0.6pt,line join=round] ( 47.46, 27.05) --
	( 47.46, 29.80);

\path[draw=drawColor,line width= 0.6pt,line join=round] ( 63.78, 27.05) --
	( 63.78, 29.80);

\path[draw=drawColor,line width= 0.6pt,line join=round] ( 80.09, 27.05) --
	( 80.09, 29.80);

\path[draw=drawColor,line width= 0.6pt,line join=round] ( 96.41, 27.05) --
	( 96.41, 29.80);

\path[draw=drawColor,line width= 0.6pt,line join=round] (112.72, 27.05) --
	(112.72, 29.80);

\path[draw=drawColor,line width= 0.6pt,line join=round] (129.03, 27.05) --
	(129.03, 29.80);

\path[draw=drawColor,line width= 0.6pt,line join=round] (145.35, 27.05) --
	(145.35, 29.80);

\path[draw=drawColor,line width= 0.6pt,line join=round] (161.66, 27.05) --
	(161.66, 29.80);

\path[draw=drawColor,line width= 0.6pt,line join=round] (177.98, 27.05) --
	(177.98, 29.80);

\path[draw=drawColor,line width= 0.6pt,line join=round] (194.29, 27.05) --
	(194.29, 29.80);
\end{scope}
\begin{scope}
\path[clip] (  0.00,  0.00) rectangle (209.58,115.63);
\definecolor{drawColor}{gray}{0.30}

\node[text=drawColor,anchor=base,inner sep=0pt, outer sep=0pt, scale=  0.88] at ( 47.46, 18.79) {2};

\node[text=drawColor,anchor=base,inner sep=0pt, outer sep=0pt, scale=  0.88] at ( 63.78, 18.79) {4};

\node[text=drawColor,anchor=base,inner sep=0pt, outer sep=0pt, scale=  0.88] at ( 80.09, 18.79) {8};

\node[text=drawColor,anchor=base,inner sep=0pt, outer sep=0pt, scale=  0.88] at ( 96.41, 18.79) {16};

\node[text=drawColor,anchor=base,inner sep=0pt, outer sep=0pt, scale=  0.88] at (112.72, 18.79) {32};

\node[text=drawColor,anchor=base,inner sep=0pt, outer sep=0pt, scale=  0.88] at (129.03, 18.79) {64};

\node[text=drawColor,anchor=base,inner sep=0pt, outer sep=0pt, scale=  0.88] at (145.35, 18.79) {128};

\node[text=drawColor,anchor=base,inner sep=0pt, outer sep=0pt, scale=  0.88] at (161.66, 18.79) {256};

\node[text=drawColor,anchor=base,inner sep=0pt, outer sep=0pt, scale=  0.88] at (177.98, 18.79) {512};

\node[text=drawColor,anchor=base,inner sep=0pt, outer sep=0pt, scale=  0.88] at (194.29, 18.79) {1024};
\end{scope}
\begin{scope}
\path[clip] (  0.00,  0.00) rectangle (209.58,115.63);
\definecolor{drawColor}{RGB}{0,0,0}

\node[text=drawColor,anchor=base,inner sep=0pt, outer sep=0pt, scale=  1.00] at (120.88,  7.44) {$d$};
\end{scope}
\begin{scope}
\path[clip] (  0.00,  0.00) rectangle (209.58,115.63);
\definecolor{drawColor}{RGB}{0,0,0}

\node[text=drawColor,rotate= 90.00,anchor=base,inner sep=0pt, outer sep=0pt, scale=  1.00] at ( 12.39, 69.97) {$\beta$};
\end{scope}
\end{tikzpicture}}
	\caption{Average query times of \acl{lve} run on the output of \ac{acp} and \ac{aacp} where the input \acp{fg} contain a proportion of $p=0.1$ scaled factors (left), and the distribution of the number $\beta$ of queries after which the offline overhead of \ac{aacp} amortises on input \acp{fg} containing a proportion of $p=0.1$ scaled factors (right).}
	\label{fig:aacp_plot_app_p=0.1}
\end{figure}
\begin{figure}
	\centering
	\resizebox{0.49\textwidth}{!}{
\begin{tikzpicture}[x=1pt,y=1pt]
\definecolor{fillColor}{RGB}{255,255,255}
\path[use as bounding box,fill=fillColor,fill opacity=0.00] (0,0) rectangle (209.58,115.63);
\begin{scope}
\path[clip] (  0.00,  0.00) rectangle (209.58,115.63);
\definecolor{drawColor}{RGB}{255,255,255}
\definecolor{fillColor}{RGB}{255,255,255}

\path[draw=drawColor,line width= 0.6pt,line join=round,line cap=round,fill=fillColor] (  0.00,  0.00) rectangle (209.58,115.63);
\end{scope}
\begin{scope}
\path[clip] ( 44.03, 29.80) rectangle (204.08,110.13);
\definecolor{fillColor}{RGB}{255,255,255}

\path[fill=fillColor] ( 44.03, 29.80) rectangle (204.08,110.13);
\definecolor{drawColor}{RGB}{230,159,0}

\path[draw=drawColor,line width= 0.6pt,line join=round] ( 51.30, 33.45) --
	( 51.59, 34.75) --
	( 52.16, 34.99) --
	( 53.29, 34.62) --
	( 55.57, 34.62) --
	( 60.13, 35.13) --
	( 69.24, 35.63) --
	( 87.46, 35.97) --
	(123.91, 37.40) --
	(196.81, 39.67);
\definecolor{drawColor}{RGB}{46,37,133}

\path[draw=drawColor,line width= 0.6pt,dash pattern=on 2pt off 2pt ,line join=round] ( 51.30, 33.55) --
	( 51.59, 37.01) --
	( 52.16, 40.94) --
	( 53.29, 44.79) --
	( 55.57, 51.18) --
	( 60.13, 60.98) --
	( 69.24, 69.88) --
	( 87.46, 80.49) --
	(123.91, 94.48) --
	(196.81,106.48);
\definecolor{drawColor}{RGB}{230,159,0}
\definecolor{fillColor}{RGB}{230,159,0}

\path[draw=drawColor,line width= 0.4pt,line join=round,line cap=round,fill=fillColor] ( 55.57, 34.62) circle (  1.96);

\path[draw=drawColor,line width= 0.4pt,line join=round,line cap=round,fill=fillColor] ( 51.59, 34.75) circle (  1.96);

\path[draw=drawColor,line width= 0.4pt,line join=round,line cap=round,fill=fillColor] ( 53.29, 34.62) circle (  1.96);

\path[draw=drawColor,line width= 0.4pt,line join=round,line cap=round,fill=fillColor] ( 60.13, 35.13) circle (  1.96);

\path[draw=drawColor,line width= 0.4pt,line join=round,line cap=round,fill=fillColor] ( 51.30, 33.45) circle (  1.96);

\path[draw=drawColor,line width= 0.4pt,line join=round,line cap=round,fill=fillColor] (196.81, 39.67) circle (  1.96);

\path[draw=drawColor,line width= 0.4pt,line join=round,line cap=round,fill=fillColor] ( 52.16, 34.99) circle (  1.96);

\path[draw=drawColor,line width= 0.4pt,line join=round,line cap=round,fill=fillColor] (123.91, 37.40) circle (  1.96);

\path[draw=drawColor,line width= 0.4pt,line join=round,line cap=round,fill=fillColor] ( 69.24, 35.63) circle (  1.96);

\path[draw=drawColor,line width= 0.4pt,line join=round,line cap=round,fill=fillColor] ( 87.46, 35.97) circle (  1.96);
\definecolor{fillColor}{RGB}{46,37,133}

\path[fill=fillColor] ( 53.61, 49.22) --
	( 57.53, 49.22) --
	( 57.53, 53.15) --
	( 53.61, 53.15) --
	cycle;

\path[fill=fillColor] ( 49.62, 35.04) --
	( 53.55, 35.04) --
	( 53.55, 38.97) --
	( 49.62, 38.97) --
	cycle;

\path[fill=fillColor] ( 51.33, 42.83) --
	( 55.26, 42.83) --
	( 55.26, 46.76) --
	( 51.33, 46.76) --
	cycle;

\path[fill=fillColor] ( 58.17, 59.01) --
	( 62.09, 59.01) --
	( 62.09, 62.94) --
	( 58.17, 62.94) --
	cycle;

\path[fill=fillColor] ( 49.34, 31.58) --
	( 53.26, 31.58) --
	( 53.26, 35.51) --
	( 49.34, 35.51) --
	cycle;

\path[fill=fillColor] (194.85,104.52) --
	(198.77,104.52) --
	(198.77,108.44) --
	(194.85,108.44) --
	cycle;

\path[fill=fillColor] ( 50.19, 38.97) --
	( 54.12, 38.97) --
	( 54.12, 42.90) --
	( 50.19, 42.90) --
	cycle;

\path[fill=fillColor] (121.95, 92.52) --
	(125.87, 92.52) --
	(125.87, 96.44) --
	(121.95, 96.44) --
	cycle;

\path[fill=fillColor] ( 67.28, 67.92) --
	( 71.20, 67.92) --
	( 71.20, 71.84) --
	( 67.28, 71.84) --
	cycle;

\path[fill=fillColor] ( 85.50, 78.53) --
	( 89.43, 78.53) --
	( 89.43, 82.45) --
	( 85.50, 82.45) --
	cycle;
\end{scope}
\begin{scope}
\path[clip] (  0.00,  0.00) rectangle (209.58,115.63);
\definecolor{drawColor}{RGB}{0,0,0}

\path[draw=drawColor,line width= 0.6pt,line join=round] ( 44.03, 29.80) --
	( 44.03,110.13);

\path[draw=drawColor,line width= 0.6pt,line join=round] ( 45.45,107.67) --
	( 44.03,110.13) --
	( 42.60,107.67);
\end{scope}
\begin{scope}
\path[clip] (  0.00,  0.00) rectangle (209.58,115.63);
\definecolor{drawColor}{gray}{0.30}

\node[text=drawColor,anchor=base east,inner sep=0pt, outer sep=0pt, scale=  0.88] at ( 39.08, 48.92) {100};

\node[text=drawColor,anchor=base east,inner sep=0pt, outer sep=0pt, scale=  0.88] at ( 39.08, 73.54) {1000};

\node[text=drawColor,anchor=base east,inner sep=0pt, outer sep=0pt, scale=  0.88] at ( 39.08, 98.16) {10000};
\end{scope}
\begin{scope}
\path[clip] (  0.00,  0.00) rectangle (209.58,115.63);
\definecolor{drawColor}{gray}{0.20}

\path[draw=drawColor,line width= 0.6pt,line join=round] ( 41.28, 51.95) --
	( 44.03, 51.95);

\path[draw=drawColor,line width= 0.6pt,line join=round] ( 41.28, 76.57) --
	( 44.03, 76.57);

\path[draw=drawColor,line width= 0.6pt,line join=round] ( 41.28,101.19) --
	( 44.03,101.19);
\end{scope}
\begin{scope}
\path[clip] (  0.00,  0.00) rectangle (209.58,115.63);
\definecolor{drawColor}{RGB}{0,0,0}

\path[draw=drawColor,line width= 0.6pt,line join=round] ( 44.03, 29.80) --
	(204.08, 29.80);

\path[draw=drawColor,line width= 0.6pt,line join=round] (201.62, 28.38) --
	(204.08, 29.80) --
	(201.62, 31.23);
\end{scope}
\begin{scope}
\path[clip] (  0.00,  0.00) rectangle (209.58,115.63);
\definecolor{drawColor}{gray}{0.20}

\path[draw=drawColor,line width= 0.6pt,line join=round] ( 51.02, 27.05) --
	( 51.02, 29.80);

\path[draw=drawColor,line width= 0.6pt,line join=round] ( 86.61, 27.05) --
	( 86.61, 29.80);

\path[draw=drawColor,line width= 0.6pt,line join=round] (122.20, 27.05) --
	(122.20, 29.80);

\path[draw=drawColor,line width= 0.6pt,line join=round] (157.80, 27.05) --
	(157.80, 29.80);

\path[draw=drawColor,line width= 0.6pt,line join=round] (193.39, 27.05) --
	(193.39, 29.80);
\end{scope}
\begin{scope}
\path[clip] (  0.00,  0.00) rectangle (209.58,115.63);
\definecolor{drawColor}{gray}{0.30}

\node[text=drawColor,anchor=base,inner sep=0pt, outer sep=0pt, scale=  0.88] at ( 51.02, 18.79) {0};

\node[text=drawColor,anchor=base,inner sep=0pt, outer sep=0pt, scale=  0.88] at ( 86.61, 18.79) {250};

\node[text=drawColor,anchor=base,inner sep=0pt, outer sep=0pt, scale=  0.88] at (122.20, 18.79) {500};

\node[text=drawColor,anchor=base,inner sep=0pt, outer sep=0pt, scale=  0.88] at (157.80, 18.79) {750};

\node[text=drawColor,anchor=base,inner sep=0pt, outer sep=0pt, scale=  0.88] at (193.39, 18.79) {1000};
\end{scope}
\begin{scope}
\path[clip] (  0.00,  0.00) rectangle (209.58,115.63);
\definecolor{drawColor}{RGB}{0,0,0}

\node[text=drawColor,anchor=base,inner sep=0pt, outer sep=0pt, scale=  1.00] at (124.05,  7.44) {$d$};
\end{scope}
\begin{scope}
\path[clip] (  0.00,  0.00) rectangle (209.58,115.63);
\definecolor{drawColor}{RGB}{0,0,0}

\node[text=drawColor,rotate= 90.00,anchor=base,inner sep=0pt, outer sep=0pt, scale=  1.00] at ( 12.39, 69.97) {time (ms)};
\end{scope}
\begin{scope}
\path[clip] (  0.00,  0.00) rectangle (209.58,115.63);

\path[] ( 44.95, 78.13) rectangle (100.72,118.04);
\end{scope}
\begin{scope}
\path[clip] (  0.00,  0.00) rectangle (209.58,115.63);
\definecolor{drawColor}{RGB}{230,159,0}

\path[draw=drawColor,line width= 0.6pt,line join=round] ( 51.90,105.31) -- ( 63.46,105.31);
\end{scope}
\begin{scope}
\path[clip] (  0.00,  0.00) rectangle (209.58,115.63);
\definecolor{drawColor}{RGB}{230,159,0}
\definecolor{fillColor}{RGB}{230,159,0}

\path[draw=drawColor,line width= 0.4pt,line join=round,line cap=round,fill=fillColor] ( 57.68,105.31) circle (  1.96);
\end{scope}
\begin{scope}
\path[clip] (  0.00,  0.00) rectangle (209.58,115.63);
\definecolor{drawColor}{RGB}{46,37,133}

\path[draw=drawColor,line width= 0.6pt,dash pattern=on 2pt off 2pt ,line join=round] ( 51.90, 90.86) -- ( 63.46, 90.86);
\end{scope}
\begin{scope}
\path[clip] (  0.00,  0.00) rectangle (209.58,115.63);
\definecolor{fillColor}{RGB}{46,37,133}

\path[fill=fillColor] ( 55.72, 88.89) --
	( 59.64, 88.89) --
	( 59.64, 92.82) --
	( 55.72, 92.82) --
	cycle;
\end{scope}
\begin{scope}
\path[clip] (  0.00,  0.00) rectangle (209.58,115.63);
\definecolor{drawColor}{RGB}{0,0,0}

\node[text=drawColor,anchor=base west,inner sep=0pt, outer sep=0pt, scale=  0.80] at ( 70.41,102.55) {$\alpha$-ACP};
\end{scope}
\begin{scope}
\path[clip] (  0.00,  0.00) rectangle (209.58,115.63);
\definecolor{drawColor}{RGB}{0,0,0}

\node[text=drawColor,anchor=base west,inner sep=0pt, outer sep=0pt, scale=  0.80] at ( 70.41, 88.10) {ACP};
\end{scope}
\end{tikzpicture}}
	\resizebox{0.49\textwidth}{!}{
\begin{tikzpicture}[x=1pt,y=1pt]
\definecolor{fillColor}{RGB}{255,255,255}
\path[use as bounding box,fill=fillColor,fill opacity=0.00] (0,0) rectangle (209.58,115.63);
\begin{scope}
\path[clip] (  0.00,  0.00) rectangle (209.58,115.63);
\definecolor{drawColor}{RGB}{255,255,255}
\definecolor{fillColor}{RGB}{255,255,255}

\path[draw=drawColor,line width= 0.6pt,line join=round,line cap=round,fill=fillColor] (  0.00,  0.00) rectangle (209.58,115.63);
\end{scope}
\begin{scope}
\path[clip] ( 33.27, 29.80) rectangle (204.08,110.13);
\definecolor{fillColor}{RGB}{255,255,255}

\path[fill=fillColor] ( 33.27, 29.80) rectangle (204.08,110.13);
\definecolor{drawColor}{RGB}{46,37,133}

\path[draw=drawColor,line width= 0.6pt,line join=round] ( 43.32, 75.08) -- ( 43.32, 83.39);

\path[draw=drawColor,line width= 0.6pt,line join=round] ( 43.32, 33.45) -- ( 43.32, 33.45);
\definecolor{fillColor}{RGB}{46,37,133}

\path[draw=drawColor,line width= 0.6pt,fill=fillColor,fill opacity=0.20] ( 37.04, 75.08) --
	( 37.04, 33.45) --
	( 49.60, 33.45) --
	( 49.60, 75.08) --
	( 37.04, 75.08) --
	cycle;

\path[draw=drawColor,line width= 1.1pt] ( 37.04, 52.88) -- ( 49.60, 52.88);

\path[draw=drawColor,line width= 0.6pt,line join=round] ( 60.07, 48.49) -- ( 60.07, 48.49);

\path[draw=drawColor,line width= 0.6pt,line join=round] ( 60.07, 33.45) -- ( 60.07, 33.45);

\path[draw=drawColor,line width= 0.6pt,fill=fillColor,fill opacity=0.20] ( 53.79, 48.49) --
	( 53.79, 33.45) --
	( 66.35, 33.45) --
	( 66.35, 48.49) --
	( 53.79, 48.49) --
	cycle;

\path[draw=drawColor,line width= 1.1pt] ( 53.79, 40.97) -- ( 66.35, 40.97);
\definecolor{drawColor}{RGB}{46,37,133}

\path[draw=drawColor,draw opacity=0.20,line width= 0.4pt,line join=round,line cap=round,fill=fillColor,fill opacity=0.20] ( 76.81, 33.45) circle (  1.96);

\path[draw=drawColor,draw opacity=0.20,line width= 0.4pt,line join=round,line cap=round,fill=fillColor,fill opacity=0.20] ( 76.81,104.11) circle (  1.96);

\path[draw=drawColor,draw opacity=0.20,line width= 0.4pt,line join=round,line cap=round,fill=fillColor,fill opacity=0.20] ( 76.81, 93.58) circle (  1.96);

\path[draw=drawColor,draw opacity=0.20,line width= 0.4pt,line join=round,line cap=round,fill=fillColor,fill opacity=0.20] ( 76.81,106.48) circle (  1.96);
\definecolor{drawColor}{RGB}{46,37,133}

\path[draw=drawColor,line width= 0.6pt,line join=round] ( 76.81, 68.36) -- ( 76.81, 68.36);

\path[draw=drawColor,line width= 0.6pt,line join=round] ( 76.81, 57.28) -- ( 76.81, 48.49);

\path[draw=drawColor,line width= 0.6pt,fill=fillColor,fill opacity=0.20] ( 70.53, 68.36) --
	( 70.53, 57.28) --
	( 83.09, 57.28) --
	( 83.09, 68.36) --
	( 70.53, 68.36) --
	cycle;

\path[draw=drawColor,line width= 1.1pt] ( 70.53, 57.28) -- ( 83.09, 57.28);

\path[draw=drawColor,line width= 0.6pt,line join=round] ( 93.56, 72.46) -- ( 93.56, 92.18);

\path[draw=drawColor,line width= 0.6pt,line join=round] ( 93.56, 33.45) -- ( 93.56, 33.45);

\path[draw=drawColor,line width= 0.6pt,fill=fillColor,fill opacity=0.20] ( 87.28, 72.46) --
	( 87.28, 33.45) --
	( 99.84, 33.45) --
	( 99.84, 72.46) --
	( 87.28, 72.46) --
	cycle;

\path[draw=drawColor,line width= 1.1pt] ( 87.28, 33.45) -- ( 99.84, 33.45);

\path[draw=drawColor,line width= 0.6pt,line join=round] (110.30, 99.22) -- (110.30,100.49);

\path[draw=drawColor,line width= 0.6pt,line join=round] (110.30, 48.49) -- (110.30, 48.49);

\path[draw=drawColor,line width= 0.6pt,fill=fillColor,fill opacity=0.20] (104.03, 99.22) --
	(104.03, 48.49) --
	(116.58, 48.49) --
	(116.58, 99.22) --
	(104.03, 99.22) --
	cycle;

\path[draw=drawColor,line width= 1.1pt] (104.03, 73.45) -- (116.58, 73.45);

\path[draw=drawColor,line width= 0.6pt,line join=round] (127.05, 90.69) -- (127.05, 94.90);

\path[draw=drawColor,line width= 0.6pt,line join=round] (127.05, 33.45) -- (127.05, 33.45);

\path[draw=drawColor,line width= 0.6pt,fill=fillColor,fill opacity=0.20] (120.77, 90.69) --
	(120.77, 33.45) --
	(133.33, 33.45) --
	(133.33, 90.69) --
	(120.77, 90.69) --
	cycle;

\path[draw=drawColor,line width= 1.1pt] (120.77, 48.49) -- (133.33, 48.49);

\path[draw=drawColor,line width= 0.6pt,line join=round] (143.80, 75.66) -- (143.80, 75.66);

\path[draw=drawColor,line width= 0.6pt,line join=round] (143.80, 33.45) -- (143.80, 33.45);

\path[draw=drawColor,line width= 0.6pt,fill=fillColor,fill opacity=0.20] (137.52, 75.66) --
	(137.52, 33.45) --
	(150.08, 33.45) --
	(150.08, 75.66) --
	(137.52, 75.66) --
	cycle;

\path[draw=drawColor,line width= 1.1pt] (137.52, 54.55) -- (150.08, 54.55);

\path[draw=drawColor,line width= 0.6pt,line join=round] (160.54, 57.28) -- (160.54, 57.28);

\path[draw=drawColor,line width= 0.6pt,line join=round] (160.54, 45.37) -- (160.54, 33.45);

\path[draw=drawColor,line width= 0.6pt,fill=fillColor,fill opacity=0.20] (154.26, 57.28) --
	(154.26, 45.37) --
	(166.82, 45.37) --
	(166.82, 57.28) --
	(154.26, 57.28) --
	cycle;

\path[draw=drawColor,line width= 1.1pt] (154.26, 57.28) -- (166.82, 57.28);

\path[draw=drawColor,line width= 0.6pt,line join=round] (177.29, 33.45) -- (177.29, 33.45);

\path[draw=drawColor,line width= 0.6pt,line join=round] (177.29, 33.45) -- (177.29, 33.45);

\path[draw=drawColor,line width= 0.6pt,fill=fillColor,fill opacity=0.20] (171.01, 33.45) --
	(171.01, 33.45) --
	(183.57, 33.45) --
	(183.57, 33.45) --
	(171.01, 33.45) --
	cycle;

\path[draw=drawColor,line width= 1.1pt] (171.01, 33.45) -- (183.57, 33.45);
\end{scope}
\begin{scope}
\path[clip] (  0.00,  0.00) rectangle (209.58,115.63);
\definecolor{drawColor}{RGB}{0,0,0}

\path[draw=drawColor,line width= 0.6pt,line join=round] ( 33.27, 29.80) --
	( 33.27,110.13);

\path[draw=drawColor,line width= 0.6pt,line join=round] ( 34.70,107.67) --
	( 33.27,110.13) --
	( 31.85,107.67);
\end{scope}
\begin{scope}
\path[clip] (  0.00,  0.00) rectangle (209.58,115.63);
\definecolor{drawColor}{gray}{0.30}

\node[text=drawColor,anchor=base east,inner sep=0pt, outer sep=0pt, scale=  0.88] at ( 28.32, 30.42) {0.0};

\node[text=drawColor,anchor=base east,inner sep=0pt, outer sep=0pt, scale=  0.88] at ( 28.32, 54.25) {0.0};

\node[text=drawColor,anchor=base east,inner sep=0pt, outer sep=0pt, scale=  0.88] at ( 28.32, 80.36) {0.1};

\node[text=drawColor,anchor=base east,inner sep=0pt, outer sep=0pt, scale=  0.88] at ( 28.32,104.19) {0.3};
\end{scope}
\begin{scope}
\path[clip] (  0.00,  0.00) rectangle (209.58,115.63);
\definecolor{drawColor}{gray}{0.20}

\path[draw=drawColor,line width= 0.6pt,line join=round] ( 30.52, 33.45) --
	( 33.27, 33.45);

\path[draw=drawColor,line width= 0.6pt,line join=round] ( 30.52, 57.28) --
	( 33.27, 57.28);

\path[draw=drawColor,line width= 0.6pt,line join=round] ( 30.52, 83.39) --
	( 33.27, 83.39);

\path[draw=drawColor,line width= 0.6pt,line join=round] ( 30.52,107.22) --
	( 33.27,107.22);
\end{scope}
\begin{scope}
\path[clip] (  0.00,  0.00) rectangle (209.58,115.63);
\definecolor{drawColor}{RGB}{0,0,0}

\path[draw=drawColor,line width= 0.6pt,line join=round] ( 33.27, 29.80) --
	(204.08, 29.80);

\path[draw=drawColor,line width= 0.6pt,line join=round] (201.62, 28.38) --
	(204.08, 29.80) --
	(201.62, 31.23);
\end{scope}
\begin{scope}
\path[clip] (  0.00,  0.00) rectangle (209.58,115.63);
\definecolor{drawColor}{gray}{0.20}

\path[draw=drawColor,line width= 0.6pt,line join=round] ( 43.32, 27.05) --
	( 43.32, 29.80);

\path[draw=drawColor,line width= 0.6pt,line join=round] ( 60.07, 27.05) --
	( 60.07, 29.80);

\path[draw=drawColor,line width= 0.6pt,line join=round] ( 76.81, 27.05) --
	( 76.81, 29.80);

\path[draw=drawColor,line width= 0.6pt,line join=round] ( 93.56, 27.05) --
	( 93.56, 29.80);

\path[draw=drawColor,line width= 0.6pt,line join=round] (110.30, 27.05) --
	(110.30, 29.80);

\path[draw=drawColor,line width= 0.6pt,line join=round] (127.05, 27.05) --
	(127.05, 29.80);

\path[draw=drawColor,line width= 0.6pt,line join=round] (143.80, 27.05) --
	(143.80, 29.80);

\path[draw=drawColor,line width= 0.6pt,line join=round] (160.54, 27.05) --
	(160.54, 29.80);

\path[draw=drawColor,line width= 0.6pt,line join=round] (177.29, 27.05) --
	(177.29, 29.80);

\path[draw=drawColor,line width= 0.6pt,line join=round] (194.04, 27.05) --
	(194.04, 29.80);
\end{scope}
\begin{scope}
\path[clip] (  0.00,  0.00) rectangle (209.58,115.63);
\definecolor{drawColor}{gray}{0.30}

\node[text=drawColor,anchor=base,inner sep=0pt, outer sep=0pt, scale=  0.88] at ( 43.32, 18.79) {2};

\node[text=drawColor,anchor=base,inner sep=0pt, outer sep=0pt, scale=  0.88] at ( 60.07, 18.79) {4};

\node[text=drawColor,anchor=base,inner sep=0pt, outer sep=0pt, scale=  0.88] at ( 76.81, 18.79) {8};

\node[text=drawColor,anchor=base,inner sep=0pt, outer sep=0pt, scale=  0.88] at ( 93.56, 18.79) {16};

\node[text=drawColor,anchor=base,inner sep=0pt, outer sep=0pt, scale=  0.88] at (110.30, 18.79) {32};

\node[text=drawColor,anchor=base,inner sep=0pt, outer sep=0pt, scale=  0.88] at (127.05, 18.79) {64};

\node[text=drawColor,anchor=base,inner sep=0pt, outer sep=0pt, scale=  0.88] at (143.80, 18.79) {128};

\node[text=drawColor,anchor=base,inner sep=0pt, outer sep=0pt, scale=  0.88] at (160.54, 18.79) {256};

\node[text=drawColor,anchor=base,inner sep=0pt, outer sep=0pt, scale=  0.88] at (177.29, 18.79) {512};

\node[text=drawColor,anchor=base,inner sep=0pt, outer sep=0pt, scale=  0.88] at (194.04, 18.79) {1024};
\end{scope}
\begin{scope}
\path[clip] (  0.00,  0.00) rectangle (209.58,115.63);
\definecolor{drawColor}{RGB}{0,0,0}

\node[text=drawColor,anchor=base,inner sep=0pt, outer sep=0pt, scale=  1.00] at (118.68,  7.44) {$d$};
\end{scope}
\begin{scope}
\path[clip] (  0.00,  0.00) rectangle (209.58,115.63);
\definecolor{drawColor}{RGB}{0,0,0}

\node[text=drawColor,rotate= 90.00,anchor=base,inner sep=0pt, outer sep=0pt, scale=  1.00] at ( 12.39, 69.97) {$\beta$};
\end{scope}
\end{tikzpicture}}
	\caption{Average query times of \acl{lve} run on the output of \ac{acp} and \ac{aacp} where the input \acp{fg} contain a proportion of $p=0.15$ scaled factors (left), and the distribution of the number $\beta$ of queries after which the offline overhead of \ac{aacp} amortises on input \acp{fg} containing a proportion of $p=0.15$ scaled factors (right).}
	\label{fig:aacp_plot_app_p=0.15}
\end{figure}
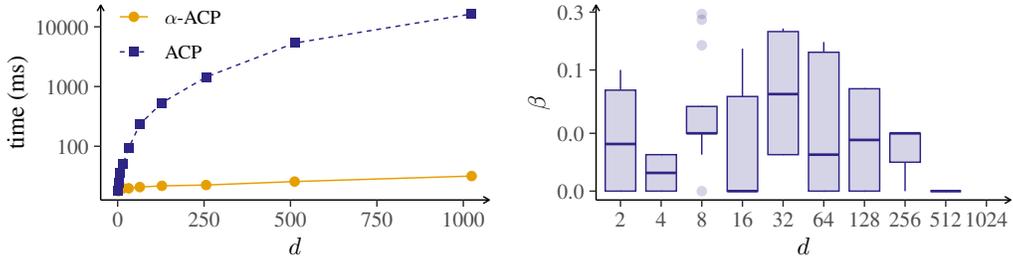

The results are illustrated in \cref{fig:aacp_plot_app_p=0.01,fig:aacp_plot_app_p=0.05,fig:aacp_plot_app_p=0.1,fig:aacp_plot_app_p=0.15}.
Unsurprisingly, the run times of \acl{lve} on the output of \ac{aacp} remain constant for all choices of $p$ because \ac{aacp} is able to detect arbitrarily many scaled factors without forfeiting compression.
At the same time, \ac{acp} is not able to detect exchangeable factors on different scales and thus, the run times of \acl{lve} on the output of \ac{acp} increase as the proportion $p$ of scaled factors increases.
Regarding the amortisation of the offline overhead (depicted in the plots on the right), we can observe the same behaviour as in \cref{fig:aacp_plot_main} (negative values for $\beta$ are again omitted).
The median value for $\beta$ is always below one and there are no notable differences between the different choices of $p$.
Even though the values of $\beta$ slightly deviate between different choices for $p$, the deviation can be considered negligible and it seems as the deviation stems from noisy measurements.
In conclusion, the additional offline overhead of \ac{aacp} amortises after a single query most of the time, highlighting the efficiency of \ac{aacp}.
\end{document}